\documentclass[10pt]{article}
\usepackage[letterpaper,margin=1in]{geometry}
\usepackage[T1]{fontenc}
\usepackage{lmodern}
\usepackage[round,authoryear]{natbib}
\usepackage{amsmath,amssymb,amsthm}
\usepackage{mathtools}
\usepackage{bm}
\usepackage{booktabs}
\usepackage{enumitem}
\usepackage{graphicx}
\usepackage[section]{placeins}    
\usepackage{hyperref}
\usepackage{url}
\usepackage{xcolor}

\renewcommand{\dh}{d_h}
\newcommand{\Khead}{K}                                

\newcommand{\bq}{\mathbf{q}}
\newcommand{\bk}{\mathbf{k}}
\newcommand{\Rcal}{\mathcal{R}}
\newcommand{\Rcalk}{\mathcal{R}^{(k)}}
\newcommand{\thetak}{\theta_k}
\newcommand{\Bcal}{\mathcal{B}}
\newcommand{\Gphi}{G(\phi)}

\newcommand{\bSigma}{\boldsymbol{\Sigma}}
\newcommand{\Sigmak}{\boldsymbol{\Sigma}^{(k)}}
\newcommand{\Sigmatk}{\widetilde{\boldsymbol{\Sigma}}^{(k)}}
\newcommand{\Ck}{C_k}
\newcommand{\Sk}{S_k}
\newcommand{\phistar}{\phi_k^{\ast}}
\newcommand{\lambdabar}{\bar{\lambda}_k}

\newcommand{\rhoeff}{\widetilde{\rho}}                
\newcommand{\bH}{\mathbf{H}}

\newcommand{\WAKV}{W4A4KV4}

\newcommand{\R}{\mathbb{R}}
\newcommand{\C}{\mathbb{C}}

\newcommand{\Z}{\mathbb{Z}}

\newcommand{\Var}{\mathrm{Var}}

\newtheorem{theorem}{Theorem}
\newtheorem{lemma}[theorem]{Lemma}
\newtheorem{proposition}[theorem]{Proposition}

\theoremstyle{definition}
\newtheorem{definition}[theorem]{Definition}

\title{When Local Variance Optimality Is Not Enough:\\ RoPE-Aligned Q/K Rotations\\ for Dynamic 4-Bit Quantisation}
\author{%
\begin{tabular}{cc}
Shuhan Wang & Yilin Luo \\
\href{mailto:shuhan.wang.25@ucl.ac.uk}{\texttt{shuhan.wang.25@ucl.ac.uk}} &
\href{mailto:yilin.luo.25@ucl.ac.uk}{\texttt{yilin.luo.25@ucl.ac.uk}} \\[0.75em]
Nan Xu\textsuperscript{\ensuremath{\dagger}} &
Chi Wang Cheung\textsuperscript{\ensuremath{\dagger}} \\
\href{mailto:nan.xu.25@ucl.ac.uk}{\texttt{nan.xu.25@ucl.ac.uk}} &
\href{mailto:c.cheung.25@ucl.ac.uk}{\texttt{c.cheung.25@ucl.ac.uk}}
\end{tabular}\\[0.5em]
{\small \textsuperscript{\ensuremath{\dagger}}Equal contribution.}}
\date{}

\begin{document}
\maketitle

\begin{abstract}
Rotation-based post-training quantisation commonly applies an orthogonal transform across an entire attention head to reduce outlier-induced error. RoPE instead partitions each head into two-dimensional frequency pairs, raising a specific design question: can a transform that respects this decomposition improve on full-head mixing? Prior work has established the per-pair rotations that commute with RoPE. We state the converse result that, for distinct frequencies, no other single-head orthogonal map commutes with RoPE. For the head-shared parameterisation used in our experiments, we then derive the rotation angle that minimises the larger channel variance under a pooled-covariance, position-averaged surrogate and verify that the implementation attains its analytic minimum.

The evaluated head-shared pairwise configuration does not improve accuracy in the tested dynamic \WAKV{} setting. Across four checkpoints, replacing the full-head Hadamard with this configuration increases perplexity at both short and long context lengths. Composing the pairwise rotation with the Hadamard satisfies the selected $\pm0.05$-PPL interval criterion under the default estimator. Estimating the shared angle from K alone improves pairwise-only on every checkpoint without eliminating its gap relative to full-head mixing. The analytic objective controls a position-averaged second moment of a covariance reconstructed from concatenated calibration observations, whereas the dynamic quantiser determines its step from a tokenwise group range. The pairwise transform also provides only two-channel support for redistributing a peak. Along a controlled interpolation from two-channel to full-head mixing, K range, relative quantisation error, and perplexity degradation decrease as support increases. These results show that optimality for a structured surrogate need not reduce quantisation error when the surrogate and mixing support are misaligned with the quantiser's scale-setting statistic.
\end{abstract}

\section{Introduction}
\label{sec:introduction}

Low-bit post-training quantisation depends strongly on how activation outliers are presented to the quantiser \citep{dettmers2022llmint8}. Rotation-based methods insert orthogonal transforms that redistribute activation energy while preserving the unquantised network function \citep{ashkboos2024quarot,liu2024spinquant,dartquant2025}. These methods build on weight-side incoherence and transformer invariance results \citep{chee2023quip,tseng2024quipsharp,ashkboos2024slicegpt}; \S\ref{sec:related_work} reviews this development. For the online transform $R_3$ applied to post-RoPE queries and keys, the relevant design choice is the \emph{mixing support}: should the transform mix all channels in an attention head, or only the two channels in each RoPE frequency pair? We study this question under dynamic \WAKV{} quantisation.

RoPE \citep{su2021rope} provides a natural motivation for pairwise mixing because it partitions every head into fixed two-dimensional subspaces. A rotation within these subspaces respects the frequency-pair decomposition and can be commuted through RoPE. Applying the same post-RoPE orthogonal map to Q and K already preserves their inner products, so commutativity is a structural constraint rather than a correctness requirement. We test whether imposing this constraint improves quantisation accuracy.

For a single head with distinct RoPE frequencies, the commuting orthogonal maps are exactly the independent planar rotations within the frequency pairs. FPTQuant \citep{fptquant2025} previously used this family by folding the transform into $W_Q$ and $W_K$; our theoretical contribution is the converse characterisation that excludes other single-head commuting orthogonal maps. The evaluated implementation further restricts the family by sharing one angle for each layer and frequency pair across attention heads. For this head-shared parameterisation, we derive the closed-form angle $\phistar$ that minimises the larger channel variance under a pooled-covariance, position-averaged surrogate. On Llama-3.2-3B, the implementation uses the checkpoint's deployed RoPE frequencies and attains the analytic minimum of that surrogate. This construction distinguishes failure to optimise the stated local objective from failure of the optimised surrogate to improve quantisation accuracy.

The experiments exhibit the second outcome. Across the evaluated Llama and Mistral checkpoints, replacing the full-head Hadamard with the head-shared pairwise configuration increases perplexity in every short-context comparison. The same ordering holds at every evaluated long-context point, including 128K where available. Composing the pairwise rotation with the Hadamard satisfies the selected interval criterion under the default estimator. These observations establish that the optimised local surrogate does not provide the range reduction achieved by full-head mixing in the evaluated setting.

We examine two factors associated with this discrepancy. First, the default angle estimator pools Q and K calibration statistics even though K is the only stream quantised after $R_3$. Estimating the shared angle from K alone improves the pairwise-only configuration on all four checkpoints, although its perplexity remains higher than that of the full-head Hadamard. Second, the analytic surrogate and the quantiser operate on different statistics: the closed-form rotation minimises a position-averaged second moment within each pair, while the dynamic INT4 quantiser determines its step from a tokenwise range over a larger group. Pairwise rotations also have two-channel mixing support. In an interpolation that expands block-Hadamard support from one pair to a full head, K range, relative quantisation error, and perplexity degradation decrease with support size.

These findings support three contributions.

\paragraph{Contributions.}
\begin{itemize}[leftmargin=1.5em,itemsep=2pt,topsep=2pt]
\item \textbf{An exact local benchmark.} For a single head with distinct RoPE frequencies, the RoPE-commuting orthogonal family consists of independent pairwise rotations. We attribute this family to \citet{fptquant2025} and establish the converse characterisation. For the head-shared implementation, we derive the minimax-variance member of a pooled-covariance, position-averaged surrogate and verify attainment of its analytic minimum.
\item \textbf{A scoped negative result for head-shared pairwise mixing.} Under dynamic \WAKV{}, the evaluated head-shared RoPE-aligned configuration yields higher perplexity than full-head Hadamard mixing across four checkpoints and all evaluated contexts. On Llama-3.2-3B, the same ordering holds for the verified surrogate optimum $\phistar$. K-only estimation improves pairwise-only without reversing its ordering relative to full-head mixing.
\item \textbf{Evidence on proxy alignment and mixing support.} The theoretical surrogate controls a pairwise, position-averaged second moment, whereas the dynamic quantiser uses a tokenwise group range to determine its scale. Along one controlled support interpolation, broader mixing is associated with lower K range, lower quantisation error, and smaller end-to-end degradation.
\end{itemize}

\section{Problem setup: RoPE-aligned versus full-head Q/K mixing}
\label{sec:setup}

We study the online transform $R_3$ applied to queries and keys after RoPE. Using the same orthogonal map for both streams preserves their inner products. We impose commutativity as the structural constraint of interest and compare two support sizes: mixing within each two-dimensional RoPE frequency pair and mixing across an entire attention head.

\paragraph{Frequency pairs and mixing support.}
Fix one even-dimensional head, $\dh=2\Khead$, with distinct RoPE frequencies $\thetak$. Let $\Gphi\in SO(2)$ be a planar rotation and $\Rcal_m=\bigoplus_{k=1}^{\Khead}G(m\thetak)$ the RoPE map at position $m$ \citep{su2021rope}. Implementations use the \texttt{rotate\_half} layout, in which frequency pair $k$ occupies channels $\{k,\,k+\dh/2\}$ rather than the adjacent coordinates $\{2k{-}1,2k\}$. The layouts differ by a fixed channel permutation, so the per-pair statements are invariant to this relabelling. The deployed implementation applies the rotation in each checkpoint's native layout, as documented in Appendix~\ref{app:reproducibility}. We refer to the $k$-th two-dimensional subspace as a \emph{RoPE frequency pair} and use \emph{band} only in equations and formal statements. A full-head Hadamard mixes all $\dh$ channels in a head, while a pairwise transform mixes only the two channels in one frequency pair. Their mixing supports are therefore $\dh$ and $2$, respectively.

\paragraph{What the deployed quantiser measures.}
At this stage of the evaluated quantisation procedure, K is quantised and the post-RoPE Q stream is not. For each token $t$ and channel group $g$, dynamic asymmetric INT4 computes the range $u_{t,g}-\ell_{t,g}$ and sets $\Delta_{t,g}=(u_{t,g}-\ell_{t,g})/(2^B-1)$ with $B=4$. The scale-setting statistic is therefore a tokenwise range over a group that spans multiple frequency pairs. It differs from a per-pair second moment averaged over positions.

\paragraph{The local variance objective is empirically non-trivial.}
A descriptive calibration pass finds non-negligible within-pair anisotropy and correlation on every evaluated checkpoint: mean $|\rho|$ is $0.135$--$0.175$, between $72\%$ and $78\%$ of (layer, head, pair) triples have $|\rho|>0.05$, and the mean absolute log variance ratio is $0.36$--$0.58$. The pairwise variance objective therefore has structure to optimise; Table~\ref{tab:data-rho} reports the complete diagnostics.

\section{The single-head RoPE centraliser and a pooled-covariance optimum}
\label{sec:theory}

With distinct frequencies, commutativity restricts each head to independent planar rotations within its RoPE frequency pairs. We first characterise this single-head family and then define the exact surrogate optimised by the head-shared implementation. The resulting angle minimises the larger channel variance under a pooled-covariance, position-averaged objective. This result provides a local analytic benchmark for the stated surrogate; it does not characterise the tokenwise range used by the quantiser or end-to-end perplexity.

The calibration collector treats samples, sequence positions, and attention heads as observation rows and accumulates their raw first and second moments. The default estimator adds the Q and K accumulators before applying a single centring operation; the alternative estimators modify only this stream-level aggregation rule. Each estimator therefore produces one covariance matrix for every layer and frequency pair. We write this matrix as $\bSigma^{(k)}$, suppressing the layer index. It is the covariance of the concatenated observation rows, rather than an average of separately centred head-wise or stream-wise covariance matrices.

The analysis reuses this matrix at every deployed position. With the finite phase averages
\begin{equation}
\label{eq:CkSk}
\Ck=\frac1L\sum_{m=0}^{L-1}\cos(2m\thetak),\qquad
\Sk=\frac1L\sum_{m=0}^{L-1}\sin(2m\thetak),
\end{equation}
we define the pooled-covariance, position-averaged surrogate
\begin{equation}
\label{eq:Sigmatilde}
\Sigmatk=\frac1L\sum_{m=0}^{L-1}\Rcalk_m\Sigmak(\Rcalk_m)^\top.
\end{equation}
If $\bSigma_m^{(k)}$ denotes the actual pre-RoPE covariance conditional on position $m$, the corresponding position-conditional quantity would instead be
\[
\frac1L\sum_{m=0}^{L-1}\Rcalk_m\bSigma_m^{(k)}(\Rcalk_m)^\top.
\]
The two expressions coincide under the additional condition $\bSigma_m^{(k)}=\bSigma^{(k)}$ at every deployed position. We neither assume nor empirically establish this position-stationarity condition. All optimality statements below therefore refer to the pooled-covariance surrogate in Eq.~\eqref{eq:Sigmatilde}.

Write $\sigma_1^2,\sigma_2^2$ for the diagonal of $\Sigmak$ and $\rho$ for its correlation. Within this surrogate, finite phase averaging can induce in-pair structure: $\Sigmatk_{12}=\tfrac12(\sigma_1^2-\sigma_2^2)\Sk+\rho\sigma_1\sigma_2\Ck$ is generally non-zero when the channel variances differ and $\Sk\neq0$, including when the pooled pre-RoPE covariance is uncorrelated. The surrogate structure can therefore arise from variance asymmetry and finite RoPE phases rather than from the optimiser-related outlier effect documented by \citet{ahmadian2023intriguing}.

\begin{definition}[Band-aligned subgroup]
\label{def:band-aligned}
$\Bcal_{\dh}:=\{\bigoplus_{k=1}^{\Khead}G(\phi_k):\phi_k\in(-\pi,\pi]\}\cong SO(2)^{\dh/2}$.
\end{definition}

\begin{lemma}[Single-head RoPE centraliser]
\label{lem:centraliser-commutes}
For pairwise-distinct RoPE angles $\theta_k\in(0,\pi)$, the centraliser of $\{\Rcal_m\}_{m\in\mathbb Z}$ in $O(\dh)$ is $\Bcal_{\dh}$.
\end{lemma}

The lemma characterises one attention head. Excluding cross-head mixing, its $H$-head direct-product extension permits
\[
R_3=\bigoplus_{h=1}^{H}\bigoplus_{k=1}^{\Khead}G(\phi_{h,k}),
\]
with head-dependent angles. The evaluated implementation imposes the additional constraint $\phi_{h,k}\equiv\phi_k$ within each layer because it aggregates calibration rows across heads and broadcasts the resulting angle to every head. The experiments therefore evaluate a strict head-shared subfamily of the head-dependent direct-product family.

\begin{theorem}[Closed-form optimum for the pooled-covariance surrogate]
\label{thm:closed-form-optimal-angle}
For fixed $\Sigmatk$, the following angle equalises the two diagonal entries of $G(\phi_k)\Sigmatk G(\phi_k)^\top$:
\begin{equation}
\label{eq:phistar}
\phistar=\operatorname{wrap}_{[-\pi/4,\pi/4)}\!\left[\tfrac12\operatorname{atan2}\!\left(\Sigmatk_{11}-\Sigmatk_{22},2\Sigmatk_{12}\right)\right].
\end{equation}
It satisfies
\[
\min_{\phi_k}\max_{j\in\{1,2\}}\left[G(\phi_k)\Sigmatk G(\phi_k)^\top\right]_{jj}
=\frac{\sigma_1^2+\sigma_2^2}{2}.
\]
This optimum is defined with respect to the pooled-covariance surrogate; finite phase averaging preserves the trace of $\Sigmak$.
\end{theorem}

Proofs, the $2\times2$ Hadamard comparison, and the multi-head scope of Lemma~\ref{lem:centraliser-commutes} are in Appendix~\ref{app:proofs}. When the surrogate objective is recomputed from covariances reconstructed after raw-moment aggregation and from the deployed checkpoint frequencies, every evaluated $\phistar$ configuration attains its analytic minimum. The maximum excess is $6.90\times10^{-7}$ under a verification tolerance of $5\times10^{-5}$ (Appendix~\ref{app:data}, \S J). This verification concerns the covariance of concatenated rows; it does not establish simultaneous equalisation of head-specific covariances or optimality over head-dependent angles.

Attention preservation constrains how the angle is \emph{applied}, but not how it is estimated:
\begin{proposition}[Attention preservation requires shared Q/K band angles]
\label{prop:rope-commutativity-requires-shared-theta}
For $R_3^Q,R_3^K\in\Bcal_{\dh}$, the original relative-position attention score is preserved exactly iff their angles agree in every band modulo $2\pi$.
\end{proposition}

The matched per-pair construction of \citet{fptquant2025} implicitly satisfies this condition. The proposition constrains application of the estimator: Q and K must receive the same angle within a matched head and frequency pair, although that angle may be estimated from either stream. Because only K is quantised at this stage, a K-derived estimator is more directly aligned with the measured quantisation pathway. We therefore test whether K-only estimation changes perplexity (\S\ref{sec:results-short-context}). An observed difference establishes sensitivity to the estimator statistic; it does not, by itself, identify K-cache quantisation error as the causal pathway. This Q/K-sharing requirement is distinct from the implementation's additional choice to share angles across attention heads.

\section{Evaluation design}
\label{sec:experimental_setup}

To test whether local optimality translates into lower quantisation error, we compare three matched online Q/K transforms:
\begin{itemize}[leftmargin=1.5em,itemsep=1pt,topsep=2pt]
\item \textbf{full-head Hadamard}: the baseline $R_3$, which mixes all $\dh$ channels in a head;
\item \textbf{pairwise-only}: the head-shared RoPE-commuting rotation used in place of the Hadamard;
\item \textbf{pairwise\,+\,Hadamard}: the same head-shared pairwise rotation composed with the Hadamard.
\end{itemize}
In the appendix and artifact code, these transforms are named \texttt{hadamard}, \texttt{eq\_only}, and \texttt{eq\_layered}, respectively.

We evaluate Llama-3.2-1B \citep{meta2024llama32_1b}, Llama-3.2-3B \citep{meta2024llama32_3b}, Llama-3.1-8B \citep{dubey2024llama3}, and Mistral-7B-v0.3 \citep{mistralai2024mistral7bv03}, whose architecture follows \citet{jiang2023mistral}. We use GPTQ weights \citep{frantar2022gptq} with dynamic per-token \WAKV{} activation and KV-cache quantisation. Calibration uses 128 WikiText-2 \citep{merity2016wikitext} sequences of length 2048 and DartQuant's \textsc{whip} objective \citep{dartquant2025} for the shared offline rotations. Within each paired comparison, all configurations use the same offline rotations, GPTQ Hessian, calibration sample, and random seed; only $R_3$ varies. Long-context evaluation follows a concatenated-stream protocol on Proof-Pile \citep{azerbayev2023proofpile} and PG19 \citep{rae2019pg19}. The primary outcomes are perplexity, mean K range after $R_3$, and relative K quantisation error.

We report paired differences relative to the same-seed full-head Hadamard baseline, with positive values indicating higher perplexity. Pairwise\,+\,Hadamard satisfies the manuscript's selected $\pm0.05$-PPL interval criterion when its complete $90\%$ paired confidence interval lies within $[-0.05,+0.05]$, corresponding to the interval form of the two one-sided tests procedure \citep{lakens2017equivalence}. This threshold is an explicit reporting criterion, not a task-independent minimum important difference; all equivalence statements are therefore conditional on the selected margin. Appendix~\ref{app:data}, Table~\ref{tab:tost-sensitivity}, reports sensitivity at $\pm0.02$ and $\pm0.10$ PPL. A cross-checkpoint equivalence statement is made only when every checkpoint satisfies the criterion.

We compare three estimators within the head-shared parameterisation. The default estimator $\hat\phi_k$ adds raw Q and K moment accumulators across attention heads and centres the resulting observation population once before constructing one covariance for each layer and frequency pair; it then applies Eq.~\ref{eq:phistar} with $(\Ck,\Sk)=(1,0)$. Because this aggregation weights streams by physical head count, grouped-query attention \citep{ainslie2023gqa} makes $75$--$80\%$ of the concatenated observations Q-derived even though only K is quantised. The K-only estimator constructs the covariance from K observations alone. On Llama-3.2-3B, the stream-balanced estimator assigns Q and K equal stream-level weights to their uncentred moments before the common centring operation. The position-averaged surrogate optimum $\phistar$ is evaluated only on Llama-3.2-3B, using the checkpoint's deployed scaled RoPE frequencies and a branch-matched $\hat\phi_k$ control to isolate the finite-phase correction.

Unless otherwise stated, confidence intervals are unadjusted, pointwise paired $t$ intervals. The angle-estimator analyses are exploratory effect-estimation analyses and are interpreted through estimated magnitudes and directional consistency across checkpoints. The twelve long-context comparisons between $\phistar$ and its branch-matched control form a separate multiplicity family and use Holm adjustment of two-sided $p$-values (Appendix~\ref{app:data}, \S J). Seed counts appear in each table caption; Appendix~\ref{app:reproducibility} documents the environments, inference conventions, and reproduction commands.

The support interpolation uses the same 3B/WikiText-2 dynamic \WAKV{} setting and pair-interleaved block-Hadamard transforms with block size $b\in\{2,4,8,16,32,64,128\}$. It measures K range and relative quantisation error after $R_3$. Only $b=2$ belongs to the commuting family; larger blocks vary support along a specific non-commuting transform path. Layout and orientation controls use one seed and are reported in Appendix~\ref{app:data}.

\section{The evaluated head-shared configuration does not match full-head mixing}
\label{sec:results-short-context}

\paragraph{Four-checkpoint short-context test.}
Table~\ref{tab:multiseed} answers the primary empirical question for the evaluated head-shared parameterisation. Under the default estimator, pairwise-only has higher mean perplexity than the full-head Hadamard on every checkpoint, with differences ranging from $+1.33$ PPL on Llama-3.2-1B to $+0.05$ PPL on Mistral-7B-v0.3. For pairwise\,+\,Hadamard, the complete $90\%$ paired confidence interval lies within the selected $\pm0.05$-PPL criterion on every checkpoint. The replacement and composition results therefore support distinct conclusions: pairwise-only does not reproduce the perplexity of full-head mixing, whereas the composed transform satisfies the stated interval criterion. The latter conclusion is conditional on the selected margin.

\begin{table}[t]
\centering\small
\caption{Short-context paired perplexity differences relative to each checkpoint's full-head Hadamard baseline on WikiText-2 under dynamic \WAKV{} with \textsc{whip} calibration, $n=6$. Positive values indicate higher perplexity. Under the default estimator, pairwise\,+\,Hadamard satisfies the selected $\pm0.05$-PPL interval criterion on every checkpoint. K-only estimation improves pairwise-only on every checkpoint without reducing its perplexity below the full-head Hadamard baseline.}
\label{tab:multiseed}\label{tab:pooling-breadth}
\begin{tabular}{lrrrr}
\toprule
 & \multicolumn{2}{c}{pairwise-only} & \multicolumn{2}{c}{pairwise\,+\,Hadamard} \\
\cmidrule(lr){2-3}\cmidrule(lr){4-5}
model & default est. & K-only est. & default est. & K-only est. \\
\midrule
Llama-3.2-1B    & $+1.3339$ & $+0.8053$ & $-0.0099$ & $+0.0461$ \\
Llama-3.2-3B    & $+0.4835$ & $+0.3169$ & $+0.0053$ & $+0.0326$ \\
Llama-3.1-8B    & $+0.1898$ & $+0.1238$ & $-0.0002$ & $+0.0127$ \\
Mistral-7B-v0.3 & $+0.0538$ & $+0.0467$ & $-0.0015$ & $+0.0003$ \\
\bottomrule
\end{tabular}
\end{table}

\paragraph{Estimator ablation within the head-shared parameterisation.}
Relative to the default estimator, K-only estimation reduces pairwise-only perplexity on all four checkpoints, including reductions of $0.53$ PPL on Llama-3.2-1B and $0.17$ PPL on Llama-3.2-3B; all four unadjusted $90\%$ intervals for the K-only-minus-default contrast lie below zero. On Llama-3.2-3B, the stream-balanced pairwise-only estimate lies between the default and K-only estimates. Relative to identity on that checkpoint, the interval for the default pairwise-only estimator includes zero, whereas the K-only and stream-balanced estimates and their unadjusted intervals are negative. For pairwise\,+\,Hadamard, the K-only-minus-default estimates are positive on Llama-3.2-1B, Llama-3.2-3B, and Llama-3.1-8B, with intervals above zero, while the Mistral-7B-v0.3 interval includes zero. These exploratory comparisons are not adjusted for family-wise multiplicity; we therefore interpret their magnitudes and directional consistency rather than treating individual intervals as confirmatory decisions. None of the evaluated estimators reduces pairwise-only perplexity below the full-head Hadamard baseline. Complete per-seed results are reported in Appendix~\ref{app:data}, \S L.

\paragraph{The pooled-covariance surrogate optimum on Llama-3.2-3B.}
The matched Llama-3.2-3B comparison gives the same ordering when the estimated angle is replaced by the verified position-averaged surrogate optimum. Pairwise-only with $\phistar$ has the highest mean perplexity among the reported pairwise-only configurations despite attaining the analytic objective value, and its unadjusted paired interval relative to identity lies entirely above zero. When composed with the Hadamard, both $\phistar$ and $\hat\phi_k$ satisfy the selected $\pm0.05$-PPL interval criterion relative to full-head Hadamard. Branch-matched, alternative-objective, and identity comparisons are reported in Appendix~\ref{app:data}.

\begin{table}[t]
\centering\small
\caption{Matched Llama-3.2-3B short-context configurations on WikiText-2 under dynamic \WAKV{} with \textsc{whip} calibration, $n=6$, with the angle taken from the default estimator $\hat\phi_k$ or from the position-averaged surrogate optimum $\phistar$ as indicated. The reported difference is configuration $-$ full-head Hadamard; positive values indicate higher perplexity. The $\phistar$ configurations use the checkpoint's deployed scaled RoPE frequencies.}
\label{tab:corrected-star-short}
\begin{tabular}{lccc}
\toprule
configuration & mean PPL & gap vs Hadamard & $90\%$ CI \\
\midrule
full-head Hadamard & $10.1121$ & \textit{n/a} & \textit{n/a} \\
\texttt{identity} (no $R_3$) & $10.5764$ & $+0.4643$ & $[+0.4294,+0.4992]$ \\
pairwise-only, $\hat\phi_k$ & $10.5956$ & $+0.4835$ & $[+0.4542,+0.5128]$ \\
pairwise-only, $\phistar$ & $10.6499$ & $+0.5378$ & $[+0.5071,+0.5686]$ \\
pairwise\,+\,Hadamard, $\hat\phi_k$ & $10.1174$ & $+0.0053$ & $[-0.0026,+0.0132]$ \\
pairwise\,+\,Hadamard, $\phistar$ & $10.1301$ & $+0.0180$ & $[+0.0050,+0.0311]$ \\
\bottomrule
\end{tabular}
\end{table}

These results motivate an examination of why the exact surrogate optimum does not transfer to the deployed quantiser.

\section{Evidence on the discrepancy between local optimality and quantised performance}
\label{sec:mechanism}

\subsection{Empirical discrepancy}

The discrepancy involves three distinct quantities: the pooled-covariance surrogate optimised in theory, the tokenwise group range used by the quantiser, and measured perplexity. The evaluated head-shared configuration attains the first objective but remains less accurate than full-head mixing. We therefore examine how the surrogate differs from the scale-setting statistic and how both relate to mixing support.

\subsection{Objective mismatch}

Theorem~\ref{thm:closed-form-optimal-angle} concerns the pooled-covariance surrogate. Dynamic INT4 instead uses a tokenwise $\max-\min$ range over groups that span multiple frequency pairs (\S\ref{sec:setup}). A quadratic moment averaged over RoPE phases can decrease while $\Delta_{t,g}$ remains unchanged because the two statistics apply different operations over different supports. In the matched 3B comparison, $\phistar$ attains the surrogate minimum and has the largest pairwise-only perplexity gap. This observation bounds the analytic certificate: it does not extend to the scale-setting statistic of the deployed quantiser. Recent methods address related objective mismatches directly. \citet{hu2025ostquant} measure utilisation of the quantisation space, and \citet{wang2025squat} constrain key-quantisation error relative to a query subspace. Our result provides a negative instance in which exact optimisation of a pooled, position-averaged second moment does not control the quantiser's tokenwise group range.

Finite phase averaging also affects the interpretation of exact surrogate optimality. For high-frequency pairs, both sums $(\Ck,\Sk)$ in Eq.~\ref{eq:CkSk} are close to zero, so the surrogate covariance approaches isotropy and many angles are nearly optimal. The selected angular representative can then depend on a small finite-sum remainder. At $L=2048$, $17$ of the $64$ deployed 3B pairs satisfy $|(\Ck,\Sk)|<0.01$, and the corresponding mean offset $|\tfrac12\operatorname{atan2}(\Sk,\Ck)|$ is $0.75$ radians. This conditioning issue limits the deployment interpretation of the surrogate optimum while leaving the numerical attainment check unchanged (Appendix~\ref{app:data}, \S J).

\subsection{Mixing support}

Mixing support imposes a complementary structural constraint. If an orthogonal transform distributes one peak with equal magnitude across $b$ channels, norm preservation gives the lower bound $|x|/\sqrt b$ on the resulting peak magnitude. A pairwise rotation has $b=2$, whereas a full-head Hadamard has $b=\dh$; at $\dh=128$, the corresponding bounds differ by approximately $8\times$. This statement concerns an outlier-range proxy and does not imply a bound on perplexity.

To examine mixing support along a controlled transformation path, we vary $b$ within a pair-interleaved block-Hadamard family while retaining the calibration and evaluation protocol (Table~\ref{tab:mixing-support}). The point estimates for K range, relative K quantisation error, and perplexity difference decrease monotonically as $b$ increases. The $b\leq32$ perplexity intervals lie above the full-head Hadamard baseline, whereas the $b=64$ and $b=128$ intervals include zero. At $b=128$, all three outcomes approach the full-head endpoint, and the paired perplexity interval includes zero. More specifically, relative K quantisation error matches at the reported precision, while mean K range differs by $0.0011$. This pattern documents a joint progression of support size and the three measured outcomes along one transformation family; it does not isolate support independently of the accompanying change in block structure.

The dependence of outlier suppression on block size is established elsewhere and is not claimed here. DuQuant \citep{lin2024duquant} rotates within blocks and then permutes channels to balance outliers across them; PeRQ \citep{sanjeet2026blockrotations} gives a non-asymptotic analysis in which the suppression a block-Hadamard transform achieves is governed by how $\ell_1$ mass is distributed over blocks, and recovers part of the lost suppression using the same permutation device; \citet{jia2026sawint4} select block-diagonal rotation with tokenwise INT4 as the preferred KV-cache configuration under serving constraints. What Table~\ref{tab:mixing-support} contributes is the post-RoPE Q/K instance of that axis under a dynamic per-token quantiser, with K range, relative K quantisation error and perplexity measured together along one interpolation path.

\begin{table}[t]
\centering\small
\caption{Mixing-support interpolation for Llama-3.2-3B on WikiText-2 under dynamic \WAKV{}, using three paired seeds. The perplexity difference is the block-Hadamard configuration minus the same-seed full-head Hadamard baseline; K diagnostics are measured after $R_3$. Only $b=2$ commutes with RoPE.}
\label{tab:mixing-support}
\begin{tabular}{@{}rcccc@{}}
\toprule
$b$ & $1/\sqrt b$ & mean K range & rel.\ K quant.\ error & PPL gap $[90\%\ \mathrm{CI}]$ \\
\midrule
2   & $0.7071$ & $14.8156$ & $0.02306$ & $+0.4158\ [+0.3720,+0.4596]$ \\
4   & $0.5000$ & $13.1189$ & $0.01805$ & $+0.2262\ [+0.2069,+0.2456]$ \\
8   & $0.3536$ & $11.5768$ & $0.01386$ & $+0.1373\ [+0.1242,+0.1503]$ \\
16  & $0.2500$ & $10.7337$ & $0.01183$ & $+0.0751\ [+0.0505,+0.0996]$ \\
32  & $0.1768$ & $10.4475$ & $0.01123$ & $+0.0544\ [+0.0266,+0.0822]$ \\
64  & $0.1250$ & $10.1357$ & $0.01054$ & $+0.0153\ [-0.0286,+0.0592]$ \\
128 & $0.0884$ & $9.6431$  & $0.00954$ & $+0.0039\ [-0.0236,+0.0313]$ \\
\midrule
full-head Hadamard & \textit{n/a} & $9.6420$ & $0.00954$ & $0$ \\
\bottomrule
\end{tabular}
\end{table}

\subsection{What the evidence supports}

Together, these analyses support a bounded interpretation. K-only estimation shows that the result is sensitive to the statistic from which the shared angle is estimated. Along the block-Hadamard interpolation, larger mixing support is accompanied by lower K range, lower relative quantisation error, and a smaller perplexity difference. This interpolation covers one checkpoint, one corpus, one context length, and one transformation path, and it leaves the RoPE-commuting family for $b>2$. It therefore supports an association within the evaluated intervention rather than a general causal attribution. The evidence identifies two relevant mismatches, namely the estimator's second-moment proxy and the transform's limited support, without assigning exclusive causality to either.

\section{Robustness across models, contexts, and tasks}
\label{sec:results_long_context}

The preceding analysis uses a short-context Llama-3.2-3B setting. We next assess whether the pairwise-replacement ordering is stable across context lengths, corpora, checkpoints, and evaluation formats. These evaluations test robustness of the ordering and provide no additional identification of its cause.

Consistent with the short-context ordering in \S\ref{sec:results-short-context}, pairwise-only has higher mean perplexity than full-head Hadamard at every evaluated long-context point under the default estimator. Table~\ref{tab:concat_proofpile} includes all evaluated checkpoint--corpus combinations: Llama-3.2-3B and Llama-3.2-1B are evaluated on Proof-Pile and PG19 through 128K, and Llama-3.1-8B and Mistral-7B-v0.3 are evaluated on both corpora through 64K. Because the scored token set depends on context length, the table summarises within-length differences and does not model perplexity as a function of length. Complete perplexities and the position-averaged contrasts are reported in Appendix~\ref{app:data}; exact per-length window counts are specified in Appendix~\ref{app:reproducibility}, \S A.4.

\begin{table}[!ht]
\centering\small
\caption{Ranges of evaluated within-length pairwise-only perplexity differences relative to full-head Hadamard under the default estimator and dynamic \WAKV{}, $n=3$. Positive values indicate higher perplexity for pairwise-only. Each range summarises point estimates across context lengths and is not a confidence interval.}
\label{tab:concat_proofpile}
\begin{tabular}{lccr}
\toprule
checkpoint / corpus & tested lengths & gap range & longest endpoint \\
\midrule
3B / Proof-Pile & 2K--128K & $+0.143$ to $+0.207$ & $+0.193$ at 128K \\
3B / PG19 & 2K--128K & $+0.566$ to $+1.199$ & $+1.169$ at 128K \\
1B / Proof-Pile & 2K--128K & $+0.416$ to $+0.574$ & $+0.513$ at 128K \\
1B / PG19 & 2K--128K & $+1.819$ to $+3.861$ & $+3.585$ at 128K \\
8B / Proof-Pile & 2K--64K & $+0.056$ to $+0.076$ & $+0.074$ at 64K \\
8B / PG19 & 2K--64K & $+0.206$ to $+0.480$ & $+0.394$ at 64K \\
Mistral / Proof-Pile & 2K--64K & $+0.034$ to $+0.056$ & $+0.056$ at 64K$^\dagger$ \\
Mistral / PG19 & 2K--64K & $+0.113$ to $+0.238$ & $+0.238$ at 64K$^\dagger$ \\
\bottomrule
\multicolumn{4}{l}{\footnotesize $^\dagger$ Exploratory 64K evaluation beyond the 32K context documented for Mistral-7B-v0.3 \citep{mistralai2024mistral7bv03}.}
\end{tabular}
\end{table}

Generation-based evaluations provide limited corroboration. In the synthetic retrieval evaluation following the NIAH format used by RULER \citep{hsieh2024ruler}, full-head Hadamard has higher retrieval accuracy than pairwise-only at context lengths with appreciable separation from the FP16 result; shorter-context results and the Llama-3.1-8B evaluation do not distinguish the configurations. LongBench-v2 \citep{bai2024longbenchv2} scores are at or below the four-choice random baseline and therefore provide no evidence for differences among the transformations at this quantisation setting. Appendix~\ref{app:data} reports the complete task results.

An exploratory static per-channel control produces lower short-context perplexity for pairwise-only than for full-head Hadamard. The control estimates fixed scales from target-corpus windows that overlap the scored text and changes scale granularity from grouped per-token to per-channel quantisation. It therefore demonstrates that the observed ordering depends on the evaluated quantiser protocol, but it neither identifies the source of that dependence nor estimates out-of-sample deployment performance. Related evidence that static and dynamic quantisers can respond differently to the same modification is reported by \citet{prefixquant2025}. No generation-based task evaluation was conducted for the static control, so this result is limited to perplexity under the stated protocol; complete curves and scale diagnostics are in Appendix~\ref{app:data}.

\section{Related work}
\label{sec:related_work}

\paragraph{From incoherence processing to online rotations.}
Rotation-based post-training quantisation originates on the weight side: QuIP \citep{chee2023quip} makes weights and Hessians incoherent with random orthogonal transforms, and QuIP\# \citep{tseng2024quipsharp} replaces those transforms with the randomised Hadamard transform. The online setting studied here inherits that construction. The offline rotations rest on the computational invariance of RMSNorm transformers established by \citet{ashkboos2024slicegpt}, which permits an orthogonal map to be folded into the weights; the online transforms are instead inserted together with their inverses. QuaRot \citep{ashkboos2024quarot} inserts Hadamard transforms into the activation path under \WAKV{}, SpinQuant \citep{liu2024spinquant} learns the offline rotations, and DartQuant \citep{dartquant2025} calibrates them efficiently; we reproduce its released calibration defaults exactly, with the consequences recorded in Appendix~\ref{app:reproducibility}. Later learned and closed-form transforms retain full-head support \citep{flatquant2025,xiang2024dfrot,kurtail2024,wush2025}, including structured factorisations whose composition still spans the head \citep{butterflyquant2025}.

\paragraph{Mixing support and block rotations.}
Confining a transform to blocks reduces cost at the expense of mixing support, and the resulting dependence on block size has been studied directly. DuQuant \citep{lin2024duquant} rotates within blocks and permutes channels so that outlier mass is balanced across them; GSR \citep{gsr2025} uses a training-free block-diagonal Walsh--Hadamard rotation, applied to the offline residual-stream transform rather than to the online $Q/K$ path, at W2A16 and W2A4; PeRQ \citep{sanjeet2026blockrotations} analyses block-Hadamard suppression non-asymptotically through the distribution of $\ell_1$ mass over blocks; and \citet{jia2026sawint4} adopt block-diagonal rotation with tokenwise INT4 for KV-cache serving. \S\ref{sec:mechanism} places the RoPE frequency pair at the low-support end of that axis for post-RoPE Q/K. Table~\ref{tab:related-axes} summarises the design axes of the nearest families.

\paragraph{Transform objectives.}
Many rotation designs optimise a variance or outlier proxy. OSTQuant \citep{hu2025ostquant} measures utilisation of the quantisation space and permits scaling alongside the orthogonal transform, while SQuat \citep{wang2025squat} constrains key-quantisation error relative to a query subspace. Closest to the present construction in analytic form, \citet{singlequant2025} derive closed-form Givens angles for W4A4 outlier smoothing under alignment and uniformity criteria, without a RoPE-pair or commutativity constraint. The present study instead derives and numerically verifies the optimum of a pooled-covariance surrogate for the implemented head-shared parameterisation, then evaluates the corresponding configurations under dynamic \WAKV{} quantisation.

\paragraph{RoPE-structured transforms and concurrent work.}
The per-pair RoPE-commuting structure predates this work. FPTQuant \citep{fptquant2025} folds a scaled per-pair map into $W_Q$ and $W_K$. ParoQuant \citep{paroquant2026} optimises Givens rotations with channel scaling over selected channel pairs in a weight-only quantisation procedure. \citet{iisr2026} propose spectral-energy rescaling and describe the per-pair RoPE-commuting form as a secondary extension, folded into $W_Q,W_K$ and evaluated at $1024$ tokens. \citet{comrope2025} study trainable commuting angle matrices as a positional-encoding design. Lemma~\ref{lem:centraliser-commutes} supplies the converse result: with distinct frequencies, every single-head orthogonal map that commutes with RoPE has the per-pair form. The empirical comparison evaluates the stricter head-shared subfamily rather than the complete head-dependent direct-product family. Appendix~\ref{app:data} also reports a restricted probe of a trained and scaled FPTQuant-style component within our evaluation procedure; this probe does not constitute a reproduction of FPTQuant.

\paragraph{Scaling and cache-side methods.}
Scaling reduces a peak without redistributing it across channels and is therefore complementary: SmoothQuant \citep{xiao2022smoothquant} and Outlier Suppression+ \citep{wei2023outlierplus} on the activation path, AWQ \citep{lin2024awq} on the weight side, and OmniQuant \citep{shao2024omniquant} on both. Closest in motivation to this work is Q-ROAR \citep{qroar2026,qroarsister2025}, which analyses RoPE channel bands under position interpolation and reports gains from per-band weight rescaling in quantised long-context models; it varies the scaling axis under interpolated positions, whereas we vary the rotation axis under native RoPE, so the null reported here does not bear on that result (\S\ref{sec:limitations}). On the cache side, KVQuant \citep{hooper2024kvquant} introduced pre-RoPE key quantisation, establishing that the post-RoPE stream is the harder one to quantise, and with KIVI \citep{liu2024kivi} sets the low-bit baselines; RotateKV \citep{su2025rotatekv} is the nearest rotation-based method, applying a pre-RoPE grouped-head rotation with channel reordering to a 2-bit cache; PolarQuant \citep{polarquant2026} and CommVQ \citep{commvq2025} act on the stored cache in a RoPE-pair-aware form, and TurboQuant \citep{turboquant2026} is a data-oblivious online vector quantiser.

\begin{table}[t]
\centering\small
\caption{Design axes of the nearest transform families. Mixing support is the number of channels $b$ over which a single peak can be redistributed. FPTQuant and ParoQuant compose the pairwise rotation with channel scaling, and DuQuant and PeRQ compose the block rotation with a permutation; the present work varies block size and mixing support without introducing channel scaling or permutation.}
\label{tab:related-axes}
\begin{tabular}{@{}lllll@{}}
\toprule
method & placement & mixing support & fitting & quantiser \\
\midrule
FPTQuant \citep{fptquant2025}          & offline $W_Q,W_K$    & RoPE pair, $b{=}2$    & trained    & static INT4 \\
ParoQuant \citep{paroquant2026}        & weight-only          & channel pair, $b{=}2$ & trained    & weight-only \\
ART\,/\,URT \citep{singlequant2025}    & offline, activation  & Givens pairs          & untrained  & W4A4 \\
IIS-R \citep{iisr2026}                 & offline, weight side & full head; per-pair   & calibrated & low bit width \\
DuQuant \citep{lin2024duquant}         & activation           & block, permuted       & calibrated & W4A4 \\
PeRQ \citep{sanjeet2026blockrotations} & activation           & block, permuted       & calibrated & INT4 \\
GSR \citep{gsr2025}                    & offline $R_1$        & block, $b{=}$ group   & untrained  & W2A16/W2A4 \\
RotateKV \citep{su2025rotatekv}        & pre-RoPE cache       & grouped heads         & calibrated & 2-bit KV \\
This work                              & online post-RoPE     & $b{=}2$ to $b{=}\dh$  & untrained  & dynamic \WAKV{} \\
\bottomrule
\end{tabular}
\end{table}

\section{Limitations and conclusion}
\label{sec:limitations}

\paragraph{Limitations.}
\begin{enumerate}[leftmargin=1.5em,itemsep=2pt,topsep=2pt]
\item \textbf{Quantiser and method scope.} The result concerns native-RoPE dynamic \WAKV{} with the evaluated calibration procedure. Learned rotations, channel scaling, other bit widths, static deployment, and positional scaling \citep{peng2023yarn} remain outside this scope. Q-ROAR \citep{qroar2026,qroarsister2025} studies per-band weight rescaling under position interpolation, whereas our evaluation varies rotations under native RoPE; the present result therefore does not evaluate the Q-ROAR setting. The evaluated pairwise parameterisation shares each layer--pair angle across attention heads. Head-dependent angles $\phi_{h,k}$ and transforms that mix across heads are untested. Among the evaluated head-shared estimators, K-only estimation yields lower perplexity than the default estimator but is not established as optimal. The surrogate optimum $\phistar$ is evaluated only on Llama-3.2-3B.
\item \textbf{Evaluation and statistical scope.} Perplexity is the primary outcome. Several NIAH comparisons are saturated near the FP16 result, and LongBench-v2 is non-discriminative at the evaluated quantisation setting. Most long-context evaluation points use three seeds. Every equivalence statement is conditional on the evaluated corpora, seeds, and selected $\pm0.05$-PPL interval criterion.
\item \textbf{Explanatory scope.} The support interpolation covers Llama-3.2-3B on WikiText-2 at short context, three seeds, and one block-Hadamard transformation path; it leaves the commuting family for $b>2$, and its layout and orientation controls use one seed. The static comparison is exploratory and does not identify a causal mechanism.
\end{enumerate}

\paragraph{Conclusion.}
RoPE defines a precise pairwise structure for online Q/K rotations. For a single head with distinct frequencies, the commuting orthogonal family consists of independent rotations within the RoPE frequency pairs, a structure established by prior work \citep{fptquant2025}; our converse result completes this single-head characterisation. For the implemented head-shared parameterisation, the pooled-covariance, position-averaged surrogate has a closed-form minimiser, and the implementation attains the associated surrogate minimum.

This objective-specific guarantee does not improve accuracy under dynamic \WAKV{} quantisation in the evaluated setting. Across the tested checkpoints and context lengths, replacing the full-head Hadamard with the evaluated head-shared pairwise configuration yields higher perplexity. K-only estimation reduces this gap without eliminating it. Along the evaluated support-interpolation path, K range, relative quantisation error, and perplexity degradation decrease as mixing support approaches the full-head endpoint. This pattern is consistent with mixing support contributing to the observed difference, but it does not identify a mechanism beyond that path. The evidence separates optimality for a pooled-covariance surrogate, improvement in the quantiser's tokenwise range statistic, and end-to-end model quality.

The broader implication is that structural alignment alone does not determine whether a rotation reduces quantisation error. A structured transform should be judged by whether its optimisation objective and its mixing support match the scale-setting rule of the quantiser in which it will be deployed.

\bibliography{main}

\begin{thebibliography}{53}
\providecommand{\natexlab}[1]{#1}
\providecommand{\url}[1]{\texttt{#1}}
\expandafter\ifx\csname urlstyle\endcsname\relax
  \providecommand{\doi}[1]{doi: #1}\else
  \providecommand{\doi}{doi: \begingroup \urlstyle{rm}\Url}\fi

\bibitem[Ahmadian et~al.(2023)Ahmadian, Dash, Chen, Venkitesh, Gou, Blunsom,
  {\"U}st{\"u}n, and Hooker]{ahmadian2023intriguing}
Arash Ahmadian, Saurabh Dash, Hongyu Chen, Bharat Venkitesh, Zhen~Stephen Gou,
  Phil Blunsom, Ahmet {\"U}st{\"u}n, and Sara Hooker.
\newblock Intriguing properties of quantization at scale.
\newblock In \emph{Advances in Neural Information Processing Systems 36
  ({NeurIPS} 2023)}, 2023.
\newblock arXiv:2305.19268.

\bibitem[Ainslie et~al.(2023)Ainslie, Lee-Thorp, de~Jong, Zemlyanskiy,
  Lebr{\'o}n, and Sanghai]{ainslie2023gqa}
Joshua Ainslie, James Lee-Thorp, Michiel de~Jong, Yury Zemlyanskiy, Federico
  Lebr{\'o}n, and Sumit Sanghai.
\newblock {GQA}: Training generalized multi-query transformer models from
  multi-head checkpoints.
\newblock In \emph{Proceedings of the 2023 Conference on Empirical Methods in
  Natural Language Processing (EMNLP)}, 2023.
\newblock arXiv:2305.13245.

\bibitem[Akhondzadeh et~al.(2025)Akhondzadeh, Bojchevski, Eleftheriou, and
  Dazzi]{kurtail2024}
Mohammad~Sadegh Akhondzadeh, Aleksandar Bojchevski, Evangelos Eleftheriou, and
  Martino Dazzi.
\newblock {KurTail}: Kurtosis-based {LLM} quantization.
\newblock In \emph{Findings of the Association for Computational Linguistics:
  {EMNLP} 2025}, 2025.
\newblock arXiv:2503.01483.

\bibitem[Ashkboos et~al.(2024{\natexlab{a}})Ashkboos, Croci, {Gennari do
  Nascimento}, Hoefler, and Hensman]{ashkboos2024slicegpt}
Saleh Ashkboos, Maximilian~L. Croci, Marcelo {Gennari do Nascimento}, Torsten
  Hoefler, and James Hensman.
\newblock {SliceGPT}: Compress large language models by deleting rows and
  columns.
\newblock In \emph{International Conference on Learning Representations
  (ICLR)}, 2024{\natexlab{a}}.
\newblock arXiv:2401.15024.

\bibitem[Ashkboos et~al.(2024{\natexlab{b}})Ashkboos, Mohtashami, Croci, Li,
  Cameron, Jaggi, Alistarh, Hoefler, and Hensman]{ashkboos2024quarot}
Saleh Ashkboos, Amirkeivan Mohtashami, Maximilian~L. Croci, Bo~Li, Pashmina
  Cameron, Martin Jaggi, Dan Alistarh, Torsten Hoefler, and James Hensman.
\newblock {QuaRot}: Outlier-free 4-bit inference in rotated {LLMs}.
\newblock In \emph{Advances in Neural Information Processing Systems
  (NeurIPS)}, 2024{\natexlab{b}}.
\newblock arXiv:2404.00456.

\bibitem[Azerbayev et~al.(2022)Azerbayev, Ayers, and
  Piotrowski]{azerbayev2023proofpile}
Zhangir Azerbayev, Edward Ayers, and Bartosz Piotrowski.
\newblock {Proof-Pile}: A pre-training dataset of mathematical texts.
\newblock Hoskinson Center for Formal Mathematics / EleutherAI, 2022.
\newblock URL \url{https://github.com/zhangir-azerbayev/proof-pile}.

\bibitem[Bai et~al.(2025)Bai, Tu, Zhang, Peng, Wang, Lv, Cao, Xu, Hou, Dong,
  Tang, and Li]{bai2024longbenchv2}
Yushi Bai, Shangqing Tu, Jiajie Zhang, Hao Peng, Xiaozhi Wang, Xin Lv, Shulin
  Cao, Jiazheng Xu, Lei Hou, Yuxiao Dong, Jie Tang, and Juanzi Li.
\newblock {LongBench v2}: Towards deeper understanding and reasoning on
  realistic long-context multitasks.
\newblock In \emph{Proceedings of the 63rd Annual Meeting of the Association
  for Computational Linguistics (Volume 1: Long Papers)}, pp.\  3639--3664,
  2025.
\newblock \doi{10.18653/v1/2025.acl-long.183}.
\newblock arXiv:2412.15204.

\bibitem[Chee et~al.(2023)Chee, Cai, Kuleshov, and De~Sa]{chee2023quip}
Jerry Chee, Yaohui Cai, Volodymyr Kuleshov, and Christopher De~Sa.
\newblock {QuIP}: 2-bit quantization of large language models with guarantees.
\newblock In \emph{Advances in Neural Information Processing Systems
  (NeurIPS)}, 2023.
\newblock arXiv:2307.13304.

\bibitem[Chen et~al.(2026{\natexlab{a}})Chen, Egiazarian, Castro, Hoefler, and
  Alistarh]{wush2025}
Jiale Chen, Vage Egiazarian, Roberto~L. Castro, Torsten Hoefler, and Dan
  Alistarh.
\newblock {WUSH}: Near-optimal adaptive transforms for {LLM} quantization.
\newblock In \emph{International Conference on Machine Learning (ICML)},
  2026{\natexlab{a}}.
\newblock arXiv:2512.00956.

\bibitem[Chen et~al.(2026{\natexlab{b}})Chen, Liu, Wang, Bin, Shao, and
  Luo]{prefixquant2025}
Mengzhao Chen, Yi~Liu, Jiahao Wang, Yi~Bin, Wenqi Shao, and Ping Luo.
\newblock {PrefixQuant}: Eliminating outliers by prefixed tokens for large
  language models quantization.
\newblock \emph{IEEE Transactions on Pattern Analysis and Machine
  Intelligence}, 2026{\natexlab{b}}.
\newblock arXiv:2410.05265.

\bibitem[Choi et~al.(2025)Choi, Song, Lim, and Yoo]{gsr2025}
Euntae Choi, Sumin Song, Woosang Lim, and Sungjoo Yoo.
\newblock Grouped sequency-arranged rotation: Optimizing rotation
  transformation for quantization for free.
\newblock In \emph{Proceedings of the 63rd Annual Meeting of the Association
  for Computational Linguistics (ACL), Student Research Workshop}, 2025.
\newblock arXiv:2505.03810.

\bibitem[Dettmers et~al.(2022)Dettmers, Lewis, Belkada, and
  Zettlemoyer]{dettmers2022llmint8}
Tim Dettmers, Mike Lewis, Younes Belkada, and Luke Zettlemoyer.
\newblock {LLM.int8()}: 8-bit matrix multiplication for transformers at scale.
\newblock In \emph{Advances in Neural Information Processing Systems
  (NeurIPS)}, volume~35, pp.\  30318--30332. Curran Associates, Inc., 2022.
\newblock \doi{10.52202/068431-2198}.
\newblock arXiv:2208.07339.

\bibitem[Frantar et~al.(2023)Frantar, Ashkboos, Hoefler, and
  Alistarh]{frantar2022gptq}
Elias Frantar, Saleh Ashkboos, Torsten Hoefler, and Dan Alistarh.
\newblock {GPTQ}: Accurate post-training quantization for generative
  pre-trained transformers.
\newblock In \emph{International Conference on Learning Representations
  (ICLR)}, 2023.
\newblock arXiv:2210.17323.

\bibitem[Grattafiori et~al.(2024)Grattafiori, Dubey, Jauhri, Pandey, Kadian,
  Al-Dahle, Letman, Mathur, Schelten, Yang, Fan, et~al.]{dubey2024llama3}
Aaron Grattafiori, Abhimanyu Dubey, Abhinav Jauhri, Abhinav Pandey, Abhishek
  Kadian, Ahmad Al-Dahle, Aiesha Letman, Akhil Mathur, Alan Schelten, Amy Yang,
  Angela Fan, et~al.
\newblock The {Llama} 3 herd of models.
\newblock \emph{arXiv preprint arXiv:2407.21783}, 2024.

\bibitem[Hooper et~al.(2024)Hooper, Kim, Mohammadzadeh, Mahoney, Shao, Keutzer,
  and Gholami]{hooper2024kvquant}
Coleman Hooper, Sehoon Kim, Hiva Mohammadzadeh, Michael~W. Mahoney,
  Yakun~Sophia Shao, Kurt Keutzer, and Amir Gholami.
\newblock {KVQuant}: Towards 10 million context length {LLM} inference with
  {KV} cache quantization.
\newblock In \emph{Advances in Neural Information Processing Systems 37
  ({NeurIPS} 2024)}, 2024.
\newblock arXiv:2401.18079.

\bibitem[Hsieh et~al.(2024)Hsieh, Sun, Kriman, Acharya, Rekesh, Jia, Zhang, and
  Ginsburg]{hsieh2024ruler}
Cheng-Ping Hsieh, Simeng Sun, Samuel Kriman, Shantanu Acharya, Dima Rekesh, Fei
  Jia, Yang Zhang, and Boris Ginsburg.
\newblock {RULER}: What's the real context size of your long-context language
  models?
\newblock In \emph{Conference on Language Modeling ({COLM})}, 2024.
\newblock arXiv:2404.06654.

\bibitem[Hu et~al.(2025)Hu, Cheng, Yang, Chen, Xu, Yu, Xu, Yuan, Jiang, and
  Zhou]{hu2025ostquant}
Xing Hu, Yuan Cheng, Dawei Yang, Zhixuan Chen, Zukang Xu, Jiangyong Yu, Chen
  Xu, Zhihang Yuan, Zhe Jiang, and Sifan Zhou.
\newblock {OSTQuant}: Refining large language model quantization with
  orthogonal and scaling transformations for better distribution fitting.
\newblock In \emph{International Conference on Learning Representations
  (ICLR)}, 2025.
\newblock arXiv:2501.13987.

\bibitem[Huang et~al.(2021)Huang, Cao, Parulian, Ji, and
  Wang]{huang2021govreport}
Luyang Huang, Shuyang Cao, Nikolaus Parulian, Heng Ji, and Lu~Wang.
\newblock Efficient attentions for long document summarization.
\newblock In \emph{Proceedings of the 2021 Conference of the North American
  Chapter of the Association for Computational Linguistics: Human Language
  Technologies (NAACL-HLT)}, 2021.
\newblock arXiv:2104.02112.

\bibitem[Jia et~al.(2026)Jia, Li, Zhou, Heo, Wang, Dao, Song, Athiwaratkun, Xu,
  Zhang, and Wu]{jia2026sawint4}
Jinda Jia, Jisen Li, Zhongzhu Zhou, Jung~Hwan Heo, Jue Wang, Tri Dao,
  Shuaiwen~Leon Song, Ben Athiwaratkun, Chenfeng Xu, Tianyi Zhang, and Xiaoxia
  Wu.
\newblock {SAW-INT4}: System-aware 4-bit {KV}-cache quantization for real-world
  {LLM} serving.
\newblock \emph{arXiv preprint arXiv:2604.19157}, 2026.

\bibitem[Jiang et~al.(2023)Jiang, Sablayrolles, Mensch, Bamford, Chaplot,
  de~las Casas, Bressand, Lengyel, Lample, Saulnier, Lavaud, Lachaux, Stock,
  Le~Scao, Lavril, Wang, Lacroix, and El~Sayed]{jiang2023mistral}
Albert~Q. Jiang, Alexandre Sablayrolles, Arthur Mensch, Chris Bamford,
  Devendra~Singh Chaplot, Diego de~las Casas, Florian Bressand, Gianna Lengyel,
  Guillaume Lample, Lucile Saulnier, L{\'e}lio~Renard Lavaud, Marie-Anne
  Lachaux, Pierre Stock, Teven Le~Scao, Thibaut Lavril, Thomas Wang,
  Timoth{\'e}e Lacroix, and William El~Sayed.
\newblock {Mistral} 7{B}.
\newblock \emph{arXiv preprint arXiv:2310.06825}, 2023.

\bibitem[Lakens(2017)]{lakens2017equivalence}
Dani{\"e}l Lakens.
\newblock Equivalence tests: A practical primer for {t} tests, correlations,
  and meta-analyses.
\newblock \emph{Social Psychological and Personality Science}, 8\penalty0
  (4):\penalty0 355--362, 2017.
\newblock \doi{10.1177/1948550617697177}.

\bibitem[Li et~al.(2025)Li, Zhang, Hassan, Chafekar, Cai, Ren, Guo, Karimzadeh,
  Reed, Wang, and Gan]{commvq2025}
Junyan Li, Yang Zhang, Muhammad~Yusuf Hassan, Talha Chafekar, Tianle Cai, Zhile
  Ren, Pengsheng Guo, Foroozan Karimzadeh, Colorado Reed, Chong Wang, and
  Chuang Gan.
\newblock {CommVQ}: Commutative vector quantization for {KV} cache compression.
\newblock In \emph{International Conference on Machine Learning (ICML)}, 2025.
\newblock arXiv:2506.18879.

\bibitem[Liang et~al.(2026)Liang, Chen, Zhang, Han, and Liu]{paroquant2026}
Yesheng Liang, Haisheng Chen, Zihan Zhang, Song Han, and Zhijian Liu.
\newblock {ParoQuant}: Pairwise rotation quantization for efficient reasoning
  {LLM} inference.
\newblock In \emph{International Conference on Learning Representations
  (ICLR)}, 2026.
\newblock arXiv:2511.10645.

\bibitem[Lin et~al.(2024{\natexlab{a}})Lin, Xu, Wu, Cui, Zhang, Mou, Song, Sun,
  and Wei]{lin2024duquant}
Haokun Lin, Haobo Xu, Yichen Wu, Jingzhi Cui, Yingtao Zhang, Linzhan Mou, Linqi
  Song, Zhenan Sun, and Ying Wei.
\newblock {DuQuant}: Distributing outliers via dual transformation makes
  stronger quantized {LLMs}.
\newblock In \emph{Advances in Neural Information Processing Systems
  (NeurIPS)}, 2024{\natexlab{a}}.
\newblock Oral; arXiv:2406.01721.

\bibitem[Lin et~al.(2024{\natexlab{b}})Lin, Tang, Tang, Yang, Chen, Wang, Xiao,
  Dang, Gan, and Han]{lin2024awq}
Ji~Lin, Jiaming Tang, Haotian Tang, Shang Yang, Wei-Ming Chen, Wei-Chen Wang,
  Guangxuan Xiao, Xingyu Dang, Chuang Gan, and Song Han.
\newblock {AWQ}: Activation-aware weight quantization for {LLM} compression and
  acceleration.
\newblock In \emph{Proceedings of Machine Learning and Systems (MLSys)},
  2024{\natexlab{b}}.
\newblock arXiv:2306.00978.

\bibitem[Liu et~al.(2025)Liu, Zhao, Fedorov, Soran, Choudhary, Krishnamoorthi,
  Chandra, Tian, and Blankevoort]{liu2024spinquant}
Zechun Liu, Changsheng Zhao, Igor Fedorov, Bilge Soran, Dhruv Choudhary,
  Raghuraman Krishnamoorthi, Vikas Chandra, Yuandong Tian, and Tijmen
  Blankevoort.
\newblock {SpinQuant}: {LLM} quantization with learned rotations.
\newblock In \emph{International Conference on Learning Representations
  (ICLR)}, 2025.
\newblock arXiv:2405.16406; OpenReview: ogO6DGE6FZ.

\bibitem[Liu et~al.(2024)Liu, Yuan, Jin, Zhong, Xu, Braverman, Chen, and
  Hu]{liu2024kivi}
Zirui Liu, Jiayi Yuan, Hongye Jin, Shaochen Zhong, Zhaozhuo Xu, Vladimir
  Braverman, Beidi Chen, and Xia Hu.
\newblock {KIVI}: A tuning-free asymmetric 2bit quantization for {KV} cache.
\newblock In \emph{Proceedings of the 41st International Conference on Machine
  Learning (ICML)}, pp.\  32332--32344, 2024.
\newblock arXiv:2402.02750.

\bibitem[Merity et~al.(2017)Merity, Xiong, Bradbury, and
  Socher]{merity2016wikitext}
Stephen Merity, Caiming Xiong, James Bradbury, and Richard Socher.
\newblock Pointer sentinel mixture models.
\newblock In \emph{International Conference on Learning Representations
  ({ICLR})}, 2017.
\newblock arXiv:1609.07843.

\bibitem[{Meta}(2024{\natexlab{a}})]{meta2024llama32_1b}
{Meta}.
\newblock {Llama 3.2 1B} model card.
\newblock Hugging Face model card, 2024{\natexlab{a}}.
\newblock URL \url{https://huggingface.co/meta-llama/Llama-3.2-1B}.

\bibitem[{Meta}(2024{\natexlab{b}})]{meta2024llama32_3b}
{Meta}.
\newblock {Llama 3.2 3B} model card.
\newblock Hugging Face model card, 2024{\natexlab{b}}.
\newblock URL \url{https://huggingface.co/meta-llama/Llama-3.2-3B}.

\bibitem[{Mistral AI}(2024)]{mistralai2024mistral7bv03}
{Mistral AI}.
\newblock {Mistral-7B-v0.3} model card.
\newblock Hugging Face model card, 2024.
\newblock URL \url{https://huggingface.co/mistralai/Mistral-7B-v0.3}.

\bibitem[Pavlov(2026)]{iisr2026}
Gorgi Pavlov.
\newblock Influence-inspired spectral rotations for extreme low-bit {LLM}
  quantization.
\newblock \emph{arXiv preprint arXiv:2605.25203}, 2026.

\bibitem[Peng et~al.(2024)Peng, Quesnelle, Fan, and Shippole]{peng2023yarn}
Bowen Peng, Jeffrey Quesnelle, Honglu Fan, and Enrico Shippole.
\newblock {YaRN}: Efficient context window extension of large language models.
\newblock In \emph{International Conference on Learning Representations
  ({ICLR})}, 2024.
\newblock arXiv:2309.00071.

\bibitem[Qiao \& Huang(2026)Qiao and Huang]{qroar2026}
Ye~Qiao and Sitao Huang.
\newblock {Q-ROAR}: Outlier-aware rescaling for {RoPE} position interpolation
  in quantized long-context {LLMs}.
\newblock In \emph{AAAI Conference on Artificial Intelligence (AAAI), Student
  Abstract Track}, 2026.
\newblock arXiv:2509.14391.

\bibitem[Qiao et~al.(2025)Qiao, Xu, Zhang, and Huang]{qroarsister2025}
Ye~Qiao, Haocheng Xu, Xiaofan Zhang, and Sitao Huang.
\newblock Rethinking {RoPE} scaling in quantized {LLM}: Theory, outlier, and
  channel-band analysis with weight rescaling.
\newblock \emph{arXiv preprint arXiv:2510.00028}, 2025.

\bibitem[Rae et~al.(2020)Rae, Potapenko, Jayakumar, Hillier, and
  Lillicrap]{rae2019pg19}
Jack~W. Rae, Anna Potapenko, Siddhant~M. Jayakumar, Chloe Hillier, and
  Timothy~P. Lillicrap.
\newblock Compressive transformers for long-range sequence modelling.
\newblock In \emph{International Conference on Learning Representations
  ({ICLR})}, 2020.
\newblock arXiv:1911.05507; introduces the PG-19 long-range language-modelling
  benchmark.

\bibitem[Sanjeet et~al.(2026)Sanjeet, Colbert, Monteagudo-Lago, Franco,
  Umuroglu, and Fraser]{sanjeet2026blockrotations}
Sai Sanjeet, Ian Colbert, Pablo Monteagudo-Lago, Giuseppe Franco, Yaman
  Umuroglu, and Nicholas~J. Fraser.
\newblock Pushing the limits of block rotations in post-training quantization.
\newblock \emph{arXiv preprint arXiv:2601.22347}, 2026.

\bibitem[Shao et~al.(2024)Shao, Chen, Zhang, Xu, Zhao, Li, Zhang, Gao, Qiao,
  and Luo]{shao2024omniquant}
Wenqi Shao, Mengzhao Chen, Zhaoyang Zhang, Peng Xu, Lirui Zhao, Zhiqian Li,
  Kaipeng Zhang, Peng Gao, Yu~Qiao, and Ping Luo.
\newblock {OmniQuant}: Omnidirectionally calibrated quantization for large
  language models.
\newblock In \emph{International Conference on Learning Representations
  (ICLR)}, 2024.
\newblock arXiv:2308.13137.

\bibitem[Shao et~al.(2025)Shao, Chen, Wang, Yu, Lin, Yao, Wei, and
  Cheng]{dartquant2025}
Yuantian Shao, Yuanteng Chen, Peisong Wang, Jianlin Yu, Jing Lin, Yiwu Yao,
  Zhihui Wei, and Jian Cheng.
\newblock {DartQuant}: Efficient rotational distribution calibration for {LLM}
  quantization.
\newblock In \emph{Advances in Neural Information Processing Systems
  (NeurIPS)}, 2025.
\newblock arXiv:2511.04063.

\bibitem[Su et~al.(2024)Su, Ahmed, Lu, Pan, Wen, and Liu]{su2021rope}
Jianlin Su, Murtadha Ahmed, Yu~Lu, Shengfeng Pan, Bo~Wen, and Yunfeng Liu.
\newblock {RoFormer}: Enhanced transformer with rotary position embedding.
\newblock \emph{Neurocomputing}, 568:\penalty0 127063, 2024.
\newblock \doi{10.1016/j.neucom.2023.127063}.
\newblock arXiv:2104.09864.

\bibitem[Su et~al.(2025)Su, Chen, Shen, Wei, Li, Yu, and Yuan]{su2025rotatekv}
Zunhai Su, Zhe Chen, Wang Shen, Hanyu Wei, Linge Li, Huangqi Yu, and Kehong
  Yuan.
\newblock {RotateKV}: Accurate and robust 2-bit {KV} cache quantization for
  {LLMs} via outlier-aware adaptive rotations.
\newblock In \emph{Proceedings of the International Joint Conference on
  Artificial Intelligence (IJCAI)}, 2025.
\newblock arXiv:2501.16383.

\bibitem[Sun et~al.(2025)Sun, Liu, Bai, Bao, Zhao, Li, Hu, Yu, Hou, Yuan,
  Jiang, Liu, and Yao]{flatquant2025}
Yuxuan Sun, Ruikang Liu, Haoli Bai, Han Bao, Kang Zhao, Yuening Li, Jiaxin Hu,
  Xianzhi Yu, Lu~Hou, Chun Yuan, Xin Jiang, Wulong Liu, and Jun Yao.
\newblock {FlatQuant}: Flatness matters for {LLM} quantization.
\newblock In \emph{International Conference on Machine Learning (ICML)}, 2025.
\newblock arXiv:2410.09426.

\bibitem[Tseng et~al.(2024)Tseng, Chee, Sun, Kuleshov, and
  De~Sa]{tseng2024quipsharp}
Albert Tseng, Jerry Chee, Qingyao Sun, Volodymyr Kuleshov, and Christopher
  De~Sa.
\newblock {QuIP\#}: Even better {LLM} quantization with {Hadamard} incoherence
  and lattice codebooks.
\newblock In \emph{International Conference on Machine Learning (ICML)}, 2024.
\newblock arXiv:2402.04396.

\bibitem[van Breugel et~al.(2025)van Breugel, Bondarenko, Whatmough, and
  Nagel]{fptquant2025}
Boris van Breugel, Yelysei Bondarenko, Paul Whatmough, and Markus Nagel.
\newblock {FPTQuant}: Function-preserving transforms for {LLM} quantization.
\newblock In \emph{ICML 2025 Workshop on Efficient Systems for Foundation
  Models (ES-FoMo-III)}, 2025.
\newblock Oral; arXiv:2506.04985. Subsequently published at ICML 2026 (main
  conference).

\bibitem[Wang et~al.(2025)Wang, Han, Xu, and Srivastava]{wang2025squat}
Hao Wang, Ligong Han, Kai Xu, and Akash Srivastava.
\newblock {SQuat}: Subspace-orthogonal {KV} cache quantization.
\newblock In \emph{Conference on Language Modeling ({COLM})}, 2025.
\newblock arXiv:2503.24358.

\bibitem[Wei et~al.(2023)Wei, Zhang, Li, Zhang, Gong, Guo, and
  Liu]{wei2023outlierplus}
Xiuying Wei, Yunchen Zhang, Yuhang Li, Xiangguo Zhang, Ruihao Gong, Jinyang
  Guo, and Xianglong Liu.
\newblock Outlier suppression+: Accurate quantization of large language models
  by equivalent and effective shifting and scaling.
\newblock In \emph{Proceedings of the 2023 Conference on Empirical Methods in
  Natural Language Processing (EMNLP)}, pp.\  1648--1665, 2023.
\newblock arXiv:2304.09145.

\bibitem[Wu et~al.(2025)Wu, Lv, Feng, Zhang, Zhang, Yin, Lin, and
  Yan]{polarquant2026}
Songhao Wu, Ang Lv, Xiao Feng, Yufei Zhang, Xun Zhang, Guojun Yin, Wei Lin, and
  Rui Yan.
\newblock {PolarQuant}: Leveraging polar transformation for efficient key cache
  quantization and decoding acceleration.
\newblock In \emph{Advances in Neural Information Processing Systems
  (NeurIPS)}, 2025.
\newblock arXiv:2502.00527.

\bibitem[Xiang \& Zhang(2025)Xiang and Zhang]{xiang2024dfrot}
Jingyang Xiang and Sai~Qian Zhang.
\newblock {DFRot}: Achieving outlier-free and massive activation-free for
  rotated {LLMs} with refined rotation.
\newblock In \emph{Conference on Language Modeling (COLM)}, 2025.
\newblock arXiv:2412.00648.

\bibitem[Xiao et~al.(2023)Xiao, Lin, Seznec, Wu, Demouth, and
  Han]{xiao2022smoothquant}
Guangxuan Xiao, Ji~Lin, Mickael Seznec, Hao Wu, Julien Demouth, and Song Han.
\newblock {SmoothQuant}: Accurate and efficient post-training quantization for
  large language models.
\newblock In \emph{Proceedings of the 40th International Conference on Machine
  Learning (ICML)}, pp.\  38087--38099, 2023.
\newblock arXiv:2211.10438.

\bibitem[Xiao et~al.(2025)Xiao, Ji, Li, Liu, Ma, Wang, Li, Zhong, Xie, Tashi,
  and Yu]{singlequant2025}
Jinying Xiao, Bin Ji, Shasha Li, Xiaodong Liu, Jun Ma, Chao Wang, Wei Li,
  Ye~Zhong, Xuan Xie, Nyima Tashi, and Jie Yu.
\newblock Outlier smoothing with closed-form rotations for {W4A4} large
  language model quantization.
\newblock \emph{arXiv preprint arXiv:2511.22316}, 2025.

\bibitem[Xu et~al.(2025)Xu, Dong, Elachqar, and Shang]{butterflyquant2025}
Bingxin Xu, Zhen Dong, Oussama Elachqar, and Yuzhang Shang.
\newblock {ButterflyQuant}: Ultra-low-bit {LLM} quantization through learnable
  orthogonal butterfly transforms.
\newblock \emph{arXiv preprint arXiv:2509.09679}, 2025.

\bibitem[Yu et~al.(2025)Yu, Jiang, Jia, Yan, Liu, Qian, Li, Dong, and
  Yuan]{comrope2025}
Hao Yu, Tangyu Jiang, Shuning Jia, Shannan Yan, Shunning Liu, Haolong Qian,
  Guanghao Li, Shuting Dong, and Chun Yuan.
\newblock {ComRoPE}: Scalable and robust rotary position embedding
  parameterized by trainable commuting angle matrices.
\newblock In \emph{Proceedings of the IEEE/CVF Conference on Computer Vision
  and Pattern Recognition (CVPR)}, pp.\  4508--4517, 2025.
\newblock arXiv:2506.03737.

\bibitem[Zandieh et~al.(2026)Zandieh, Daliri, Hadian, and
  Mirrokni]{turboquant2026}
Amir Zandieh, Majid Daliri, Majid Hadian, and Vahab Mirrokni.
\newblock {TurboQuant}: Online vector quantization with near-optimal distortion
  rate.
\newblock In \emph{International Conference on Learning Representations
  (ICLR)}, 2026.
\newblock arXiv:2504.19874.

\end{thebibliography}
\bibliographystyle{plainnat}

\appendix
\section{Reproducibility}
\label{app:reproducibility}
The commands, paths, and identifiers below describe the companion code artifact. Records retained only in the execution environment are identified explicitly; reproducing from raw outputs requires those records in addition to the compact artifact.

\subsection*{A.1 Common implementation and quantisation configuration}
All results are produced with a DartQuant-style \WAKV{} post-training quantisation
procedure \citep{dartquant2025}. The procedure applies four rotations, following the
$R_1$--$R_4$ taxonomy of \citet{ashkboos2024quarot} and \citet{liu2024spinquant}: an
offline residual-stream rotation $R_1$, an offline value/output-projection rotation
$R_2$, an online post-RoPE $Q/K$ rotation $R_3$, and an online down-projection
Hadamard transform $R_4$. Weights are quantised with GPTQ \citep{frantar2022gptq}. Unless a static-scale control is explicitly specified,
activations and the KV cache use asymmetric per-token dynamic quantisation, with
scales recomputed on every forward pass.

The RoPE-aligned $R_3$ transformation is integrated at the calibration-statistics
stage and at the online $Q/K$ transformation stage. In all reported comparisons,
$R_4$ remains a Hadamard transform. The default configuration is Llama-3.2-3B with
head dimension $\dh=128$, corresponding to $64$ RoPE frequency pairs. Calibration
uses $128$ WikiText-2 sequences of length $2048$. Unless otherwise stated, the
reference seed is $0$ (\texttt{--seed 0}). All quantised evaluations use
\texttt{--quantizer\_type int4 --w\_bits 4 --a\_bits 4 --k\_bits 4 --v\_bits 4}.

\paragraph{RoPE channel layout.}
The evaluated checkpoints use the Hugging Face \texttt{rotate\_half} layout, in which
RoPE frequency pair $k$ occupies channels $\{k,\,k+\dh/2\}$, not the adjacent pair
$\{2k{-}1,2k\}$ used in the statements of \S\ref{sec:theory}. The two layouts differ by
a fixed permutation of channels, and all per-pair statements are invariant under it, but
the implementation must apply the rotation in the layout the checkpoint actually uses.
It does so throughout: the band rotation is applied across the two halves of the channel
axis in \path{dartquant_v2/equalising_r3.py}, and the runtime collector
(\path{dartquant_v2/equalising_r3_runtime.py}) forms its covariances under the same
$(k,\,k+\dh/2)$ pairing. The submitted head-resolved calibration summary is
\path{docs/data_db/paper/table1_rho.csv}; the artifact verifier checks its complete
three-decimal paper values but does not reconstruct it from omitted raw activations. For the mixing-support interpolation
(\S A.7), \path{dartquant_v2/block_hadamard_r3.py} applies an explicit permutation
$2k+s\mapsto k+s\,\dh/2$ before forming blocks, so that the $b=2$ block spans the two
channels of one frequency pair. Cutting blocks directly on raw half-half channel indices
would instead pair channels belonging to two \emph{different} bands, which would silently
break the $b=2$ endpoint's membership in the commuting family; the permutation exists to
prevent that.

The reader-facing transformation names correspond to the following code identifiers:
\begin{itemize}[leftmargin=1.4em,itemsep=1pt,topsep=2pt]
  \item \textbf{full-head Hadamard}: \texttt{hadamard}, with no equalising-$R_3$ flags;
  \item \textbf{pairwise-only}: \texttt{eq\_only}, using
        \texttt{--equalising\_r3 --equalising\_r3\_no\_hadamard};
  \item \textbf{pairwise\,+\,Hadamard}: \texttt{eq\_layered} or \texttt{eq\_r3}, using
        \texttt{--equalising\_r3} without \texttt{--equalising\_r3\_no\_hadamard}.
\end{itemize}
The variance objective is always selected explicitly with
\texttt{--equalising\_r3\_objective variance} in the corresponding configurations,
because the implementation otherwise defaults to \texttt{linf}.

The primary Llama-3.2-3B evaluations use NVIDIA A100 nodes with $40$ or $80$\,GB of
memory. The cross-model evaluations and the separate six-seed Llama-3.2-3B evaluation
use H100/H200-class nodes with A100 fallback. Both hardware pools use
\texttt{torch} 2.6.0+cu124 and \texttt{transformers} 4.49.0. The A100 evaluations
verify \texttt{fast-hadamard-transform} 1.1.0; the second pool uses a source build
whose version tag is not confirmed. Each result directory records the source
revision, software versions, and GPU model in \texttt{env.json}.

\subsection*{A.2 Implementation of the RoPE-aligned $Q/K$ rotation}

\paragraph{A.2.1 Calibration statistics and angle estimation.}
With $R_1$ and $R_2$ folded into the weights, one calibration forward pass records the outputs of $q_{\mathrm{proj}}$ and $k_{\mathrm{proj}}$ in every layer. For each stream and frequency pair, the collector folds the calibration-sample, sequence-position, and attention-head axes into one observation axis and accumulates $n$, $\sum a$, $\sum b$, $\sum a^2$, $\sum b^2$, and $\sum ab$ in CPU \texttt{float64}. Under the default rule, the Q and K totals are added elementwise and the covariance is then formed once as
\[
\widehat{\Sigma}_k
=\frac1n\sum_{i=1}^{n}x_i x_i^\top-\bar{x}\bar{x}^\top,
\qquad x_i=(a_i,b_i)^\top.
\]
This operation computes the covariance of the concatenated Q/K observation rows. It is not an arithmetic average of separately centred head-wise or stream-wise covariance matrices; when group means differ, the common centring operation retains the corresponding between-group contribution. Because each physical head contributes the same number of calibration positions, the default Q:K observation ratios are $3{:}1$ for Llama-3.2-3B and $4{:}1$ for the other evaluated checkpoints. The procedure produces one covariance matrix for every layer and frequency pair.

The pooling rule is selected by \texttt{--equalising\_r3\_pool}:
\begin{itemize}[leftmargin=1.4em,itemsep=1pt,topsep=2pt]
  \item \texttt{rows} adds the raw Q and K sufficient statistics and then centres once, equivalently computing the covariance of their concatenated observation rows;
  \item \texttt{k\_only} constructs the covariance from K observations alone;
  \item \texttt{balanced} averages the per-stream first and uncentred second moments with equal stream weights and then centres once. It is therefore not, in general, the arithmetic mean of the two centred stream covariances.
\end{itemize}
This option changes only the statistic used to estimate the angle. Every configuration
still applies one shared angle to $Q$ and $K$ and therefore satisfies
Proposition~\ref{prop:rope-commutativity-requires-shared-theta}; the estimator is
compared in App.~\ref{app:data},~\S L.

For the default variance objective, the implementation uses the closed-form Jacobi
equalising angle
$\hat\phi_k=\tfrac12\,\mathrm{atan2}\!\left(\sigma_{1,k}^2-\sigma_{2,k}^2,\,
2\rho_k\sigma_{1,k}\sigma_{2,k}\right)$,
which is Eq.~\ref{eq:phistar} evaluated at $(C_k,S_k)=(1,0)$. The alternative objectives
\texttt{linf} (max-channel $L_\infty$), \texttt{l2max} (sum of per-channel max-squared),
and \texttt{quantile} do not have a closed form in the implementation and are evaluated
over a $200$-point discretisation of $(-\pi/4,\pi/4]$. Candidate observations are drawn
from a per-pair top-$k$ set ranked by $|a|+|b|$.

The stored angle array has shape $(\text{n\_layers},\,\dh/2)=(\text{n\_layers},64)$ and is broadcast across attention heads and sequence positions. Thus $\phi_{\ell,h,k}=\phi_{\ell,k}$ for every head $h$. The implementation therefore evaluates a head-shared subset of the within-head commuting family and does not estimate head-dependent angles.

\paragraph{A.2.2 Online application.}
The implementation applies the pairwise $2{\times}2$ rotation $G(\hat\phi_k)$
immediately after RoPE produces $(q,k)$. This is implemented through a runtime
modification of the $Q/K$ processing module. The same per-pair angle is applied to both
streams. Distinct angles $\phi_k^{Q}\not\equiv\phi_k^{K}\pmod{2\pi}$ would leave
attention with a non-cancelling $G(\phi_k^{Q}-\phi_k^{K})$ factor
(Proposition~\ref{prop:rope-commutativity-requires-shared-theta}) and therefore would
not preserve the original attention geometry.

The online application occurs after the KV-cache quantiser has been configured. This
placement is identical across all compared configurations except for the choice of
$R_3$. The offline rotations, GPTQ Hessian, calibration sample, and random seed are
held fixed within every paired comparison.

\subsection*{A.3 Short-context evaluation protocol}
Short-context perplexity is evaluated on WikiText-2 under dynamic \WAKV{} quantisation.
Within each model and seed, the compared transformations use the same offline
rotations, GPTQ Hessian, calibration sequences, and random seed. Reported differences
are therefore paired within seed and isolate the online $R_3$ transformation.

The reference invocation below compares the full-head Hadamard and
pairwise\,+\,Hadamard configurations for seeds $0$--$2$.
\begin{verbatim}
for SEED in 0 1 2; do
  # Full-head Hadamard baseline
  python dartquant_v2/run_quantize.py \
      --model /path/to/meta-Llama-3.2-3B \
      --loss whip --quantizer_type int4 --cal_dataset wikitext2 \
      --w_bits 4 --a_bits 4 --k_bits 4 --v_bits 4 \
      --seed $SEED --ppl_eval --ppl_eval_dataset wikitext2 \
      --output_dir out/multiseed/whip/hadamard_seed$SEED
  # Pairwise + Hadamard
  python dartquant_v2/run_quantize.py \
      --model /path/to/meta-Llama-3.2-3B \
      --loss whip --quantizer_type int4 --cal_dataset wikitext2 \
      --w_bits 4 --a_bits 4 --k_bits 4 --v_bits 4 \
      --seed $SEED --ppl_eval --ppl_eval_dataset wikitext2 \
      --equalising_r3 --equalising_r3_objective variance \
      --output_dir out/multiseed/whip/eq_r3_seed$SEED
done
\end{verbatim}
The scheduler-neutral script
\path{scripts/reproduction/run_extended_short_context.sh} records this protocol and
prints its commands by default.
Paired differences are computed as $\Delta(\texttt{eq\_r3}-\texttt{hadamard})$ within
seed, and the reported standard deviations use the sample ($n{-}1$) convention. The
full-head Hadamard baseline varies by approximately $0.1$ PPL across seeds in
representative long-context evaluations (the 2K cross-seed sample standard deviation is
$0.082$--$0.289$ across the reported corpora, Table~\ref{tab:data-crossover}), which
motivates the use of within-seed comparisons for differences of a few hundredths of a
perplexity point (the largest per-seed 3B \textsc{whip} difference in
Table~\ref{tab:data-null} is $+0.0168$).

The cross-model short-context evaluation applies the same protocol to Llama-3.2-1B,
Llama-3.1-8B, and Mistral-7B-v0.3 by changing \texttt{--model}. The submitted compact
view \path{docs/data_db/paper/table2_null_short.csv} records the seed-specific
differences and an explicitly named sample-standard-deviation field. The verifier
recomputes that field from the three displayed seeds. Sample standard deviations use
the $n{-}1$ convention throughout the manuscript.

The additional nine-seed Llama-3.2-1B comparison uses
\path{scripts/reproduction/run_extended_short_context.sh} and evaluates the
full-head Hadamard and pairwise\,+\,Hadamard configurations for seeds $0$--$8$. The
angle-objective comparison uses
\path{scripts/reproduction/run_angle_objectives.sh}.
Comparisons that use references produced under a different source state are identified
explicitly as such in Table~\ref{tab:data-objectives} and are not used for equivalence
conclusions.

\subsection*{A.4 Long-context evaluation protocol}
Long-context perplexity is evaluated with the concatenated-stream protocol on Proof-Pile
and PG19. The compared transformations are full-head Hadamard, pairwise-only, and
pairwise\,+\,Hadamard. Within each model, seed, corpus, and context length, all three
configurations score identical token sequences, permitting paired perplexity
comparisons.

The standard context-length set is $2048$, $4096$, $8192$, $16384$, $32768$, and
$65536$ tokens. Llama-3.2-3B and Llama-3.2-1B are additionally evaluated at $131072$
tokens. Native RoPE is used without YaRN or position interpolation. Mistral results
beyond its documented 32K context window are labelled exploratory.
\begin{verbatim}
SEEDS="0 1 2" \
LENGTHS=2048,4096,8192,16384,32768,65536 \
DATASETS=proof_pile,pg19 \
CONFIGURATIONS="hadamard pairwise_only pairwise_plus_hadamard" \
MODEL=/path/to/meta-Llama-3.2-3B \
bash scripts/reproduction/run_long_context.sh
\end{verbatim}
The portable driver prints commands unless \texttt{EXECUTE=1}; no workload-manager
commands are embedded. It invokes the following evaluation command for each
specified configuration:
\begin{verbatim}
python dartquant_v2/run_quantize.py \
    --model /path/to/meta-Llama-3.2-3B \
    --loss whip --quantizer_type int4 --cal_dataset wikitext2 \
    --w_bits 4 --a_bits 4 --k_bits 4 --v_bits 4 --seed $SEED \
    --ppl_eval --ppl_eval_dataset wikitext2 \
    --eval_long_context \
    --long_context_datasets proof_pile \
    --long_context_lengths 2048,4096,8192,16384,32768,65536 \
    --long_context_stride 256 --long_context_max_docs 64 \
    --long_context_output <out_dir>/long_context_eval.json \
    <configuration flags>
\end{verbatim}
The configuration-specific flags are empty for full-head Hadamard;
\texttt{--equalising\_r3 --equalising\_r3\_no\_hadamard --equalising\_r3\_objective variance}
for pairwise-only; and
\texttt{--equalising\_r3 --equalising\_r3\_objective variance} for
pairwise\,+\,Hadamard.

For memory efficiency, the language-model head is applied to hidden states in blocks of
$1024$ tokens. Each requested context length is evaluated independently; if a context
length exceeds available memory, the output records an explicit error for that length
and evaluation proceeds for the remaining lengths. The reported Llama-3.2-3B and
Llama-3.2-1B 128K evaluations fit on a $40$\,GB A100. Llama-3.1-8B and Mistral-7B-v0.3
are reported through 64K.

The concatenated-stream evaluator scores the first $\min(24,\lfloor N/L\rfloor)$
non-overlapping windows from the tokenised 64-document stream at each context length
$L$. Both compared configurations therefore score byte-identical tokens at a given
length. The scored token set changes with context length: Proof-Pile contributes
$24/24/24/14/7/3/1$ windows at 2K--128K, whereas PG19 contributes $24$ windows at every
reported length through 128K (a prefix growing from $49$K to $3.15$M tokens).
Consequently, absolute perplexities should not be compared across context lengths; the
analysis uses within-length paired differences.

\subsection*{A.5 Pooled-covariance surrogate optimum: evaluation and verification}
The position-averaged configurations add \texttt{--equalising\_r3\_position\_averaged}
and specify the averaging length with \texttt{--equalising\_r3\_position\_length L}. In
long-context evaluations, $L$ is set to the context length being scored ($2048$,
$32768$, or $131072$ in the reported comparisons), so no evaluation uses an angle
averaged for a different deployment length.

The branch-matched controls \texttt{phi\_hat\_canon\_only} and
\texttt{phi\_hat\_canon\_layered} apply the same branch reduction to $\hat\phi_k$.
Comparisons between $\phistar$ and these controls therefore isolate the
finite-position-averaging correction rather than the branch convention.

RoPE frequencies are obtained from the instantiated checkpoint rather than
reconstructed from a base parameter. At runtime, the implementation searches the
rotary-embedding modules for a materialised \texttt{inv\_freq} buffer of length
$\dh/2$. Execution terminates with an explicit error if no compatible buffer is found
or if more than one distinct compatible buffer is present. This procedure includes
checkpoint-specific frequency scaling, including the Llama-3 \texttt{rope\_type} rule.

Each position-averaged evaluation records the applied angles, per-pair covariances reconstructed after raw-moment aggregation, and deployed frequency vector. A separate verification procedure iterates
over the registered configuration list rather than a filesystem glob, recomputes the
pooled-covariance, position-averaged surrogate objective from these stored quantities, and compares the achieved
per-pair minimax variance with the analytic minimum $\lambdabar$ of
Theorem~\ref{thm:closed-form-optimal-angle}. It writes a per-pair CSV and a summary
carrying \texttt{phi\_star\_verified}, the worst-case excess, the tolerance, and the
frequency source. Missing required verification artifacts cause the procedure to
terminate with an error; configurations that do not use $\phistar$ are excluded by
construction.

All $30$ position-averaged verification records satisfy the surrogate-optimality criterion. The largest
excess over the analytic minimum is $6.90\times10^{-7}$, compared with a tolerance of
$5\times10^{-5}$, and every record identifies the live rotary module as the frequency
source. Completeness of the evaluation set is determined by the same predicate the
evaluation harness uses to decide whether a configuration must still be evaluated, so
the criterion is identical for the harness and for the verification procedure. The
registered primary evaluation set contains $81$ manifest-defined configurations, all of
which produced valid completion artifacts. A further $14$ supplementary
configurations (nine long-context controls and five component evaluations) also
completed successfully.

Results obtained in an earlier evaluation that reconstructed frequencies from the base
parameter are excluded. That construction omitted checkpoint-specific RoPE scaling and
therefore implemented a different transformation. No reported value is derived from
that evaluation.

\subsection*{A.6 Static-quantisation control}
The static per-channel control is enabled with \texttt{DARTQUANT\_STATIC\_ACT=1},
\texttt{DARTQUANT\_STATIC\_CALIB=64}, and
\texttt{DARTQUANT\_STATIC\_MATCH\_EVAL\_LEN=1}. Under this condition,
\texttt{ActQuantizer.find\_params} uses one scale per channel, estimated across the
token dimension. The calibration-derived scales are then held fixed during evaluation.

For each corpus--context-length combination, the static state is reset and the scales
are re-estimated from up to $64$ non-overlapping windows of the target length drawn
from the evaluation corpus. Calibration and scoring text therefore overlap. This
protocol also changes the quantiser geometry from the default grouped per-token rule to
a static per-channel rule. The comparison is consequently interpreted as a
quantiser-regime contrast rather than as an identification of a causal mechanism.
\begin{verbatim}
SEEDS="0 1 2 3 4 5" \
DARTQUANT_STATIC_MATCH_EVAL_LEN=1 \
DARTQUANT_STATIC_ACT=1 DARTQUANT_STATIC_CALIB=64 \
DARTQUANT_STATIC_SCOPE=both \
LENGTHS=2048,4096,8192,16384,32768,65536 \
DATASETS=proof_pile \
CONFIGURATIONS="hadamard pairwise_only" \
MODEL=/path/to/meta-Llama-3.2-3B \
bash scripts/reproduction/run_long_context.sh
\end{verbatim}
This expands to the same \texttt{run\_quantize.py} invocation as \S A.4, with the
environment variables exported and the two configuration flag sets given there. The
environment variable \texttt{DARTQUANT\_STATIC\_SCOPE} selects which quantiser uses
calibration-derived fixed scales: \texttt{both} holds the activation and KV-cache
scales fixed; \texttt{act} holds the activation scales fixed while the KV-cache scales
remain per-token dynamic; and \texttt{kv} holds the KV-cache scales fixed while the
activation scales remain dynamic. The \texttt{both} setting reproduces the original
joint static modification. The runtime modification is inactive unless the
corresponding environment variable is set and does not alter the upstream DartQuant
source files.

\subsection*{A.7 Mixing-support interpolation}
The mixing-support evaluation uses pair-interleaved block-Hadamard transformations with
block sizes $b\in\{2,4,8,16,32,64,128\}$. The primary comparison includes three paired
perplexity seeds and three K-diagnostic seeds for each block size. Three reference
configurations are evaluated at the same source state for each seed, and single-seed
channel-layout and orientation controls are reported descriptively.

Only $b=2$ belongs to the RoPE-commuting family. Larger block sizes deliberately relax
commutativity to vary the number of channels over which a peak can be redistributed.
Each evaluation records WikiText-2 perplexity, mean and maximum K range after $R_3$, the
quantiser step diagnostic, and relative K quantisation error.

The complete command set is generated without an execution-system dependency by
\path{scripts/reproduction/generate_mixing_support_manifest.py}:
\begin{verbatim}
python scripts/reproduction/generate_mixing_support_manifest.py \
  --output mixing_support_manifest.csv
\end{verbatim}
The manifest contains exactly $33$ unique cells: seven block sizes at three seeds,
three reference transformations at three seeds, and three single-seed layout/orientation
controls. Each row records the model, seed, layout, orientation, and complete method
arguments and may be executed using the same \texttt{run\_quantize.py} prefix as
\S A.4.

The submitted compact aggregate and its paper values are verified entirely from files in
the companion code artifact with
\begin{verbatim}
python scripts/verify_artifact.py
\end{verbatim}
The verifier requires all three seeds for the primary interpolation and the corresponding
references, checks the published intervals and both support endpoints, and verifies the
reported monotone point-estimate trend. The single-seed layout and orientation controls
are excluded from the primary intervals.
Table~\ref{tab:data-mixing-support} is generated directly from
\path{docs/data_db/clean/mixing_support.csv}.

\subsection*{A.8 Implementation constraints and reproducibility notes}
\begin{itemize}
  \item \textbf{Model-path requirement.} The DartQuant evaluation path selects
        model-specific handling using the substring \texttt{meta} in \texttt{--model}. A
        local checkpoint path must therefore contain \texttt{meta} (for example,
        \texttt{.../meta-Llama-3.2-3B}) for the intended evaluation branch to be
        used.
  \item \textbf{Behaviour of the released $R_2$ calibration defaults.} The released
        DartQuant $R_2$ calibrator \citep{dartquant2025}
        (\texttt{calibrater/r2\_base\_qr.py}, as of the public release used here;
        later revisions may differ) defaults to
        \texttt{--nsamples 128 --bsz 128 --accumulation\_steps 2}. The complete
        calibration set therefore forms one batch per epoch, and the condition
        \texttt{(batch\_idx+1) \% accumulation\_steps} is never satisfied. No optimiser
        update is applied, and $R_2$ remains at its random-Hadamard initialisation. The
        present implementation reproduces this behaviour intentionally (the $R_2$ batch
        size is set to the dataset size, with a comment recording why) so that the
        experimental substrate matches the released method. The $R_2$ loss trace is
        therefore expected to remain constant, and the saved $R_2$ equals its initial
        value. Because $R_2$ is identical across transformations within a seed, it
        cancels in all paired comparisons.
  \item \textbf{Empty per-document sliding-window evaluations.} Under the per-document
        sliding-window protocol (\texttt{concat=False}), a requested window longer than
        every document yields no scored window, and the average NLL is vacuous. Earlier
        code reported the value $\mathrm{ppl}=1.0$ in this case, which produced a
        spurious GovReport \citep{huang2021govreport} $\geq\!32$K crossover. The current
        output records \texttt{n\_docs\_scored} and \texttt{n\_short\_skipped}, and the
        reported long-context results use the concatenated-stream protocol, which always
        supplies scored tokens at the requested lengths. The earlier GovReport values at
        32K and above obtained from the empty-window case are excluded.
  \item \textbf{Concatenated-stream window cap.} At each context length $L$, the
        evaluator scores the first $\min(24,\lfloor N/L\rfloor)$ non-overlapping windows
        from the 64-document token stream. This cap limits evaluation time and is fixed
        by the implementation. It does not affect within-length pairing, because
        compared transformations score identical tokens, but it changes the scored token
        set across lengths (\S A.4) and therefore precludes interpreting absolute
        perplexity as a length trend, as disclosed in
        \S\ref{sec:results_long_context}.
  \item \textbf{128K memory requirement.} The $131072$-token evaluation is the most
        memory-intensive condition. Llama-3.2-3B and Llama-3.2-1B fit on a $40$\,GB
        A100 when the language-model head is applied in $1024$-token blocks; no
        $80$\,GB card is required. The principal memory requirement arises from the
        forward pass and KV cache rather than from the language-model head. If a
        requested length exceeds available memory, the result file records an error for
        that length while preserving results for the remaining lengths. The
        Llama-3.1-8B and Mistral-7B-v0.3 evaluations are reported through 64K.
  \item \textbf{Non-additivity of the fixed-scale conditions.} The joint
        activation-and-KV fixed-scale result is not the sum of the activation-only and
        KV-only results. At 2K, the activation-only difference is $+1.73$ PPL and the
        KV-only difference is $-1.02$ PPL (sum $+0.71$), whereas the joint difference is
        $-4.69$ PPL in the preliminary unmatched three-seed analysis
        (App.~\ref{app:data},~\S D). This non-additivity is descriptive. Because the
        protocol is transductive and changes quantiser geometry, it does not isolate an
        interaction mechanism or attribute the result to either stage. The preliminary
        2K joint result also has substantial seed variability (sample standard deviation
        $3.67$ at $n{=}3$); the matched six-seed estimate in
        Table~\ref{tab:mechanism-static} ($-5.60$ PPL at 2K) is the reported reference.
\end{itemize}

\subsection*{A.9 Artifact and source-state provenance}

The companion code artifact is a compact release rather than a copy of the
execution environment. It contains the implementation, portable command
drivers and manifest generators, final compact summaries, selected seed-level tables,
a selective paper-number verifier, and content digests. It does not contain raw-output
aggregation utilities and does not claim to rebuild every table from raw logs. It excludes machine names, scheduler identifiers, private
development revisions, model weights, evaluation corpora, and superseded execution
notes. Consequently, the supplied tables can be inspected and the documented
configurations can be re-executed, but the ZIP alone does not reconstruct the
original scheduler history.

\paragraph{A.9.1 Position-averaged and component evaluations.}
The evaluated set comprises $81$ primary configurations, $23$ repeated
evaluations, nine additional long-context controls, and five component or precheck
configurations. The companion package represents these experiments through the
configuration generators and portable drivers. Their relation to the
reported tables is documented in \path{REPRODUCIBILITY.md}.
\begin{itemize}[leftmargin=1.4em,itemsep=1pt,topsep=2pt]
  \item \textbf{Configuration generation.}
        \path{scripts/reproduction/generate_primary_configuration_manifest.py}
        emits the 81 primary rows without invoking a scheduler. It records the
        model, configuration, seed, corpus, context-length, and method arguments
        used by the evaluation harness.
  \item \textbf{Scheduler-neutral commands.}
        \path{scripts/reproduction/run_extended_short_context.sh} records the
        additional short-context and position-averaged configurations, while
        \path{scripts/reproduction/run_angle_objectives.sh} records the
        angle-objective comparison, and
        \path{scripts/reproduction/run_long_context.sh} records the context-length
        and fixed-scale commands. These scripts print commands by default and require an
        explicit execution setting before launching experiments.
  \item \textbf{Verification.} Position-averaged runs request the per-band
        objective dump, which records the achieved objective, the analytic minimum,
        the numerical tolerance, and the checkpoint frequency source.
  \item \textbf{Compact data.} The compact CSVs under
        \path{docs/data_db/paper/} preserve the fields documented for each mapped
        table in \path{REPRODUCIBILITY.md}. Coverage is table-specific: some files
        contain displayed per-seed values, while others contain only final summaries.
\end{itemize}

\paragraph{A.9.2 Angle-estimator evaluation.}
The angle-estimator analysis (App.~\ref{app:data},~\S L) contains $77$
Llama-3.2-3B configurations and $103$ cross-model configurations; all listed
processes terminated successfully. The companion package contains compact final
views for this analysis; it does not include the raw-output aggregator. Three details
govern their interpretation.
\begin{itemize}[leftmargin=1.4em,itemsep=1pt,topsep=2pt]
  \item \textbf{Calibration-row budget.} The activation-bank row budget is derived from
        an explicit byte budget rather than from available host memory, because the
        realised row count affects the subsampling used to derive $R_1$ and $R_2$. Each
        evaluated configuration validates and records the model-specific row counts,
        ensuring that a change in the resulting rotations cannot occur without being
        recorded.
  \item \textbf{Combination of forward and reverse work lists.} The cross-model manifest
        is the union of forward and reverse evaluations over the same fixed list,
        producing $103$ unique configurations with no conflicting duplicates. When a
        configuration appears in both sets, the result fields are identical and the
        longer execution record is retained. One result reuses a previously validated
        evaluation produced under the same source state, byte budget, and realised row
        counts; the aggregation script records this reuse explicitly.
  \item \textbf{Source-state separation.} The retained execution records distinguish
        the revision captured when a work list was emitted from the source state that
        executed it. Those development identifiers are omitted from the compact
        package because they are not required to use the released source. The release instead
        fixes the supplied implementation by file-level SHA-256 digests.
\end{itemize}

\paragraph{A.9.3 Package integrity.}
The archive-generation procedure excludes Git metadata, local outputs,
scheduler-specific development launchers, and private denylist material. It
writes \path{MANIFEST.sha256} over the included files. The included archive builder
can additionally emit an optional sidecar SHA-256 digest next to the ZIP. The preparation
scanner and its release-specific denylist are not included in the companion package. These
digests verify the supplied contents; they do not claim identity with private
scheduler records that are intentionally absent.

\paragraph{A.9.4 Scope limitation.}
The companion package does not contain the raw scheduler logs or superseded status
files. In the retained execution archive, twelve unsuccessful status records coexist
with later schema-valid outputs; all affected configurations were subsequently
evaluated successfully, and all $81$ registered primary configurations are complete
under the valid-output criterion. The private records are retained outside the
distributed ZIP and can be consulted separately if provenance verification requires
them.

\section{Complete experimental data}
\label{app:data}

This appendix reports the complete numerical results supporting the main text.
Sections~\S A--\S H present the detailed Llama-3.2-3B evaluations; \S I presents the
cross-model results for Llama-3.2-1B, Llama-3.1-8B, and Mistral-7B-v0.3; and \S J--\S M
report the position-averaged, component, estimator, and mixing-support analyses.
App.~\ref{app:reproducibility} specifies the corresponding implementation and evaluation
protocols. Throughout this appendix, paired differences are computed within seed against
the full-head Hadamard configuration unless another comparator is stated explicitly.

\subsection*{Common environment and experimental configuration}
\begin{itemize}\setlength\itemsep{1pt}
  \item \textbf{Hardware and software.} The primary Llama-3.2-3B evaluations use
        NVIDIA A100 nodes with $40$ or $80$\,GB of memory. The cross-model evaluations
        and the separate six-seed Llama-3.2-3B evaluation use H100/H200-class nodes
        with A100 fallback. Both environments use \texttt{torch} 2.6.0+cu124 and
        \texttt{transformers} 4.49.0. The A100 environment verifies
        \texttt{fast-hadamard-transform} 1.1.0; the second environment uses a source
        build whose version tag is not confirmed. Each result directory records the
        source revision, software versions, and GPU model in \texttt{env.json}.
  \item \textbf{Quantisation.} Unless otherwise stated, all evaluations use W4A4KV4
        INT4 (\texttt{--quantizer\_type int4 --w\_bits 4 --a\_bits 4 --k\_bits 4
        --v\_bits 4}), GPTQ weight quantisation, and DartQuant asymmetric per-token
        dynamic quantisation for activations and the KV cache. \S D reports the static
        per-channel control.
  \item \textbf{Calibration.} The standard calibration set contains 128 WikiText-2
        sequences of length 2048. The closed-form $R_3$ angle is estimated once from
        this calibration pass (\S\ref{sec:experimental_setup}).
  \item \textbf{Statistical reporting.} Seed-level variability is reported with the sample ($n{-}1$) standard deviation. Population standard deviations emitted by the evaluation procedure are converted by multiplying by $\sqrt{n/(n-1)}$. Unless stated otherwise, confidence intervals are $90\%$ paired $t$-intervals computed within seed. The twelve position-averaged long-context contrasts in \S J constitute one comparison family; their two-sided $p$-values use Holm adjustment. The estimator comparisons in \S L are exploratory and report unadjusted $90\%$ paired intervals. Equivalence conclusions are stated relative to the selected margin and do not imply a task-independent minimum important difference.
  \item \textbf{Transformations.} The reader-facing names are full-head Hadamard,
        pairwise-only, and pairwise\,+\,Hadamard. Their literal code identifiers are
        \texttt{hadamard} (no $R_3$ flags), \texttt{eq\_only}
        (\texttt{--equalising\_r3 --equalising\_r3\_no\_hadamard
        --equalising\_r3\_objective variance}), and \texttt{eq\_layered} $\equiv$
        \texttt{eq\_r3} (\texttt{--equalising\_r3 --equalising\_r3\_objective variance};
        the closed form composed \emph{before} Hadamard), respectively. The variance
        objective is selected explicitly in all corresponding evaluations, because the
        flag otherwise defaults to \texttt{linf}.
  \item \textbf{Long-context evaluation.} Long-context perplexity uses the
        concatenated-stream protocol described in App.~\ref{app:reproducibility}, with
        $\min(24,\lfloor N/L\rfloor)$ non-overlapping windows at each context length
        (the window cap is fixed by the implementation, not a tunable flag). Proof-Pile
        and PG19 are evaluated with native RoPE and no YaRN. Mistral-7B-v0.3 results
        beyond its documented 32K context window are labelled exploratory.
  \item \textbf{Models.} The evaluated checkpoints are Llama-3.2-1B ($\dh{=}64$),
        Llama-3.2-3B ($\dh{=}128$, the primary checkpoint), Llama-3.1-8B
        ($\dh{=}128$), and Mistral-7B-v0.3 ($\dh{=}128$, a second pretraining family).
        The local Mistral path contains the substring \texttt{meta}, as required by the
        DartQuant evaluation branch.
  \item \textbf{Evaluation coverage.} The default-estimator comparisons are complete
        for all reported checkpoints and contexts. The position-averaged surrogate optimum
        $\phi_k^\ast$ (\S J) is evaluated only on \textbf{Llama-3.2-3B} and uses the
        checkpoint's deployed scaled RoPE frequencies. The pooling analysis (\S L)
        evaluates the default $Q/K$-weighted and K-only estimators on all four
        checkpoints, the stream-balanced estimator on Llama-3.2-3B, and $\phi_k^\ast$
        pooling variants on Llama-3.2-3B; non-default pooling is evaluated only at
        short context, so no long-context claim is made about it. The $L_\infty$ and
        high-quantile objectives are likewise restricted to 2K short-context
        evaluation. The matched static-scale comparison is reported through 64K
        (\S H). LongBench-v2 results are reported only for the aggregated
        Llama-3.2-3B evaluation of \S G: the corresponding cross-model evaluations
        were run at three seeds and produced the same non-discriminative value, but were
        never aggregated, so no cross-model LongBench-v2 value is reported and no
        conclusion is drawn from it.
  \item \textbf{Position-averaged configurations.} The code identifiers
        \texttt{eq\_only\_star} and \texttt{eq\_layered\_star} denote pairwise-only and
        pairwise\,+\,Hadamard with $\phi_k^\ast$, frequencies read from the
        checkpoint's live rotary module. \texttt{phi\_hat\_canon\_only} and
        \texttt{phi\_hat\_canon\_layered} denote the branch-matched $\hat\phi_k$
        controls, which apply the same branch reduction so that the difference between
        $\phi_k^\ast$ and the control isolates the position-averaging correction.
\end{itemize}

\subsection*{A. Intra-band calibration statistics}
A descriptive calibration analysis retains the attention-head dimension and estimates
one covariance matrix for every $(\text{layer},\text{head},\text{frequency-pair})$
combination. For each stream, Table~\ref{tab:data-rho} reports the mean absolute
correlation $\overline{|\rho|}$, the proportion of combinations with $|\rho|>0.05$, and
the mean absolute log variance ratio
$\overline{|\log(\sigma_1^2/\sigma_2^2)|}$. This analysis is distinct from the deployed
angle estimator (App.~\ref{app:reproducibility},~\S A.2.1), which aggregates raw moments across attention-head observation rows before constructing one covariance matrix and one angle for each
$(\text{layer},\text{frequency-pair})$ combination. Results are reported for all four
checkpoints.

\begin{table}[ht]\centering\small
\begin{tabular}{llccc}
\toprule
model & stream & $\overline{|\rho|}$ & frac $|\rho|{>}0.05$ & $\overline{|\log\,\mathrm{var\text{-}ratio}|}$ \\
\midrule
Llama-3.2-1B & Q & $0.165$ & $0.780$ & $0.466$ \\
             & K & $0.169$ & $0.769$ & $0.577$ \\
Llama-3.2-3B & Q & $0.141$ & $0.734$ & $0.402$ \\
             & K & $0.159$ & $0.748$ & $0.534$ \\
Llama-3.1-8B & Q & $0.135$ & $0.723$ & $0.367$ \\
             & K & $0.154$ & $0.738$ & $0.504$ \\
Mistral-7B-v0.3 & Q & $0.148$ & $0.735$ & $0.356$ \\
                & K & $\mathbf{0.175}$ & $0.755$ & $0.479$ \\
\bottomrule
\end{tabular}
\caption{Intra-band calibration statistics by checkpoint and stream. The table reports
the mean absolute correlation, the proportion of
$(\text{layer},\text{head},\text{frequency-pair})$ combinations with $|\rho|>0.05$, and
the mean absolute log variance ratio. These statistics describe the head-resolved
calibration distribution and are distinct from the concatenated-row angle estimator used in the
evaluated transformations.}
\label{tab:data-rho}
\end{table}

\subsection*{B. Short-context paired comparisons}
Table~\ref{tab:data-null} reports the three-seed Llama-3.2-3B comparison between
pairwise\,+\,Hadamard and full-head Hadamard on WikiText-2 under dynamic W4A4KV4. The
paired difference is
$\Delta\mathrm{PPL}=\mathrm{PPL}(\text{pairwise}+\text{Hadamard})-\mathrm{PPL}(\text{full-head Hadamard})$.
The corresponding Llama-3.2-1B, Llama-3.1-8B, and Mistral-7B-v0.3 results appear in \S I.

\begin{table}[ht]\centering\small
\begin{tabular}{cccc}
\toprule
\multicolumn{2}{c}{mean PPL} & \multicolumn{2}{c}{paired PPL difference} \\
\cmidrule(lr){1-2}\cmidrule(lr){3-4}
full-head Hadamard & pairwise\,+\,Hadamard & by seed 0/1/2 & mean $\pm$ sample sd \\
\midrule
$10.0823$ & $10.0864$ & $-0.0078 / {+}0.0034 / {+}0.0168$ & $+0.0041 \pm 0.0123$ \\
\bottomrule
\end{tabular}
\caption{Llama-3.2-3B short-context paired comparison on WikiText-2 under dynamic
W4A4KV4 with \textsc{whip} calibration, $n=3$; $\pm$ is sample sd. The difference is
pairwise\,+\,Hadamard minus full-head Hadamard; positive values indicate higher
perplexity for pairwise\,+\,Hadamard. The $90\%$ confidence interval lies within the
selected $\pm0.05$-PPL interval criterion.}
\label{tab:data-null}
\end{table}

\paragraph{Additional seed counts.}
An independent six-seed evaluation of $\hat\phi_k$, conducted in the cross-model
software environment and paired within that environment, yields a corrected $90\%$
confidence interval of $[-0.005,+0.023]$ PPL under \textsc{whip} calibration. These
values are not pooled with the three-seed results because the software environments
differ. An additional nine-seed Llama-3.2-1B evaluation yields a mean paired difference
of $+0.0002$ PPL, sample standard deviation $0.0283$, and $90\%$ confidence interval
$[-0.0173,+0.0178]$.

\begin{table}[ht]\centering\small
\begin{tabular}{lrrrr}
\toprule
comparison & $n$ & mean difference & sample sd & $90\%$ CI \\
\midrule
Llama-3.2-1B, pairwise\,+\,Hadamard with $\hat\phi_k$ & 9 & $+0.0002$ & $0.0283$ & $[-0.0173,+0.0178]$ \\
\bottomrule
\end{tabular}
\caption{Additional nine-seed short-context interval-criterion evaluation for Llama-3.2-1B on
WikiText-2 under dynamic W4A4KV4 with \textsc{whip} calibration. The paired difference
is pairwise\,+\,Hadamard minus full-head Hadamard. The $90\%$ confidence interval
lies within the selected $\pm0.05$-PPL margin under the TOST interval formulation.}
\label{tab:data-fresh-short}
\end{table}

\paragraph{Angle-objective comparison.}
Table~\ref{tab:data-objectives} compares the variance, $L_\infty$, and high-quantile
angle objectives. Comparisons labelled \emph{matched} use configurations evaluated from
the same source revision and execution state. Comparisons labelled \emph{different} use
a previously recorded full-head Hadamard or pairwise\,+\,Hadamard reference; a
validation evaluation measured a $0.110$-PPL difference attributable to source-state
drift, so different-source-state comparisons are interpreted directionally and are not
used to establish equivalence.

\begin{table}[ht]\centering\small
\resizebox{\linewidth}{!}{%
\begin{tabular}{lllcrrr}
\toprule
configuration & objective & comparator & source state & mean difference & sample sd & $90\%$ CI \\
\midrule
pairwise-only & variance & full-head Hadamard & different & $+0.4945$ & $0.0476$ & $[+0.4143,+0.5747]$ \\
pairwise-only & $L_\infty$ & full-head Hadamard & different & $+0.4108$ & $0.0257$ & $[+0.3676,+0.4541]$ \\
pairwise-only & $L_\infty$ & pairwise-only, variance & matched & $-0.0837$ & $0.0717$ & $[-0.2045,+0.0372]$ \\
pairwise-only & high-quantile & full-head Hadamard & different & $+0.4264$ & $0.0344$ & $[+0.3684,+0.4845]$ \\
pairwise-only & high-quantile & pairwise-only, variance & matched & $-0.0681$ & $0.0600$ & $[-0.1692,+0.0331]$ \\
pairwise\,+\,Hadamard & $L_\infty$ & full-head Hadamard & different & $+0.0110$ & $0.0225$ & $[-0.0269,+0.0489]$ \\
pairwise\,+\,Hadamard & $L_\infty$ & pairwise\,+\,Hadamard, variance & different & $+0.0005$ & $0.0058$ & $[-0.0093,+0.0102]$ \\
pairwise\,+\,Hadamard & high-quantile & full-head Hadamard & different & $-0.0066$ & $0.0232$ & $[-0.0457,+0.0324]$ \\
pairwise\,+\,Hadamard & high-quantile & pairwise\,+\,Hadamard, variance & different & $-0.0172$ & $0.0032$ & $[-0.0226,-0.0118]$ \\
\bottomrule
\end{tabular}%
}
\caption{Llama-3.2-3B short-context comparison of angle objectives under dynamic
W4A4KV4 with \textsc{whip} calibration, $n=3$. Matched-source-state rows compare
pairwise-only objectives within the same execution state. Rows using a
different-source-state reference are reported as directional checks and are not used
for equivalence inference.}
\label{tab:data-objectives}
\end{table}

\paragraph{Equivalence-margin sensitivity.}
Table~\ref{tab:tost-sensitivity} evaluates the criterion-relative conclusion at margins of $\pm0.02$, $\pm0.05$, and $\pm0.10$ PPL. Equivalence at margin $m$ is established when the complete sample-sd $90\%$ paired confidence interval lies within $[-m,+m]$. The manuscript uses $m=0.05$ PPL as its primary reporting criterion. This value is not presented as an externally validated minimum important difference; conclusions are conditional on the selected margin, and the additional columns show their sensitivity to alternative thresholds.

\begin{table}[ht]\centering\small
\begin{tabular}{lccccc}
\toprule
model & mean $\Delta$ & $90\%$ CI & $\pm0.02$ & $\pm0.05$ & $\pm0.10$ \\
\midrule
Llama-3.2-1B ($n=9$) & $+0.0002$ & $[-0.017,+0.018]$ & equivalent & equivalent & equivalent \\
Llama-3.2-3B    & $+0.0041$ & $[-0.017,+0.025]$ & not established & equivalent & equivalent \\
Llama-3.1-8B    & $+0.0063$ & $[-0.002,+0.015]$ & equivalent & equivalent & equivalent \\
Mistral-7B-v0.3 & $-0.0012$ & $[-0.006,+0.004]$ & equivalent & equivalent & equivalent \\
\bottomrule
\end{tabular}
\caption{Sensitivity of the TOST conclusion for pairwise\,+\,Hadamard versus full-head
Hadamard under \textsc{whip} calibration. The Llama-3.2-1B row uses the additional
$n=9$ evaluation; the other rows use $n=3$. Equivalence at margin $m$ is established
when the full $90\%$ confidence interval lies within $[-m,+m]$. All four checkpoints
satisfy the selected $\pm0.05$-PPL criterion; the table shows that this conclusion is margin-dependent for Llama-3.2-3B.}
\label{tab:tost-sensitivity}
\end{table}

\subsection*{C. Long-context evaluation under dynamic quantisation}
Table~\ref{tab:data-crossover} reports concatenated-stream perplexity for Llama-3.2-3B
on Proof-Pile and PG19 under dynamic per-token W4A4KV4, with three paired seeds. The
compared transformations are full-head Hadamard, pairwise-only, and
pairwise\,+\,Hadamard. Results are complete from $2$K through $128$K on both corpora;
the $128$K evaluations fit on the $40$\,GB A100 environment, so no $80$\,GB hardware is
needed. Within every corpus and context length, the compared transformations score
identical tokens.

\begin{table}[ht]\centering\small
\begin{tabular}{rccccc}
\toprule
 & \multicolumn{3}{c}{mean PPL} & \multicolumn{2}{c}{PPL difference vs Hadamard} \\
\cmidrule(lr){2-4}\cmidrule(lr){5-6}
$L$ & full-head Hadamard & pairwise-only & pairwise\,+\,Hadamard & pairwise-only & pairwise\,+\,Hadamard \\
\midrule
\multicolumn{6}{l}{\emph{Proof-Pile}}\\
2K  & $4.536{\scriptstyle\pm.082}$ & $4.696{\scriptstyle\pm.095}$ & $4.536{\scriptstyle\pm.073}$ & $+0.161{\scriptstyle\pm.018}$ & $+0.000{\scriptstyle\pm.015}$ \\
4K  & $3.516{\scriptstyle\pm.052}$ & $3.659{\scriptstyle\pm.062}$ & $3.514{\scriptstyle\pm.044}$ & $+0.143{\scriptstyle\pm.014}$ & $-0.002{\scriptstyle\pm.008}$ \\
8K  & $3.606{\scriptstyle\pm.047}$ & $3.770{\scriptstyle\pm.065}$ & $3.605{\scriptstyle\pm.046}$ & $+0.164{\scriptstyle\pm.020}$ & $-0.001{\scriptstyle\pm.005}$ \\
16K & $3.411{\scriptstyle\pm.039}$ & $3.582{\scriptstyle\pm.047}$ & $3.414{\scriptstyle\pm.036}$ & $+0.171{\scriptstyle\pm.009}$ & $+0.003{\scriptstyle\pm.004}$ \\
32K & $3.290{\scriptstyle\pm.037}$ & $3.481{\scriptstyle\pm.046}$ & $3.294{\scriptstyle\pm.036}$ & $+0.191{\scriptstyle\pm.010}$ & $+0.004{\scriptstyle\pm.002}$ \\
64K & $3.233{\scriptstyle\pm.028}$ & $3.439{\scriptstyle\pm.036}$ & $3.234{\scriptstyle\pm.027}$ & $+0.207{\scriptstyle\pm.011}$ & $+0.001{\scriptstyle\pm.001}$ \\
128K & $2.933{\scriptstyle\pm.021}$ & $3.126{\scriptstyle\pm.032}$ & $2.937{\scriptstyle\pm.020}$ & $+0.193{\scriptstyle\pm.016}$ & $+0.004{\scriptstyle\pm.002}$ \\
\midrule
\multicolumn{6}{l}{\emph{PG19}}\\
2K  & $10.953{\scriptstyle\pm.289}$ & $11.518{\scriptstyle\pm.324}$ & $10.937{\scriptstyle\pm.282}$ & $+0.566{\scriptstyle\pm.060}$ & $-0.015{\scriptstyle\pm.041}$ \\
4K  & $15.986{\scriptstyle\pm.414}$ & $16.964{\scriptstyle\pm.491}$ & $15.974{\scriptstyle\pm.399}$ & $+0.979{\scriptstyle\pm.080}$ & $-0.012{\scriptstyle\pm.016}$ \\
8K  & $17.148{\scriptstyle\pm.442}$ & $18.213{\scriptstyle\pm.476}$ & $17.163{\scriptstyle\pm.424}$ & $+1.065{\scriptstyle\pm.035}$ & $+0.015{\scriptstyle\pm.019}$ \\
16K & $16.527{\scriptstyle\pm.407}$ & $17.726{\scriptstyle\pm.448}$ & $16.560{\scriptstyle\pm.426}$ & $+1.199{\scriptstyle\pm.042}$ & $+0.033{\scriptstyle\pm.026}$ \\
32K & $14.735{\scriptstyle\pm.352}$ & $15.888{\scriptstyle\pm.405}$ & $14.751{\scriptstyle\pm.380}$ & $+1.153{\scriptstyle\pm.054}$ & $+0.016{\scriptstyle\pm.029}$ \\
64K & $13.054{\scriptstyle\pm.272}$ & $14.169{\scriptstyle\pm.326}$ & $13.059{\scriptstyle\pm.290}$ & $+1.115{\scriptstyle\pm.056}$ & $+0.005{\scriptstyle\pm.018}$ \\
128K & $12.484{\scriptstyle\pm.324}$ & $13.653{\scriptstyle\pm.410}$ & $12.488{\scriptstyle\pm.341}$ & $+1.169{\scriptstyle\pm.085}$ & $+0.004{\scriptstyle\pm.016}$ \\
\bottomrule
\end{tabular}
\caption{Llama-3.2-3B concatenated-stream perplexity on Proof-Pile and PG19 under
dynamic W4A4KV4 with \textsc{whip} calibration, $n=3$; $\pm$ is sample sd. Differences
are computed relative to the same-seed full-head Hadamard configuration; positive
values indicate higher perplexity. Absolute perplexity should not be compared across
context lengths because the scored token set changes with length; exact per-length window
counts are reported in Appendix~\ref{app:reproducibility}, \S A.4. Pairwise-only yields higher
perplexity at every reported evaluation point through 128K.}
\label{tab:data-crossover}
\end{table}

\paragraph{Frequency source for the position-averaged configurations.}
Only position-averaged evaluations that obtain RoPE frequencies from the checkpoint's
live rotary module are included. Reconstructing frequencies from the base parameter
omits checkpoint-specific scaling and defines a different transformation; values
obtained from that construction are excluded from all reported results, and every
$\phi_k^\ast$ value reported here comes from the deployed-frequency evaluations of
\S J.

\subsection*{D. Static per-channel control and fixed-scale component comparison}
The preliminary three-seed analysis in Table~\ref{tab:data-static-both} uses a static per-channel quantiser in which activation and KV scales are estimated from $64$ calibration forwards and retained during evaluation (\texttt{DARTQUANT\_STATIC\_SCOPE} $\in$ \{\texttt{both}, \texttt{act}, \texttt{kv}\}), on the same Llama-3.2-3B / Proof-Pile concatenated-stream setting as \S C. Calibration length is not matched to evaluation length in this analysis. The matched six-seed comparison in Table~\ref{tab:mechanism-static} removes that length mismatch, but it continues to estimate scales from windows that are subsequently scored and changes the scale geometry from grouped per-token to per-channel quantisation. The static comparisons are consequently treated as descriptive contrasts between evaluation protocols. They do not identify the process responsible for the dynamic-quantisation result or quantify out-of-sample performance under static quantisation.

\begin{table}[ht]\centering\small
\begin{tabular}{rccccc}
\toprule
 & \multicolumn{3}{c}{mean PPL} & \multicolumn{2}{c}{PPL difference vs Hadamard} \\
\cmidrule(lr){2-4}\cmidrule(lr){5-6}
$L$ & full-head Hadamard & pairwise-only & pairwise\,+\,Hadamard & pairwise-only & pairwise\,+\,Hadamard \\
\midrule
\multicolumn{6}{l}{\emph{joint activation-and-KV fixed-scale configuration}}\\
2K  & $22.430{\scriptstyle\pm3.47}$ & $17.738{\scriptstyle\pm1.21}$ & $21.077{\scriptstyle\pm2.65}$ & $\mathbf{-4.692{\scriptstyle\pm3.67}}$ & $-1.353{\scriptstyle\pm.87}$ \\
4K  & $14.922{\scriptstyle\pm1.37}$ & $12.657{\scriptstyle\pm0.32}$ & $14.681{\scriptstyle\pm1.91}$ & $-2.265{\scriptstyle\pm1.41}$ & $-0.241{\scriptstyle\pm.58}$ \\
8K  & $14.808{\scriptstyle\pm1.30}$ & $13.129{\scriptstyle\pm0.15}$ & $15.105{\scriptstyle\pm1.71}$ & $-1.679{\scriptstyle\pm1.38}$ & $+0.297{\scriptstyle\pm.42}$ \\
16K & $12.791{\scriptstyle\pm1.15}$ & $12.305{\scriptstyle\pm0.45}$ & $13.547{\scriptstyle\pm1.42}$ & $-0.485{\scriptstyle\pm1.47}$ & $+0.757{\scriptstyle\pm.32}$ \\
32K & $11.714{\scriptstyle\pm1.19}$ & $12.960{\scriptstyle\pm0.73}$ & $12.324{\scriptstyle\pm1.15}$ & $+1.246{\scriptstyle\pm1.74}$ & $+0.610{\scriptstyle\pm.05}$ \\
64K & $11.413{\scriptstyle\pm1.18}$ & $15.179{\scriptstyle\pm2.46}$ & $11.681{\scriptstyle\pm0.93}$ & $\mathbf{+3.766{\scriptstyle\pm2.79}}$ & $+0.268{\scriptstyle\pm.42}$ \\
\bottomrule
\end{tabular}
\caption{Preliminary unmatched static per-channel comparison on Llama-3.2-3B with
\textsc{whip} calibration, $n=3$; $\pm$ is sample sd. Activation and KV scales are
estimated during calibration and held fixed during evaluation. Calibration length is
confounded with evaluation length; the apparent 64K reversal is therefore not used for
inference.}
\label{tab:data-static-both}
\end{table}

Table~\ref{tab:data-static-roles} compares activation-only and KV-only fixed-scale
conditions under the same preliminary three-seed protocol. Neither condition reproduces
the joint activation-and-KV result, and the three differences are not additive.

\begin{table}[ht]\centering\small
\begin{tabular}{rcccccc}
\toprule
 & \multicolumn{6}{c}{pairwise-only $-$ full-head Hadamard PPL difference} \\
\cmidrule(lr){2-7}
fixed-scale condition & 2K & 4K & 8K & 16K & 32K & 64K \\
\midrule
activation-only & $+1.733{\scriptstyle\pm.33}$ & $+1.292{\scriptstyle\pm.28}$ & $+1.028{\scriptstyle\pm.26}$ & $+0.935{\scriptstyle\pm.27}$ & $+0.964{\scriptstyle\pm.27}$ & $+1.038{\scriptstyle\pm.27}$ \\
KV-only  & $-1.019{\scriptstyle\pm.06}$ & $-0.729{\scriptstyle\pm.12}$ & $-0.795{\scriptstyle\pm.12}$ & $-0.710{\scriptstyle\pm.10}$ & $-0.321{\scriptstyle\pm.09}$ & $+0.278{\scriptstyle\pm.12}$ \\
\bottomrule
\end{tabular}
\caption{Preliminary comparison of activation-only and KV-only fixed-scale conditions
on Llama-3.2-3B with \textsc{whip} calibration, $n=3$; $\pm$ is sample sd. Values are
pairwise-only minus full-head Hadamard. Neither single-component condition reproduces
the joint activation-and-KV fixed-scale difference.}
\label{tab:data-static-roles}
\end{table}

\subsection*{E. Quantiser-step and post-quantisation-error diagnostic}
The quantiser diagnostic instruments \texttt{ActQuantizer.find\_params} without changing
its numerical output. For each transformation, it records the selected per-channel step
$\Delta_c$ and the relative post-quantisation error $\|y-\hat y\|^2/\|y\|^2$ for the
tensors whose scales are being determined. Each reported condition contains $39{,}480$
quantiser evaluations.

Under dynamic per-token quantisation, the values are reported for seed~0 and the
call-level observations are dependent; they therefore provide a descriptive readout
rather than independent statistical replication. Under the joint static per-channel
condition, values are aggregated across three seeds. The static condition increases
error for both transformations and substantially reduces the diagnostic advantage of the
full-head Hadamard.

Table~\ref{tab:data-deltac} reports the corresponding step-size and error diagnostics for the dynamic and joint static conditions.

\begin{table}[ht]\centering\small
\begin{tabular}{lccccc}
\toprule
configuration & $\overline{\Delta_c}$ & sd $\Delta_c$ & max $\Delta_c$ & post-quant rel-MSE & quantiser evaluations \\
\midrule
\multicolumn{6}{l}{\emph{dynamic per-token rule (seed 0)}}\\
full-head Hadamard & $0.2843$ & $0.2188$ & $1.800$ & $0.01348$ & $39{,}480$ \\
pairwise-only & $0.3093$ & $0.2774$ & $2.329$ & $0.01634$ & $39{,}480$ \\
\multicolumn{6}{l}{\emph{joint activation-and-KV fixed-scale rule (3 seeds)}}\\
full-head Hadamard & $0.3393$ & $0.281$ & $1.647$ & $0.02418$ & $39{,}480$ \\
pairwise-only & $0.3288$ & $0.265$ & $2.875$ & $0.02391$ & $39{,}480$ \\
\bottomrule
\end{tabular}
\caption{Quantiser-step and relative-error diagnostic for Llama-3.2-3B with
\textsc{whip} calibration. Dynamic values are from seed~0 and call-level observations
are dependent; fixed-scale values aggregate three seeds. The diagnostic is consistent
with partial absorption of the transformation-dependent range difference by dynamic
scale estimation, but it is not an independent statistical test.}
\label{tab:data-deltac}
\end{table}

\subsection*{F. Lower-bit evaluation}
Table~\ref{tab:data-bitwidth} reports WikiText-2 perplexity for Llama-3.2-3B at
W4A3KV4 and W3A3KV3 using three paired seeds. Absolute perplexity is highly variable at
these bit widths, although within-seed differences remove part of the shared variation.
The models are substantially degraded, so the results are descriptive and do not support
a general claim across bit widths (\S\ref{sec:limitations}).

\begin{table}[ht]\centering\small
\begin{tabular}{lccccc}
\toprule
configuration & seed 0 & seed 1 & seed 2 & mean & PPL difference vs Hadamard \\
\midrule
\multicolumn{6}{l}{\emph{W4A3KV4}}\\
full-head Hadamard & $85.373$ & $56.406$ & $55.569$ & $65.783$ & \textit{n/a} \\
pairwise\,+\,Hadamard & $83.947$ & $56.776$ & $55.133$ & $65.286$ & $-0.497$ \\
pairwise-only & $93.287$ & $64.522$ & $62.502$ & $73.437$ & $+7.654$ \\
\midrule
\multicolumn{6}{l}{\emph{W3A3KV3}}\\
full-head Hadamard & $155.760$ & $146.080$ & $165.384$ & $155.741$ & \textit{n/a} \\
pairwise\,+\,Hadamard & $159.480$ & $152.329$ & $164.379$ & $158.729$ & $+2.988$ \\
pairwise-only & $205.163$ & $191.455$ & $213.939$ & $203.519$ & $+47.778$ \\
\bottomrule
\end{tabular}
\caption{Llama-3.2-3B lower-bit evaluation on WikiText-2 with \textsc{whip}
calibration, $n=3$. The quantised models are substantially degraded; the paired
differences are reported descriptively and do not establish a general bit-width trend.}
\label{tab:data-bitwidth}
\end{table}

\subsection*{G. Long-context task evaluations}
\paragraph{NIAH.}
Table~\ref{tab:data-niah} reports depth-averaged NIAH retrieval accuracy for
Llama-3.2-3B under dynamic quantisation. Each value averages five needle depths and $20$
probes per depth. Quantised results are averaged over seeds $\{0,1,2\}$; the FP16
reference uses seed 0. A context length is described as \emph{discriminative} when the
quantised configurations are sufficiently separated from the near-saturation FP16 value to
permit an informative comparison; this designation is a descriptive diagnostic rather
than a preregistered statistical rule. The depth-averaged values used in this table are
available in \path{docs/data_db/paper/appG_niah_3b.csv}; individual probe records are
not included in the companion code artifact.

\begin{table}[ht]\centering\small
\begin{tabular}{rccccl}
\toprule
$L$ & FP16 & full-head Hadamard & pairwise-only & pairwise\,+\,Hadamard & diagnostic status \\
\midrule
4K  & $0.990$ & $0.973$ & $0.967$ & $0.980$ & near FP16 saturation \\
8K  & $0.980$ & $0.983$ & $0.970$ & $0.977$ & near FP16 saturation \\
16K & $0.990$ & $0.957$ & $0.953$ & $0.953$ & near FP16 saturation \\
32K & $1.000$ & $0.973$ & $0.923$ & $0.950$ & \textbf{discriminative} \\
64K & $0.990$ & $0.983$ & $0.923$ & $0.933$ & \textbf{discriminative} \\
\bottomrule
\end{tabular}
\caption{Llama-3.2-3B NIAH depth-averaged retrieval accuracy with \textsc{whip}
calibration. The 32K and 64K evaluations provide sufficient separation from the FP16
result to distinguish the quantised configurations; shorter contexts are near saturation
and are therefore non-discriminative.}
\label{tab:data-niah}
\end{table}

\paragraph{LongBench-v2.}
The same Llama-3.2-3B quantised model instances are evaluated on the $384$ LongBench-v2
examples within the $64$K evaluation window. Answers are scored by option log
probability. The random baseline for four-choice questions is $0.25$. All reported means
are at or below this baseline, so the benchmark does not discriminate among the
transformations at this quantisation point. Table~\ref{tab:data-longbench} reports the complete per-seed results.

\begin{table}[ht]\centering\small
\begin{tabular}{lccc}
\toprule
configuration & overall accuracy & full-head Hadamard $-$ configuration & per-seed accuracy \\
\midrule
full-head Hadamard    & $0.207{\scriptstyle\pm.011}$ & \textit{n/a} & $.195/.214/.214$ \\
pairwise-only    & $0.224{\scriptstyle\pm.021}$ & $-0.016{\scriptstyle\pm.030}$ & $.245/.203/.224$ \\
pairwise\,+\,Hadamard & $0.239{\scriptstyle\pm.049}$ & $-0.031{\scriptstyle\pm.043}$ & $.203/.219/.294$ \\
\bottomrule
\end{tabular}
\caption{Llama-3.2-3B LongBench-v2 accuracy with \textsc{whip} calibration, $n=3$;
$\pm$ is sample sd. Scores are at or below the $0.25$ four-choice random baseline; the
evaluation is therefore non-discriminative for the compared transformations.}
\label{tab:data-longbench}
\end{table}

\subsection*{H. Interpretation of the matched static-scale comparison}
The matched-length six-seed analysis removes the calibration-length mismatch in Table~\ref{tab:data-static-both}. At 64K, the mean pairwise-only-minus-Hadamard difference is $-0.18$ PPL with sample standard deviation $0.89$; this comparison does not satisfy the selected $\pm0.05$-PPL criterion. A matched 128K result is not reported because the concatenated stream contains only one 128K window, which would further increase the overlap between scale estimation and scoring. Such overlap is present at every reported length.

For each target length in Table~\ref{tab:mechanism-static}, scales are estimated from windows of that length and retained during scoring. The protocol also replaces the deployed grouped dynamic geometry with static per-channel quantisation, and the number of calibration tokens varies with context length. The results are therefore used only to assess sensitivity to quantiser protocol. They are not used as evidence of deployment performance or as identification of a causal mechanism.

\begin{table}[ht]
\centering\small
\caption{Matched static per-channel comparison on Llama-3.2-3B/Proof-Pile with
\textsc{whip} calibration, $n=6$; $\pm$ is sample sd. Positive differences indicate
higher perplexity for pairwise-only. Scales are estimated during target-length
calibration and held fixed during evaluation. Calibration overlaps the scored windows,
and the per-channel geometry differs from the deployed grouped dynamic quantiser.}
\label{tab:mechanism-static}
\begin{tabular}{rccrrr}
\toprule
& \multicolumn{2}{c}{joint activation-and-KV fixed scale} & \multicolumn{3}{c}{pairwise-only $-$ full-head Hadamard} \\
\cmidrule(lr){2-3}\cmidrule(lr){4-6}
length & full-head Hadamard PPL & pairwise-only PPL & joint & activation-only & KV-only \\
\midrule
2K  & 19.76 & 14.16 & $-5.60{\scriptstyle\pm2.69}$ & $+1.50{\scriptstyle\pm.23}$ & $-1.12{\scriptstyle\pm.22}$ \\
4K  & 17.48 & 11.68 & $-5.80{\scriptstyle\pm5.28}$ & $+1.26{\scriptstyle\pm.26}$ & $-0.91{\scriptstyle\pm.14}$ \\
8K  & 15.98 & 13.25 & $-2.74{\scriptstyle\pm1.26}$ & $+1.30{\scriptstyle\pm.53}$ & $-1.16{\scriptstyle\pm.14}$ \\
16K & 14.63 & 12.36 & $-2.27{\scriptstyle\pm2.11}$ & $+1.20{\scriptstyle\pm.34}$ & $-1.25{\scriptstyle\pm.17}$ \\
32K & 13.32 & 12.76 & $-0.55{\scriptstyle\pm1.59}$ & $+1.12{\scriptstyle\pm.24}$ & $-1.14{\scriptstyle\pm.18}$ \\
64K & 12.56 & 12.38 & $-0.18{\scriptstyle\pm.89}$  & $+1.33{\scriptstyle\pm.14}$ & $-0.97{\scriptstyle\pm.28}$ \\
\bottomrule
\end{tabular}
\end{table}

Figure~\ref{fig:gap-vs-length} compares the dynamic and matched static pairwise-only differences across context lengths. It visualises sensitivity to quantiser protocol and does not support a direct performance comparison between the two regimes.

\begin{figure}[ht]
\centering
\includegraphics[width=0.92\linewidth]{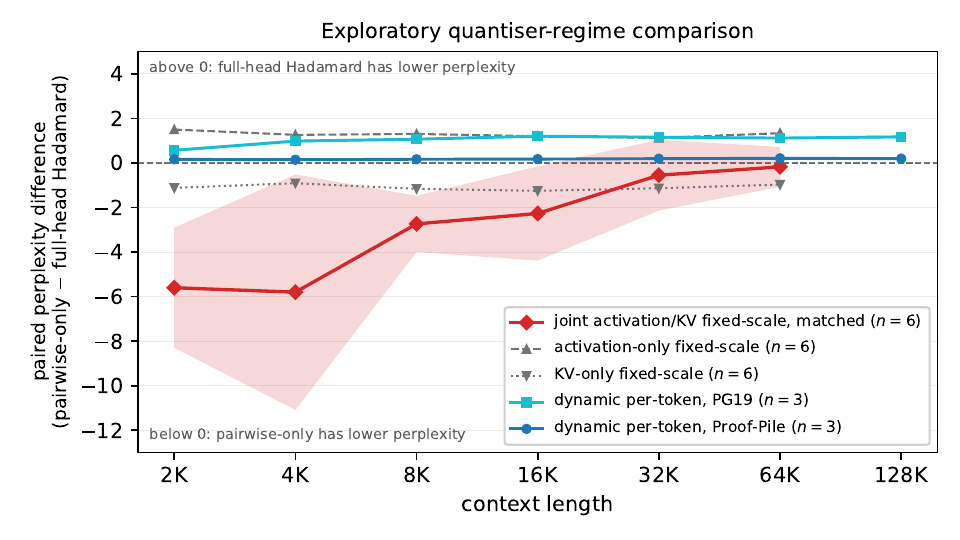}
\caption{Exploratory comparison across quantiser regimes. Dynamic curves use the
deployed per-token rule with three paired seeds; fixed-scale curves use matched,
transductive per-channel calibration with six paired seeds. Positive values indicate
higher perplexity for pairwise-only than for full-head Hadamard.}
\label{fig:gap-vs-length}
\end{figure}

Figure~\ref{fig:mechanism-schematic} summarises the associated scale-setting diagnostics. The measurements are descriptive because the dynamic values use one seed and the static protocol is transductive.

\begin{figure}[ht]
\centering
\includegraphics[width=\linewidth]{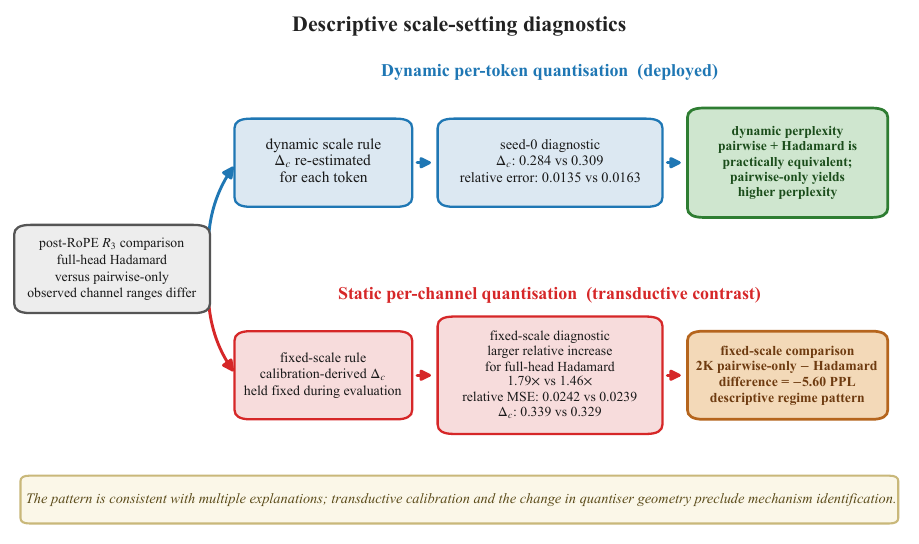}
\caption{Descriptive scale-setting diagnostic. The dynamic $\Delta_c$ values are from
seed~0, and individual quantiser evaluations are not independent replicates. The
fixed-scale comparison is transductive and changes quantiser geometry. The figure
therefore illustrates a quantiser-regime pattern rather than identifying a mechanism.}
\label{fig:mechanism-schematic}
\end{figure}

\subsection*{I. Cross-model evaluation}
\label{app:data-breadth}
The cross-model evaluation applies the same W4A4KV4 INT4 protocol to Llama-3.2-1B,
Llama-3.1-8B, and Mistral-7B-v0.3, using GPTQ weights, dynamic per-token activation/KV
quantisation, 128 WikiText-2 calibration sequences of length 2048, and seeds
$\{0,1,2\}$.

Table~\ref{tab:data-breadth-null} reports the short-context comparison between
pairwise\,+\,Hadamard and full-head Hadamard. The Llama-3.1-8B and Mistral-7B-v0.3
confidence intervals lie within the selected $\pm0.05$-PPL interval criterion. The original
three-seed Llama-3.2-1B result is inconclusive and is superseded for equivalence
inference by the independent nine-seed result in Table~\ref{tab:data-fresh-short}.

\begin{table}[ht]\centering\small
\resizebox{\linewidth}{!}{%
\begin{tabular}{lcccl}
\toprule
 & \multicolumn{2}{c}{mean PPL} & \multicolumn{2}{c}{paired PPL difference} \\
\cmidrule(lr){2-3}\cmidrule(lr){4-5}
model & full-head Hadamard & pairwise\,+\,Hadamard & mean $\pm$ sample sd & by seed 0/1/2 \\
\midrule
Llama-3.2-1B    & $14.4268$ & $14.4163$ & $-0.0105 \pm 0.0359$ & $-0.0490/{-}0.0048/{+}0.0222$ \\
Llama-3.1-8B    & $7.8251$  & $7.8314$  & $+0.0063 \pm 0.0049$ & $+0.0040/{+}0.0030/{+}0.0120$ \\
Mistral-7B-v0.3 & $5.8921$  & $5.8909$  & $-0.0012 \pm 0.0031$ & $-0.0034/{+}0.0023/{-}0.0025$ \\
\bottomrule
\end{tabular}%
}
\caption{Cross-model short-context paired differences on WikiText-2 under dynamic
W4A4KV4 with \textsc{whip} calibration, $n=3$; $\pm$ is sample sd. The difference is
pairwise\,+\,Hadamard minus full-head Hadamard. The Llama-3.1-8B and Mistral-7B-v0.3
intervals satisfy the selected $\pm0.05$-PPL interval criterion; the three-seed Llama-3.2-1B
result is superseded for inference by Table~\ref{tab:data-fresh-short}.}
\label{tab:data-breadth-null}
\end{table}

Table~\ref{tab:data-breadth-crossover} reports the long-context pairwise-only difference
relative to full-head Hadamard. Every within-length mean is positive. The Mistral
differences are small and are included as cross-model context rather than as
individually decisive results. Absolute perplexity and slopes are not compared across
lengths.

\begin{table}[ht]\centering\small
\begin{tabular}{lrrrrrrr}
\toprule
 & \multicolumn{7}{c}{pairwise-only $-$ full-head Hadamard PPL difference} \\
\cmidrule(lr){2-8}
model / corpus & 2K & 4K & 8K & 16K & 32K & 64K & 128K \\
\midrule
1B / Proof-Pile      & $+0.525$ & $+0.416$ & $+0.456$ & $+0.515$ & $+0.539$ & $+0.574$ & $+0.513$ \\
1B / PG19            & $+1.819$ & $+3.102$ & $+3.401$ & $+3.861$ & $+3.636$ & $+3.455$ & $+3.585$ \\
8B / Proof-Pile      & $+0.057$ & $+0.056$ & $+0.072$ & $+0.075$ & $+0.076$ & $+0.074$ & \textit{n/a} \\
8B / PG19            & $+0.206$ & $+0.371$ & $+0.421$ & $+0.480$ & $+0.459$ & $+0.394$ & \textit{n/a} \\
Mistral / Proof-Pile & $+0.043$ & $+0.035$ & $+0.035$ & $+0.034$ & $+0.038$ & $+0.056$ & \textit{n/a} \\
Mistral / PG19       & $+0.113$ & $+0.159$ & $+0.196$ & $+0.223$ & $+0.189$ & $+0.238$ & \textit{n/a} \\
\bottomrule
\end{tabular}
\caption{Cross-model concatenated-stream perplexity differences under dynamic W4A4KV4
with \textsc{whip} calibration, $n=3$. Positive values indicate higher perplexity for
pairwise-only than for full-head Hadamard. Every reported within-length mean is
positive. Mistral-7B-v0.3 differences are small and are interpreted as cross-model
context rather than as individually decisive comparisons.}
\label{tab:data-breadth-crossover}
\end{table}

\paragraph{Cross-model NIAH.} Depth-averaged NIAH retrieval accuracy is also evaluated
on the corresponding quantised model instances (seed 0; accuracy, so higher is better).
Llama-3.2-1B provides sufficient separation at long context and follows the perplexity
ordering: at 16K/32K/64K, full-head Hadamard obtains $0.68/0.59/0.51$, pairwise-only
obtains $0.53/0.47/0.36$, and pairwise\,+\,Hadamard obtains $0.69/0.62/0.54$. The FP16
value is approximately $1.0$ at 16K. Llama-3.1-8B remains near FP16 saturation at all
evaluated lengths, with all transformations between approximately $0.95$ and $0.99$, and
therefore does not distinguish them. This is the saturation counterpart of \S G. The complete
model-by-length-by-depth results, including Mistral-7B-v0.3, are recorded in
\path{docs/data_db/paper/appI_breadth_niah.csv}.

\subsection*{J. Pooled-covariance surrogate optimum with deployed RoPE frequencies (Llama-3.2-3B)}
\label{app:data-star}
This section reports the only $\phi_k^\ast$ evaluations used in the manuscript. All
configurations use Llama-3.2-3B, dynamic W4A4KV4, \textsc{whip} calibration, a single
recorded software environment, and one source state. Inverse frequencies are read from
the checkpoint's live rotary module, including \texttt{llama3} frequency scaling.
Reconstructing frequencies from the base parameter would misstate $63$ of the $64$
frequency pairs, with errors of up to $32\times$ in the lowest-frequency pairs.

\paragraph{Verification of the analytic optimum.}
For every position-averaged configuration, the verification procedure reconstructs the
surrogate objective from covariances formed after raw-moment aggregation and from the deployed frequency vector, then
compares the achieved per-pair minimax variance with the analytic minimum $\bar\lambda_k$ of
Theorem~\ref{thm:closed-form-optimal-angle}. All $30$ verification records satisfy the
criterion. The maximum excess over the analytic minimum, across all layers, frequency pairs, and
evaluated configurations, is $6.90\times10^{-7}$, compared with a tolerance of
$5\times10^{-5}$.

Applying the same verification procedure to alternative transformations yields mean
minimax excesses of $0.433$ for identity, $0.375$ for the per-pair $\bH_2$, and $0.0246$
for $\hat\phi_k$ (raw or branch-matched alike, the branch shift being
objective-preserving). The corresponding value for $\phi_k^\ast$ is $1.26\times10^{-8}$.
The implementation therefore reproduces the analytic ordering. End-to-end perplexity
does not follow the same ordering: the position-averaged surrogate optimum attains the analytic
minimum but yields the largest short-context pairwise-only perplexity difference.
\S\ref{sec:mechanism} gives the conditioning reason why the position-averaging
correction is largest, and least meaningful, on precisely the frequency pairs where
$|(C_k,S_k)|\approx0$.

\paragraph{Conditioning of the finite-position correction.}
At $L=2048$, $17$ of the $64$ deployed frequency pairs satisfy $|(C_k,S_k)|<0.01$. For
these pairs, $|\tfrac12\operatorname{atan2}(S_k,C_k)|$ has mean $0.75$ radians and
maximum $1.51$ radians. At $L=8192$, $25$ of $64$ pairs satisfy the same threshold. When
both finite sums are close to zero, the position-averaged covariance is close to
isotropic and many angles are nearly optimal for the variance objective. The selected
branch can therefore depend strongly on the small finite-sum remainder while the
objective remains nearly flat over angle. This qualifies the deployment interpretation
of $\phi_k^\ast$ without affecting the analytic verification.

\paragraph{Short-context comparison (2K WikiText-2).}
Table~\ref{tab:data-star-short} reports per-seed WikiText-2 perplexity for the full-head
Hadamard, identity, per-pair $\bH_2$, pairwise-only, and pairwise\,+\,Hadamard
configurations. Table~\ref{tab:data-star-short-contrasts} reports the corresponding
paired differences, which back Table~\ref{tab:corrected-star-short} of the body. The
branch-matched controls are evaluated for seeds 0--2; comparisons involving those
controls therefore use the shared three-seed subset rather than the six-seed means.

\begin{table}[ht]\centering\small
\resizebox{\linewidth}{!}{%
\begin{tabular}{lrrrrrrcc}
\toprule
configuration & s0 & s1 & s2 & s3 & s4 & s5 & mean & sd \\
\midrule
full-head Hadamard        & $10.3690$ & $9.7748$ & $10.1030$ & $10.5908$ & $9.9134$ & $9.9216$ & $10.1121$ & $0.3114$ \\
identity (no $R_3$)       & $10.8305$ & $10.2184$ & $10.5697$ & $11.1342$ & $10.3311$ & $10.3747$ & $10.5764$ & $0.3478$ \\
per-pair $\bH_2$          & $10.8245$ & $10.2020$ & $10.5264$ & $11.0978$ & $10.3243$ & $10.3443$ & $10.5532$ & $0.3436$ \\
pairwise-only, $\hat\phi_k$        & $10.8570$ & $10.2054$ & $10.5943$ & $11.1317$ & $10.3851$ & $10.4001$ & $10.5956$ & $0.3437$ \\
pairwise-only, $\phi_k^\ast$       & $10.9367$ & $10.2610$ & $10.6261$ & $11.1793$ & $10.4607$ & $10.4358$ & $10.6499$ & $0.3452$ \\
pairwise\,+\,Hadamard, $\hat\phi_k$      & $10.3724$ & $9.7916$ & $10.0952$ & $10.6043$ & $9.9101$ & $9.9309$ & $10.1174$ & $0.3119$ \\
pairwise\,+\,Hadamard, $\phi_k^\ast$     & $10.3690$ & $9.8161$ & $10.1120$ & $10.6202$ & $9.9188$ & $9.9447$ & $10.1301$ & $0.3083$ \\
\midrule
\multicolumn{9}{l}{\emph{branch-matched controls (seeds 0--2)}}\\
branch-matched pairwise-only, $\hat\phi_k$    & $10.9052$ & $10.2581$ & $10.6349$ & \textit{n/a} & \textit{n/a} & \textit{n/a} & $10.5994$ & $0.3250$ \\
branch-matched pairwise\,+\,Hadamard, $\hat\phi_k$  & $10.3747$ & $9.7889$ & $10.1170$ & \textit{n/a} & \textit{n/a} & \textit{n/a} & $10.0935$ & $0.2936$ \\
\bottomrule
\end{tabular}%
}
\caption{Per-seed WikiText-2 perplexity for Llama-3.2-3B under dynamic W4A4KV4 with
\textsc{whip} calibration. The primary transformations use six seeds; branch-matched
controls use seeds 0--2. The table supports the short-context comparison in
Table~\ref{tab:corrected-star-short}.}
\label{tab:data-star-short}
\end{table}

\begin{table}[ht]\centering\small
\resizebox{\linewidth}{!}{%
\begin{tabular}{llrrrl}
\toprule
configuration & comparator & $n$ & mean difference & sample sd & $90\%$ CI \\
\midrule
identity (no $R_3$)    & full-head Hadamard & 6 & $+0.4643$ & $0.0424$ & $[+0.4294,+0.4992]$ \\
per-pair $\bH_2$       & full-head Hadamard & 6 & $+0.4411$ & $0.0355$ & $[+0.4119,+0.4703]$ \\
pairwise-only, $\hat\phi_k$     & full-head Hadamard & 6 & $+0.4835$ & $0.0356$ & $[+0.4542,+0.5128]$ \\
pairwise-only, $\phi_k^\ast$    & full-head Hadamard & 6 & $+0.5378$ & $0.0374$ & $[+0.5071,+0.5686]$ \\
pairwise\,+\,Hadamard, $\hat\phi_k$   & full-head Hadamard & 6 & $+0.0053$ & $0.0096$ & $[-0.0026,+0.0132]$ \\
pairwise\,+\,Hadamard, $\phi_k^\ast$  & full-head Hadamard & 6 & $+0.0180$ & $0.0159$ & $[+0.0050,+0.0311]$ \\
\midrule
branch-matched pairwise-only, $\hat\phi_k$    & pairwise-only, $\hat\phi_k$            & 3 & $+0.0472$ & $0.0061$ & $[+0.0369,+0.0575]$ \\
pairwise-only, $\phi_k^\ast$             & branch-matched pairwise-only, $\hat\phi_k$  & 3 & $+0.0085$ & $0.0207$ & $[-0.0264,+0.0435]$ \\
branch-matched pairwise\,+\,Hadamard, $\hat\phi_k$  & pairwise\,+\,Hadamard, $\hat\phi_k$          & 3 & $+0.0071$ & $0.0129$ & $[-0.0147,+0.0290]$ \\
pairwise\,+\,Hadamard, $\phi_k^\ast$           & branch-matched pairwise\,+\,Hadamard, $\hat\phi_k$ & 3 & $+0.0055$ & $0.0188$ & $[-0.0262,+0.0372]$ \\
\bottomrule
\end{tabular}%
}
\caption{Paired short-context perplexity differences for the position-averaged
analysis. Standard deviations use the sample convention, and intervals are $90\%$
paired $t$-intervals. Both pairwise\,+\,Hadamard comparisons against full-head Hadamard
lie within the selected $\pm0.05$-PPL interval criterion. Three of the four branch-control comparisons
also satisfy this criterion; the interval for branch-matched minus raw pairwise-only
$\hat\phi_k$ extends to $+0.0575$ PPL.}
\label{tab:data-star-short-contrasts}
\end{table}

\paragraph{Long-context comparison.}
Table~\ref{tab:data-star-long} reports mean concatenated-stream perplexity at 2K, 32K,
and 128K. Each $\phi_k^\ast$ configuration uses the averaging length at which it is
evaluated. Independent evaluations of the full-head Hadamard, pairwise-only
$\hat\phi_k$, and pairwise\,+\,Hadamard $\hat\phi_k$ configurations agree with the
corresponding Table~\ref{tab:data-crossover} values to all printed digits across all
$18$ model-seed-corpus-length combinations. The two result sets are therefore directly
comparable at the reported precision. Table~\ref{tab:data-star-long-contrasts} reports the paired $\phi_k^\ast$ comparisons and the Holm adjustment of their two-sided $p$-values.

\begin{table}[ht]\centering\small
\resizebox{\linewidth}{!}{%
\begin{tabular}{rccccccc}
\toprule
 & & \multicolumn{3}{c}{pairwise-only} & \multicolumn{3}{c}{pairwise\,+\,Hadamard} \\
\cmidrule(lr){3-5}\cmidrule(lr){6-8}
$L$ & full-head Hadamard & $\hat\phi_k$ & branch-matched $\hat\phi_k$ & $\phi_k^\ast$ & $\hat\phi_k$ & branch-matched $\hat\phi_k$ & $\phi_k^\ast$ \\
\midrule
\multicolumn{8}{l}{\emph{Proof-Pile}}\\
2K   & $4.5357$ & $4.6963$ & $4.7358$ & $4.7342$ & $4.5358$ & $4.5347$ & $4.5381$ \\
32K  & $3.2902$ & $3.4812$ & $3.4887$ & $3.4931$ & $3.2938$ & $3.2972$ & $3.2944$ \\
128K & $2.9331$ & $3.1257$ & $3.1207$ & $3.1093$ & $2.9370$ & $2.9356$ & $2.9388$ \\
\midrule
\multicolumn{8}{l}{\emph{PG19}}\\
2K   & $10.9528$ & $11.5183$ & $11.5813$ & $11.5683$ & $10.9373$ & $10.9760$ & $10.9546$ \\
32K  & $14.7352$ & $15.8883$ & $15.9415$ & $15.8977$ & $14.7511$ & $14.7505$ & $14.7574$ \\
128K & $12.4838$ & $13.6527$ & $13.5616$ & $13.5370$ & $12.4875$ & $12.4933$ & $12.4912$ \\
\bottomrule
\end{tabular}}
\caption{Mean concatenated-stream perplexity for Llama-3.2-3B under dynamic W4A4KV4
with \textsc{whip} calibration, $n=3$. Position-averaged surrogate configurations use the context
length at which they are evaluated. Absolute perplexity should not be compared across
lengths. These values support the long-context comparisons in \S\ref{sec:results_long_context}.}
\label{tab:data-star-long}
\end{table}

\begin{table}[ht]\centering\small
\begin{tabular}{llrrll}
\toprule
corpus & $L$ & mean difference & sample sd & $90\%$ CI & $p$ \\
\midrule
\multicolumn{6}{l}{\emph{pairwise-only: $\phi_k^\ast$ $-$ branch-matched $\hat\phi_k$}}\\
Proof-Pile & 2K   & $-0.0016$ & $0.0177$ & $[-0.0315,+0.0283]$ & $.8906$ \\
Proof-Pile & 32K  & $+0.0044$ & $0.0133$ & $[-0.0180,+0.0268]$ & $.6225$ \\
Proof-Pile & 128K & $-0.0114$ & $0.0083$ & $[-0.0253,+0.0025]$ & $.1395$ \\
PG19       & 2K   & $-0.0130$ & $0.0154$ & $[-0.0389,+0.0129]$ & $.2805$ \\
PG19       & 32K  & $-0.0438$ & $0.0342$ & $[-0.1015,+0.0139]$ & $.1568$ \\
PG19       & 128K & $-0.0246$ & $0.0088$ & $[-0.0395,-0.0097]$ & $.0404$ \\
\midrule
\multicolumn{6}{l}{\emph{pairwise\,+\,Hadamard: $\phi_k^\ast$ $-$ branch-matched $\hat\phi_k$}}\\
Proof-Pile & 2K   & $+0.0034$ & $0.0153$ & $[-0.0223,+0.0292]$ & $.7335$ \\
Proof-Pile & 32K  & $-0.0028$ & $0.0036$ & $[-0.0089,+0.0033]$ & $.3095$ \\
Proof-Pile & 128K & $+0.0032$ & $0.0028$ & $[-0.0016,+0.0080]$ & $.1888$ \\
PG19       & 2K   & $-0.0214$ & $0.0043$ & $[-0.0286,-0.0141]$ & $.0132$ \\
PG19       & 32K  & $+0.0070$ & $0.0066$ & $[-0.0042,+0.0181]$ & $.2102$ \\
PG19       & 128K & $-0.0021$ & $0.0050$ & $[-0.0106,+0.0063]$ & $.5387$ \\
\bottomrule
\end{tabular}
\caption{Paired comparisons between $\phi_k^\ast$ and branch-matched $\hat\phi_k$,
$n=3$. Standard deviations use the sample convention and intervals are $90\%$ paired
$t$-intervals. Eleven of twelve intervals lie within the selected $\pm0.05$-PPL criterion; the PG19 32K pairwise-only comparison is inconclusive under that criterion. The twelve two-sided comparisons form one multiplicity family. Neither of the two smallest unadjusted $p$-values ($.0132$, $.0404$) passes the first Holm threshold $0.05/12=.0042$, so no adjusted comparison rejects its corresponding null hypothesis of zero mean difference.}
\label{tab:data-star-long-contrasts}
\end{table}

\subsection*{K. Local FPTQuant-style Q/K component (Llama-3.2-3B)}
\label{app:data-fptqk}
Table~\ref{tab:data-fptqk} reports a local FPTQuant-style \citep{fptquant2025} per-pair
$Q/K$ component implemented within the dynamic W4A4KV4 evaluation procedure used in this paper. All compared
components use $128$ optimisation steps, learning rate $0.02$, and $128$ calibration
sequences divided into $112$ optimisation and $16$ held-out sequences. $Q$ and $K$ use
the same fake-quantisation settings. The table reports the objective at the first and
final optimisation steps together with held-out final-token logit MSE.

\begin{table}[ht]\centering\small
\begin{tabular}{llrrrr}
\toprule
component & seed & PPL & first loss & last loss & held-out MSE \\
\midrule
rotation            & 0 & $11.0218$ & $1.3816$ & $0.7670$ & $1.9030$ \\
                    & 1 & $10.3707$ & $2.4586$ & $0.8684$ & $1.7253$ \\
                    & 2 & $10.6160$ & $1.7880$ & $1.4216$ & $1.6383$ \\
scaling             & 0 & $11.2190$ & $1.3816$ & $3.5531$ & $2.2056$ \\
                    & 1 & $10.4694$ & $2.4586$ & $1.1914$ & $2.5308$ \\
                    & 2 & $10.8642$ & $1.7880$ & $1.5660$ & $3.8017$ \\
rotation $+$ scaling & 0 & $10.9445$ & $1.3816$ & $1.7670$ & $2.0365$ \\
                    & 1 & $10.5346$ & $2.4586$ & $0.6165$ & $1.9659$ \\
                    & 2 & $10.8299$ & $1.7880$ & $1.2372$ & $2.5088$ \\
\bottomrule
\end{tabular}
\caption{Per-seed results for the local FPTQuant-style $Q/K$ component on WikiText-2
under dynamic W4A4KV4 with \textsc{whip} calibration. Comparisons use the full-head
Hadamard seeds from Table~\ref{tab:data-star-short}. The reported optimisation objective
is not monotone and does not track perplexity across seeds; no optimisation-budget
conclusion is drawn.}
\label{tab:data-fptqk}
\end{table}

\emph{Scope.} This component is an independent reimplementation of an
FPTQuant-\emph{style} local map and is not a reproduction of FPTQuant. It therefore does
not bear on the results reported for that method. A single-seed comparison at $32$,
$128$, and $512$ optimisation steps yields perplexities
$10.9755/11.0653/12.0509$. Repeating the $128$-step condition at the same seed changes
both the optimisation trace and perplexity by more than $0.12$ PPL, which exceeds the
observed $32\to128$ step difference. The deterministic portions of the procedure
reproduce exactly, localising this variability to the component optimisation. For this
reason, the manuscript does not interpret these values as an optimisation-budget trend.

\subsection*{L. Angle-estimator comparison}
\label{app:data-pooling}
This section reports the full angle-estimator analysis underlying the four-checkpoint
summary in the main text. The only varied quantity is the aggregation rule used to estimate
the shared (layer, frequency-pair) angle from calibration statistics, selected by
\texttt{--equalising\_r3\_pool}:
\begin{itemize}[leftmargin=1.4em,itemsep=1pt,topsep=2pt]
  \item \textbf{default $Q/K$-weighted estimator} (\texttt{rows}): raw $Q$ and $K$
        sufficient statistics are summed and centred once, weighting the streams by physical head count;
  \item \textbf{K-only estimator} (\texttt{k\_only}): the angle is estimated from $K$
        statistics alone;
  \item \textbf{stream-balanced estimator} (\texttt{balanced}): per-stream first and uncentred second moments are averaged so that $Q$ and $K$ receive equal stream-level weight, followed by one centring operation.
\end{itemize}
Every configuration applies one shared angle to $Q$ and $K$.
Proposition~\ref{prop:rope-commutativity-requires-shared-theta} therefore applies
within every head. The angle is also shared across heads, so the evaluated transformations occupy the head-shared subset of the head-dependent direct-product family. The implementation path for the default $Q/K$-weighted estimator is unchanged.

The comparisons in this section are exploratory. All reported intervals are unadjusted $90\%$ paired $t$-intervals, and no family-wise error adjustment is applied. We therefore report effect estimates and cross-checkpoint consistency without assigning confirmatory significance to individual estimator contrasts. Multiplicity-adjusted testing is restricted to the twelve two-sided contrasts in \S J.

\paragraph{Comparability with the reported short-context values.}
Six configurations from Table~\ref{tab:corrected-star-short} were evaluated again within
this source state. Across all six, the largest absolute difference from the previously
reported value is $3.3\times10^{-5}$ PPL, equal to the effect of four-decimal rounding.
The estimator comparisons are therefore made within a common source state and are
directly comparable with the reported values at the stated precision. Table~\ref{tab:data-pooling-short} reports the per-seed perplexities, and Table~\ref{tab:data-pooling-contrasts} reports the paired estimator contrasts.

\begin{table}[ht]\centering\small
\resizebox{\linewidth}{!}{%
\begin{tabular}{lrrrrrrcc}
\toprule
configuration & s0 & s1 & s2 & s3 & s4 & s5 & mean & sd \\
\midrule
\multicolumn{9}{l}{\emph{pairwise-only}}\\
$\hat\phi_k$, default $Q/K$-weighted   & $10.8570$ & $10.2054$ & $10.5943$ & $11.1317$ & $10.3851$ & $10.4001$ & $10.5956$ & $0.3437$ \\
$\hat\phi_k$, stream-balanced          & $10.8137$ & $10.2037$ & $10.5043$ & $11.0659$ & $10.3208$ & $10.3420$ & $10.5417$ & $0.3326$ \\
$\hat\phi_k$, K-only                   & $10.7260$ & $10.0706$ & $10.4145$ & $10.9185$ & $10.2224$ & $10.2218$ & $10.4290$ & $0.3293$ \\
$\phi_k^\ast$, default $Q/K$-weighted  & $10.9367$ & $10.2610$ & $10.6261$ & $11.1793$ & $10.4607$ & $10.4358$ & $10.6499$ & $0.3452$ \\
$\phi_k^\ast$, K-only                  & $10.8161$ & $10.1952$ & $10.5095$ & $11.0157$ & $10.3437$ & $10.3334$ & $10.5356$ & $0.3173$ \\
\midrule
\multicolumn{9}{l}{\emph{pairwise\,+\,Hadamard}}\\
$\hat\phi_k$, default $Q/K$-weighted   & $10.3724$ & $9.7916$ & $10.0952$ & $10.6043$ & $9.9101$ & $9.9309$ & $10.1174$ & $0.3119$ \\
$\hat\phi_k$, stream-balanced          & $10.3753$ & $9.8046$ & $10.0980$ & $10.6096$ & $9.9359$ & $9.9579$ & $10.1302$ & $0.3048$ \\
$\hat\phi_k$, K-only                   & $10.3811$ & $9.8188$ & $10.1310$ & $10.6337$ & $9.9387$ & $9.9646$ & $10.1447$ & $0.3086$ \\
$\phi_k^\ast$, default $Q/K$-weighted  & $10.3690$ & $9.8161$ & $10.1120$ & $10.6202$ & $9.9188$ & $9.9447$ & $10.1301$ & $0.3083$ \\
$\phi_k^\ast$, K-only                  & $10.3747$ & $9.7992$ & $10.0952$ & $10.5902$ & $9.9287$ & $9.9442$ & $10.1220$ & $0.3023$ \\
\bottomrule
\end{tabular}%
}
\caption{Per-seed WikiText-2 perplexity for the Llama-3.2-3B angle-estimator comparison
under dynamic W4A4KV4 with \textsc{whip} calibration. Full-head Hadamard and identity
are evaluated in the same source state and agree with the corresponding
Table~\ref{tab:data-star-short} values to the stated tolerance.}
\label{tab:data-pooling-short}
\end{table}

\begin{table}[ht]\centering\small
\resizebox{\linewidth}{!}{%
\begin{tabular}{llrrrl}
\toprule
configuration & comparator & $n$ & mean difference & sample sd & $90\%$ CI \\
\midrule
\multicolumn{6}{l}{\emph{pooling-rule comparisons, Llama-3.2-3B}}\\
pairwise-only, $\hat\phi_k$, K-only   & pairwise-only, $\hat\phi_k$, default   & 6 & $-0.1666$ & $0.0309$ & $[-0.1921,-0.1412]$ \\
pairwise-only, $\hat\phi_k$, stream-balanced  & pairwise-only, $\hat\phi_k$, default   & 6 & $-0.0539$ & $0.0297$ & $[-0.0783,-0.0294]$ \\
pairwise-only, $\phi_k^\ast$, K-only  & pairwise-only, $\phi_k^\ast$, default  & 6 & $-0.1143$ & $0.0316$ & $[-0.1403,-0.0884]$ \\
pairwise\,+\,Hadamard, $\hat\phi_k$, K-only    & pairwise\,+\,Hadamard, $\hat\phi_k$, default    & 6 & $+0.0272$ & $0.0097$ & $[+0.0193,+0.0352]$ \\
pairwise\,+\,Hadamard, $\hat\phi_k$, stream-balanced   & pairwise\,+\,Hadamard, $\hat\phi_k$, default    & 6 & $+0.0128$ & $0.0112$ & $[+0.0036,+0.0220]$ \\
pairwise\,+\,Hadamard, $\phi_k^\ast$, K-only   & pairwise\,+\,Hadamard, $\phi_k^\ast$, default   & 6 & $-0.0081$ & $0.0156$ & $[-0.0209,+0.0047]$ \\
\midrule
\multicolumn{6}{l}{\emph{comparisons with identity (no $R_3$)}}\\
pairwise-only, $\hat\phi_k$, default      & identity & 6 & $+0.0192$ & $0.0238$ & $[-0.0004,+0.0388]$ \\
pairwise-only, $\hat\phi_k$, K-only   & identity & 6 & $-0.1475$ & $0.0403$ & $[-0.1806,-0.1144]$ \\
pairwise-only, $\hat\phi_k$, stream-balanced  & identity & 6 & $-0.0347$ & $0.0260$ & $[-0.0561,-0.0133]$ \\
pairwise-only, $\phi_k^\ast$, default     & identity & 6 & $+0.0735$ & $0.0359$ & $[+0.0440,+0.1030]$ \\
pairwise-only, $\phi_k^\ast$, K-only  & identity & 6 & $-0.0408$ & $0.0454$ & $[-0.0781,-0.0035]$ \\
\bottomrule
\end{tabular}%
}
\caption{Exploratory paired comparisons within the Llama-3.2-3B angle-estimator analysis, using six shared seeds, sample standard deviations, and unadjusted $90\%$ paired $t$-intervals. The first block compares aggregation rules at fixed transformation structure; the second compares each pairwise-only configuration with identity. Under the default estimator, the interval for pairwise-only $\hat\phi_k$ minus identity includes zero. The K-only and stream-balanced estimates and their intervals are negative. These intervals are reported for effect estimation and are not multiplicity-adjusted.}
\label{tab:data-pooling-contrasts}
\end{table}

\paragraph{Cross-model estimator effects.}
Table~\ref{tab:pooling-breadth} of the body reports the four-checkpoint comparison with six paired seeds. For pairwise-only, all four K-only-minus-default estimates are negative, and their unadjusted $90\%$ intervals exclude zero: $-0.5286$ $[-0.5747,-0.4825]$ on Llama-3.2-1B, $-0.1666$ $[-0.1921,-0.1412]$ on Llama-3.2-3B, $-0.0661$ $[-0.0797,-0.0525]$ on Llama-3.1-8B, and $-0.0071$ $[-0.0091,-0.0051]$ on Mistral-7B-v0.3. For pairwise\,+\,Hadamard, the corresponding estimates are $+0.0559$ $[+0.0348,+0.0771]$, $+0.0272$ $[+0.0193,+0.0352]$, $+0.0129$ $[+0.0027,+0.0231]$, and $+0.0018$ $[-0.0013,+0.0050]$. The first three intervals lie above zero, whereas the Mistral-7B-v0.3 interval includes zero. The repeated direction across checkpoints supports estimator sensitivity within the evaluated configurations, but these exploratory comparisons do not constitute separately multiplicity-controlled advantage claims.

\paragraph{Software-environment separation.}
The estimator analysis, including its cross-model evaluations, uses the A100 environment
employed for the primary Llama-3.2-3B results. \S I uses the second hardware
environment. Absolute full-head Hadamard perplexity therefore differs between the two
result sets: $14.482$ versus $14.427$ for Llama-3.2-1B, $7.919$ versus $7.825$ for
Llama-3.1-8B, and $5.847$ versus $5.892$ for Mistral-7B-v0.3. Results from the two
environments are not pooled. Within the estimator-analysis environment, the default
$Q/K$-weighted pairwise\,+\,Hadamard differences independently reproduce the magnitudes
reported in \S I at $n=6$ and lead to the same conclusion under the selected $\pm0.05$-PPL interval criterion. The signs
agree on Llama-3.2-1B and Mistral-7B-v0.3 but not on Llama-3.1-8B ($-0.0002$ here
versus $+0.0063$ in \S I); both values lie far inside the $\pm0.05$-PPL margin, and at
this magnitude the sign is not resolvable across environments.

\subsection*{M. Mixing-support interpolation (Llama-3.2-3B)}
\label{app:data-mixing-support}

The mixing-support analysis varies the block size of a pair-interleaved block-Hadamard
transformation in the same Llama-3.2-3B, WikiText-2, dynamic W4A4KV4 setting as the
primary short-context comparison. The transformation orientation is normalised to the
$SO(2)$ convention. Block sizes $b\in\{2,4,8,16,32,64,128\}$ each use three paired
perplexity seeds and three K-diagnostic seeds. Channel-contiguous layout and
raw-orientation controls use one seed and are reported descriptively. Complete
machine-readable values are in \path{docs/data_db/clean/mixing_support.csv}.

Mean K range, relative K quantisation error, and the paired perplexity difference
relative to full-head Hadamard decrease monotonically as block size increases. The $b=2$
diagnostic agrees with the per-pair $H_2$ endpoint to the reported precision. At
$b=128$, relative K quantisation error matches the full-head endpoint at the reported
precision, mean K range differs by $0.0011$, and the paired perplexity interval includes
zero. Only $b=2$ commutes with the RoPE
frequency-pair blocks; larger blocks deliberately leave the centraliser to test the
effect of broader mixing support. The interpolation therefore evaluates the support
hypothesis rather than proposing an additional member of the commuting family.

\begin{table}[ht]
\centering\scriptsize
\resizebox{\linewidth}{!}{%
\begin{tabular}{llccrrrrrr}
\toprule
configuration & channel layout / orientation & RoPE-commuting & $n$ & PPL mean (sd) & PPL difference vs full-head Hadamard $[90\%\ \mathrm{CI}]$ & mean K range & max K range & $\Delta$ & relative K quantisation error \\
\midrule
$b=2$   & pair/$SO(2)$ & yes & 3 & $10.4981\ (.3234)$ & $+.4158\ [+.3720,+.4596]$ & $14.815614$ & $36.458333$ & $.987708$ & $.02305523$ \\
$b=4$   & pair/$SO(2)$ & no  & 3 & $10.3085\ (.3087)$ & $+.2262\ [+.2069,+.2456]$ & $13.118872$ & $38.479167$ & $.874591$ & $.01805045$ \\
$b=8$   & pair/$SO(2)$ & no  & 3 & $10.2195\ (.3026)$ & $+.1373\ [+.1242,+.1503]$ & $11.576764$ & $35.875000$ & $.771784$ & $.01385570$ \\
$b=16$  & pair/$SO(2)$ & no  & 3 & $10.1573\ (.2877)$ & $+.0751\ [+.0505,+.0996]$ & $10.733735$ & $31.791667$ & $.715582$ & $.01183181$ \\
$b=32$  & pair/$SO(2)$ & no  & 3 & $10.1367\ (.2833)$ & $+.0544\ [+.0266,+.0822]$ & $10.447527$ & $33.645833$ & $.696502$ & $.01122667$ \\
$b=64$  & pair/$SO(2)$ & no  & 3 & $10.0976\ (.2913)$ & $+.0153\ [-.0286,+.0592]$ & $10.135678$ & $30.375000$ & $.675712$ & $.01053918$ \\
$b=128$ & pair/$SO(2)$ & no  & 3 & $10.0861\ (.2847)$ & $+.0039\ [-.0236,+.0313]$ & $9.643106$ & $27.062500$ & $.642874$ & $.00953660$ \\
\midrule
$b=2$, raw orientation & pair/raw & yes & 1 & $10.8119$ & $+.4429$ & $14.748907$ & $34.125000$ & $.983260$ & $.02307303$ \\
$b=8$, channel-contiguous layout control & channel/$SO(2)$ & no & 1 & $10.4793$ & $+.1103$ & $11.125169$ & $33.000000$ & $.741678$ & $.01284636$ \\
$b=32$, channel-contiguous layout control & channel/$SO(2)$ & no & 1 & $10.3926$ & $+.0236$ & $10.207206$ & $28.875000$ & $.680480$ & $.01080546$ \\
\midrule
full-head Hadamard & full head & no & 3 & $10.0823\ (.2976)$ & $0$ & $9.641980$ & $27.145833$ & $.642799$ & $.00953662$ \\
identity & \textit{n/a} & no & 3 & $10.5395\ (.3072)$ & $+.4573\ [+.4368,+.4777]$ & $15.851159$ & $36.145833$ & $1.056744$ & $.02634069$ \\
per-pair $H_2$ & RoPE pair & yes & 3 & $10.5176\ (.3113)$ & $+.4354\ [+.4058,+.4649]$ & $14.815659$ & $36.375000$ & $.987711$ & $.02305625$ \\
\bottomrule
\end{tabular}%
}
\caption{Complete mixing-support comparison. The primary interpolation uses three paired
seeds at each block size; the orientation and channel-layout controls use one seed and
therefore have no confidence interval. K diagnostics are measured after $R_3$ with the
same procedure for every configuration.}
\label{tab:data-mixing-support}
\end{table}

\section{Proofs}
\label{app:proofs}
Throughout this appendix $G(\phi) := \bigl(\begin{smallmatrix}\cos\phi & -\sin\phi\\ \sin\phi & \cos\phi\end{smallmatrix}\bigr)\in SO(2)$ denotes the Givens rotation in the orientation fixed by Definition~\ref{def:band-aligned} (so that $\bH_2 = G(\pi/4)\,\mathrm{diag}(1,-1)$, as used in the body), $\Rcal_m^{(k)} = G(m\theta_k)$ is the band-$k$ RoPE block, and $\Rcal_m = \bigoplus_{k=1}^{K}\Rcal_m^{(k)}$ with $K=\dh/2$. We write the band-$k$ pre-RoPE block as $\bSigma^{(k)} = \bigl(\begin{smallmatrix}\sigma_1^2 & \rho\sigma_1\sigma_2\\ \rho\sigma_1\sigma_2 & \sigma_2^2\end{smallmatrix}\bigr)$ and reserve $\widetilde\bSigma^{(k)}$ for the pooled-covariance, position-averaged surrogate of Eq.~\eqref{eq:Sigmatilde}. In the empirical application, $\bSigma^{(k)}$ is reconstructed after raw calibration moments have been aggregated across observation rows. The analysis treats this fixed matrix as the pre-RoPE covariance at every position; it does not estimate position-conditional matrices $\bSigma_m^{(k)}$. The algebra below is exact for this fixed-matrix surrogate.

\begin{proof}[Proof of Lemma~\ref{lem:centraliser-commutes}]
We prove both inclusions.

\emph{(a) Band-aligned $\Rightarrow$ commutes.}
Let $R_3 = \bigoplus_{k=1}^{K} G(\phi_k)\in\Bcal_{\dh}$. Both $R_3$ and $\Rcal_m=\bigoplus_k\Rcal_m^{(k)}$ are block-diagonal under the same partition of $\{1,\dots,\dh\}$ into the $K$ RoPE coordinate pairs $(2k{-}1,2k)$. The product of two block-diagonal matrices sharing a block partition is the block-diagonal matrix of the blockwise products, so it suffices to commute within each $2\times2$ block. There both factors lie in $SO(2)$: $G(\phi_k)$ by definition and $\Rcal_m^{(k)}=G(m\theta_k)$. Since $SO(2)$ is abelian with $G(\alpha)G(\beta)=G(\alpha+\beta)=G(\beta)G(\alpha)$,
\[
G(\phi_k)\,\Rcal_m^{(k)} = G(\phi_k+m\theta_k) = \Rcal_m^{(k)}\,G(\phi_k)\qquad(k=1,\dots,K),
\]
and assembling the blocks gives $R_3\Rcal_m=\Rcal_m R_3$ for every $m\in\Z$. Hence $\Bcal_{\dh}$ lies in the centraliser of $\{\Rcal_m\}_{m\in\Z}$ inside $O(\dh)$.

\emph{(b) Commutes with all $m$ $\Rightarrow$ band-diagonal, under distinct frequencies.}
Let $A\in O(\dh)$ satisfy $A\Rcal_m=\Rcal_m A$ for all $m\in\Z$. It suffices to use $m=1$. Work over the complexification $\C^{\dh}=\R^{\dh}\otimes\C$. The block $\Rcal_1^{(k)}=G(\theta_k)$ has characteristic polynomial $\lambda^2-2\cos\theta_k\,\lambda+1$, hence eigenvalues $e^{\pm i\theta_k}$ with conjugate eigenvectors $v_k^{\pm}=e_{2k-1}\mp i\,e_{2k}$ spanning the complexification of the band-$k$ real $2$-plane $P_k:=\mathrm{span}_\R\{e_{2k-1},e_{2k}\}$. Under the lemma's assumptions, $\theta_k\in(0,\pi)$ and the $\theta_k$ are pairwise distinct, so the $2K=\dh$ numbers $\{e^{\pm i\theta_k}\}_{k=1}^K$ are all distinct: $e^{i\theta_k}=e^{i\theta_{k'}}$ forces $\theta_k=\theta_{k'}$, while $e^{i\theta_k}=e^{-i\theta_{k'}}$ forces $\theta_k+\theta_{k'}\in2\pi\Z$, impossible on $(0,\pi)$. Thus $\Rcal_1$ has $\dh$ distinct eigenvalues and each eigenspace is one-dimensional. Any $A$ commuting with $\Rcal_1$ preserves each eigenspace; being real, $A$ maps the conjugate pair $\{v_k^+,v_k^-\}$ to itself (a real matrix sends $\overline{v}$ to $\overline{Av}$), hence preserves its real span $P_k$. Therefore $A$ leaves each band $2$-plane $P_k$ invariant and is block-diagonal under the band partition, $A=\bigoplus_k A_k$ with $A_k\in O(2)$ (orthogonality of $A$ restricts to the invariant $2$-plane).

It remains to place each $A_k$ inside $SO(2)$. Restricted to $P_k$ the commutation reads $A_k G(\theta_k)=G(\theta_k)A_k$ with $\theta_k\notin\pi\Z$, so $G(\theta_k)$ is a non-trivial rotation. If $\det A_k=+1$ then $A_k=G(\psi_k)\in SO(2)$, which commutes with $G(\theta_k)$ automatically. If $\det A_k=-1$ then $A_k$ is a reflection, and a reflection conjugates any rotation to its inverse: $A_k G(\theta)A_k^{-1}=G(-\theta)$, equivalently $A_k G(\theta)=G(-\theta)A_k$. Combined with the commutation hypothesis this yields $G(\theta_k)A_k=G(-\theta_k)A_k$, so $G(2\theta_k)=I$ and $\theta_k\in\pi\Z$, contradicting $\theta_k\in(0,\pi)$. Hence the reflection case cannot occur and $A_k\in SO(2)$ for every $k$, giving $A=\bigoplus_k G(\psi_k)\in\Bcal_{\dh}$. This proves the reverse inclusion, so the centraliser equals $\Bcal_{\dh}$.
\end{proof}

\paragraph{Multi-head scope.}
Lemma~\ref{lem:centraliser-commutes} is a single-head statement. If cross-head mixing is excluded, its $H$-head direct-product extension is
\[
\left\{\bigoplus_{h=1}^{H}\bigoplus_{k=1}^{K}G(\phi_{h,k})\right\},
\]
which permits head-dependent angles. The implementation evaluates the strict subset satisfying $\phi_{h,k}\equiv\phi_k$ within each layer. If mixing among the $H$ repeated copies of a frequency were permitted, the corresponding real $2H$-dimensional centraliser would contain a $U(H)$ factor. Thus the lemma identifies the commuting factor for one head, whereas the empirical conclusions apply only to the implemented head-shared subset.

\begin{proof}[Proof of Theorem~\ref{thm:closed-form-optimal-angle}]
We verify the entries of $\widetilde\bSigma^{(k)}$, reduce the problem to a $2\times2$ Jacobi rotation, solve for the equalising angle, identify the minimax branch, and establish the analytic minimum.

\emph{(0) Entries of $\widetilde\bSigma^{(k)}$.}
Fix band $k$, drop the superscript, and write $c=\cos m\theta_k$, $s=\sin m\theta_k$, so $\Rcal_m^{(k)}=\bigl(\begin{smallmatrix}c&-s\\ s&c\end{smallmatrix}\bigr)$. Direct multiplication of $\Rcal_m^{(k)}\bSigma^{(k)}(\Rcal_m^{(k)})^\top$ gives
\begin{align*}
[\,\cdot\,]_{11} &= c^2\sigma_1^2 - 2cs\,\rho\sigma_1\sigma_2 + s^2\sigma_2^2,\\
[\,\cdot\,]_{22} &= s^2\sigma_1^2 + 2cs\,\rho\sigma_1\sigma_2 + c^2\sigma_2^2,\\
[\,\cdot\,]_{12} &= cs(\sigma_1^2-\sigma_2^2) + (c^2-s^2)\rho\sigma_1\sigma_2.
\end{align*}
With $c^2=\tfrac12(1+\cos2m\theta_k)$, $s^2=\tfrac12(1-\cos2m\theta_k)$, $2cs=\sin2m\theta_k$, $c^2-s^2=\cos2m\theta_k$, these become
\begin{align*}
[\,\cdot\,]_{11} &= \tfrac{\sigma_1^2+\sigma_2^2}{2} + \tfrac{\sigma_1^2-\sigma_2^2}{2}\cos2m\theta_k - \rho\sigma_1\sigma_2\sin2m\theta_k,\\
[\,\cdot\,]_{22} &= \tfrac{\sigma_1^2+\sigma_2^2}{2} - \tfrac{\sigma_1^2-\sigma_2^2}{2}\cos2m\theta_k + \rho\sigma_1\sigma_2\sin2m\theta_k,\\
[\,\cdot\,]_{12} &= \tfrac{\sigma_1^2-\sigma_2^2}{2}\sin2m\theta_k + \rho\sigma_1\sigma_2\cos2m\theta_k.
\end{align*}
For the fixed-matrix surrogate, averaging over $m=0,\dots,L-1$ passes the constants through unchanged and replaces $\cos2m\theta_k\mapsto C_k$, $\sin2m\theta_k\mapsto S_k$ (definitions \eqref{eq:CkSk}), reproducing the displayed entries of $\widetilde\bSigma^{(k)}$.

\emph{(i) Reduction to a single Jacobi problem.}
By Lemma~\ref{lem:centraliser-commutes}, $R_3=\bigoplus_k G(\phi_k)\in\Bcal_{\dh}$ commutes with each $\Rcal_m$, and so does $R_3^\top$. Hence
\[
\frac1L\sum_{m=0}^{L-1} \Rcal_m R_3\,\bSigma\,R_3^\top\Rcal_m^\top
= R_3\Bigl(\frac1L\sum_{m=0}^{L-1} \Rcal_m\bSigma\Rcal_m^\top\Bigr)R_3^\top
\;= R_3\,\widetilde\bSigma\,R_3^\top,
\]
the first equality moving $R_3,R_3^\top$ out through the average by commutativity, the second by the definition of $\widetilde\bSigma$. We do not assume that the full covariance $\bSigma$ is band-block-diagonal: cross-band covariance blocks may be nonzero. However, because every $\Rcal_m$ and $R_3$ is block-diagonal under the band partition, the $k$th principal $2\times2$ block of $R_3\widetilde\bSigma R_3^\top$ is exactly $G(\phi_k)\,\widetilde\bSigma^{(k)}\,G(\phi_k)^\top$. Thus the objective, which uses only these channel variances, decouples band by band. This reduction applies to the fixed pooled-covariance surrogate; it does not replace $\bSigma$ with position-conditional matrices.

\emph{(ii) Equalising condition and closed form.}
Write $\widetilde\bSigma^{(k)}=\bigl(\begin{smallmatrix}\lambda_1&\beta\\ \beta&\lambda_2\end{smallmatrix}\bigr)$ with $\lambda_1=\widetilde\bSigma^{(k)}_{11}$, $\lambda_2=\widetilde\bSigma^{(k)}_{22}$, $\beta=\widetilde\bSigma^{(k)}_{12}$. With $G(\phi)$ in the orientation above, the post-rotation diagonal entries are
\begin{align*}
d_+(\phi) &= \lambda_1\cos^2\phi + \lambda_2\sin^2\phi - 2\beta\cos\phi\sin\phi,\\
d_-(\phi) &= \lambda_1\sin^2\phi + \lambda_2\cos^2\phi + 2\beta\cos\phi\sin\phi,
\end{align*}
and, via $\cos^2\phi-\sin^2\phi=\cos2\phi$, $2\sin\phi\cos\phi=\sin2\phi$,
\begin{equation}
\label{app:eq-diff}
f(\phi) := d_+(\phi)-d_-(\phi) = (\lambda_1-\lambda_2)\cos2\phi - 2\beta\sin2\phi.
\end{equation}
Equalisation $d_+=d_-$ is therefore
\begin{equation}
\label{app:eq-cond}
(\lambda_1-\lambda_2)\cos2\phi - 2\beta\sin2\phi = 0
\quad\Longleftrightarrow\quad
\tan2\phi = \frac{\lambda_1-\lambda_2}{2\beta}.
\end{equation}
The two-argument inverse tangent solves \eqref{app:eq-cond} over a full half-period and resolves the $\beta=0$ degeneracy by quadrant:
\begin{equation}
\label{app:phistar}
\phi_k^{\ast} = \tfrac12\,\mathrm{atan2}\!\bigl(\lambda_1-\lambda_2,\ 2\beta\bigr)
= \tfrac12\,\mathrm{atan2}\!\bigl(\widetilde\bSigma^{(k)}_{11}-\widetilde\bSigma^{(k)}_{22},\ 2\widetilde\bSigma^{(k)}_{12}\bigr).
\end{equation}
Indeed $\mathrm{atan2}(y,x)$ returns the angle $2\phi$ with $\cos2\phi\propto x=2\beta$ and $\sin2\phi\propto y=\lambda_1-\lambda_2$ (a common positive factor), so $(\lambda_1-\lambda_2)\cos2\phi-2\beta\sin2\phi\propto(\lambda_1-\lambda_2)(2\beta)-2\beta(\lambda_1-\lambda_2)=0$, verifying \eqref{app:eq-cond}. After reduction modulo $\pi/2$, this is the body's Eq.~\eqref{eq:phistar} under the standard Givens orientation of Definition~\ref{def:band-aligned}. The reference implementation computes the same Jacobi formula on the pooled calibration covariance before applying the finite-phase correction, $\hat\phi_k=\tfrac12\,\mathrm{atan2}(\sigma_1^2-\sigma_2^2,\,2\rho\sigma_1\sigma_2)$, i.e.\ \eqref{app:phistar} evaluated at $(C_k,S_k)=(1,0)$. The two are related exactly, $\phi_k^\ast = \hat\phi_k - \tfrac12\,\mathrm{atan2}(S_k,C_k)\pmod{\pi/2}$ (coinciding on bands with $S_k=0$), and the excess of $\hat\phi_k$ over the minimum $\bar\lambda_k$ in the position-averaged minimax metric of the statement is $\tfrac12|S_k|\sqrt{(\sigma_1^2-\sigma_2^2)^2+(2\rho\sigma_1\sigma_2)^2}$, vanishing on isotropic bands and potentially largest in the low-frequency tail. Its numerical value must be computed from the checkpoint's deployed inverse frequencies, including any RoPE scaling. When $(\lambda_1-\lambda_2,\beta)=(0,0)$ the band is already isotropic, every $\phi$ equalises, and the convention $\mathrm{atan2}(0,0)=0$ returns the identity.

\emph{(iii) Minimax vs.\ maximin branch.}
Since orthogonal conjugation preserves the trace, $d_+(\phi)+d_-(\phi)=\lambda_1+\lambda_2$ is constant in $\phi$, so
\[
\max\bigl(d_+(\phi),d_-(\phi)\bigr) = \tfrac{\lambda_1+\lambda_2}{2} + \tfrac12\,|f(\phi)|,
\]
with $f$ the amplitude-$r$ sinusoid \eqref{app:eq-diff}, $r=\sqrt{(\lambda_1-\lambda_2)^2+4\beta^2}>0$ in the non-degenerate case $(\lambda_1-\lambda_2,\beta)\neq(0,0)$. Thus minimising the maximum band variance is equivalent to minimising $|f|$. Its zeros occur at $2\phi=\mathrm{atan2}(\lambda_1-\lambda_2,2\beta)+n\pi$, i.e.\ at $\phi=\phi_k^{\ast}+n\pi/2$. The two equalizing representatives differ by $\pi/2$ in $\phi$ and merely swap the two channel labels. Reducing modulo $\pi/2$ into $[-\pi/4,\pi/4)$ selects one representative. By contrast the branch with $2\phi=\mathrm{atan2}(2\beta,\lambda_2-\lambda_1)+n\pi$ maximises $|f|$ and gives $\max(d_+,d_-)=\tfrac12(\lambda_1+\lambda_2)+\tfrac12r$.

\emph{(iv) Minimum maximal variance and $O(2)$ optimality.}
Summing the diagonal entries of (0), $\mathrm{tr}\,\widetilde\bSigma^{(k)}=\lambda_1+\lambda_2=\sigma_1^2+\sigma_2^2$ (the $C_k,S_k$ terms cancel between $\widetilde\bSigma^{(k)}_{11}$ and $\widetilde\bSigma^{(k)}_{22}$). Hence for any $\phi$,
\[
\max\bigl(d_+(\phi),d_-(\phi)\bigr) \ge \frac{d_+(\phi)+d_-(\phi)}{2} = \frac{\sigma_1^2+\sigma_2^2}{2} = \bar\lambda_k,
\]
with equality iff $d_+(\phi)=d_-(\phi)$; the minimax $\phi_k^{\ast}$ attains it. This used only that the transform is a $2\times2$ rotation, but the bound extends to all of $O(2)$ by trace invariance alone: for any $U\in O(2)$ the conjugation $U\widetilde\bSigma^{(k)}U^\top$ preserves the trace, $\mathrm{tr}(U\widetilde\bSigma^{(k)}U^\top)=\mathrm{tr}\,\widetilde\bSigma^{(k)}=\lambda_1+\lambda_2$, so the larger of its two diagonal entries is at least $\tfrac12\mathrm{tr}(U\widetilde\bSigma^{(k)}U^\top)=\bar\lambda_k$, a bound attained inside $SO(2)$ by $G(\phi_k^{\ast})$. Therefore $\bar\lambda_k$ is the minimum of $\max(\Var(\tilde y_{2k-1}),\Var(\tilde y_{2k}))$ over all of $O(2)$, and the $SO(2)$ optimum equals the $O(2)$ optimum.
\end{proof}

\begin{proposition}[Calibration-only and deterministic]
\label{prop:calibration-only}
The $d_h/2$ angles $\{\phi_k^\ast\}$ are computable in $O(Nd_h)$ time and $O(d_h)$ memory from one calibration pass over $N$ tokens. For fixed inputs, accumulation order, and numerical implementation, no optimisation randomness is involved.
\end{proposition}

\begin{proof}[Proof of Proposition~\ref{prop:calibration-only}]
For each band $k$ the data-dependent quantities are the entries of the pre-RoPE covariance $\bSigma^{(k)}$. Five streaming sums,
$\sum a_{2k-1}$, $\sum a_{2k}$, $\sum a_{2k-1}^2$, $\sum a_{2k}^2$, and
$\sum a_{2k-1}a_{2k}$, recover the two variances and covariance after centering.
This requires $O(1)$ work per token and band, hence $O(N\dh)$ time and
$O(\dh)$ auxiliary memory overall. The constants $C_k,S_k$ depend only on the
model frequencies and calibration length and can be precomputed. Constructing
$\widetilde\bSigma^{(k)}$ and evaluating one $\mathrm{atan2}$ per band then costs
$O(\dh)$. There is no iterative optimiser, learning-rate schedule, momentum,
or per-iteration sampling. For fixed calibration tensors, reduction order, and
numerical implementation, the computation is deterministic; this statement does
not assert bitwise identity across hardware or numerical libraries.
\end{proof}

\begin{proposition}[Hadamard comparison under the surrogate band metric]
\label{prop:hadamard-suboptimality-in-band}
For the pooled-covariance, position-averaged per-band variance metric, $G(\phi_k^\ast)$ weakly dominates the $2\times2$ Hadamard proxy and strictly dominates it whenever $\widetilde\bSigma_{12}^{(k)}\ne0$.
\end{proposition}

\begin{proof}[Proof of Proposition~\ref{prop:hadamard-suboptimality-in-band}]
For the per-channel minimax metric, $\bH_2$ and $G(\pi/4)$ yield the same unordered pair of output variances, although the channel order is reversed. Direct multiplication with $\bH_2=\tfrac1{\sqrt2}\bigl(\begin{smallmatrix}1&1\\1&-1\end{smallmatrix}\bigr)$, $\lambda_1=\widetilde\bSigma^{(k)}_{11}$, $\lambda_2=\widetilde\bSigma^{(k)}_{22}$, and $\beta=\widetilde\bSigma^{(k)}_{12}=\rhoeff\sqrt{\lambda_1\lambda_2}$ gives
\[
\bH_2\,\widetilde\bSigma^{(k)}\,\bH_2^\top
= \frac12\begin{pmatrix} \lambda_1+\lambda_2+2\beta & \lambda_1-\lambda_2 \\ \lambda_1-\lambda_2 & \lambda_1+\lambda_2-2\beta \end{pmatrix},
\]
so the two output band-channel variances are
\[
\lambda_1'(\pi/4) = \frac{\lambda_1+\lambda_2}{2}+\rhoeff\sqrt{\lambda_1\lambda_2},\qquad
\lambda_2'(\pi/4) = \frac{\lambda_1+\lambda_2}{2}-\rhoeff\sqrt{\lambda_1\lambda_2},
\]
the displayed formulas, with difference
\[
\lambda_1'(\pi/4)-\lambda_2'(\pi/4) = 2\beta = 2\rhoeff\sqrt{\lambda_1\lambda_2}.
\]
As $\lambda_1,\lambda_2>0$, this vanishes iff $\rhoeff=0$; whenever $\rhoeff\neq0$, $\bH_2$ fails to equalise and yields $\max(\lambda_1',\lambda_2')=\tfrac12(\lambda_1+\lambda_2)+|\rhoeff|\sqrt{\lambda_1\lambda_2}>\bar\lambda_k$.

For strict domination, recall from the proof of Theorem~\ref{thm:closed-form-optimal-angle}(iii)--(iv) that any $2\times2$ orthogonal $U$ scores $\max(d_+(U),d_-(U))=\bar\lambda_k+\tfrac12|f_U|$ in the minimax band metric, where $f_U=d_+(U)-d_-(U)$ and $d_++d_-=\lambda_1+\lambda_2$ is fixed. The closed-form $G(\phi_k^{\ast})$ achieves $f=0$ and therefore reaches $\bar\lambda_k$, whereas $f_{\bH_2}=2\rhoeff\sqrt{\lambda_1\lambda_2}=2\widetilde\bSigma^{(k)}_{12}$. Thus $G(\phi_k^{\ast})$ strictly dominates $\bH_2$ exactly when $\widetilde\bSigma^{(k)}_{12}\neq0$. If $\widetilde\bSigma^{(k)}_{12}=0$ and $\widetilde\bSigma^{(k)}_{11}\neq\widetilde\bSigma^{(k)}_{22}$, Eq.~\eqref{app:phistar} yields $\phi_k^{\ast}=\pm\pi/4$, so the two transforms attain the same diagonal variances. Consequently, $G(\phi_k^{\ast})$ weakly dominates $\bH_2$ on every band under the surrogate metric and strictly dominates it whenever $\widetilde\bSigma^{(k)}_{12}\neq0$.
\end{proof}

\begin{proof}[Proof of Proposition~\ref{prop:rope-commutativity-requires-shared-theta}]
Both $R_3^Q,R_3^K\in\Bcal_{\dh}$, so by Lemma~\ref{lem:centraliser-commutes} each commutes with every $\Rcal_m$. Using $\Rcal_m^\top=\Rcal_{-m}$ and $M^{QK}:=(R_3^Q)^\top R_3^K$,
\[
(R_3^Q\Rcal_m\bq)^\top(R_3^K\Rcal_n\bk)
= \bq^\top\Rcal_m^\top (R_3^Q)^\top R_3^K \Rcal_n\bk
= \bq^\top\Rcal_m^\top\, M^{QK}\,\Rcal_n\bk.
\]
Now $M^{QK}=\bigoplus_k G(\phi_k^Q)^\top G(\phi_k^K)=\bigoplus_k G(\phi_k^K-\phi_k^Q)$, using $G(\alpha)^\top=G(-\alpha)$ and $SO(2)$ additivity. Since $M^{QK}$, $\Rcal_m^\top$, $\Rcal_n$ are all band-diagonal under the same partition, the bilinear form splits over bands:
\[
\bq^\top\Rcal_m^\top M^{QK}\Rcal_n\bk
= \sum_{k=1}^{K}\bq^{(k)\top}\,(\Rcal_m^{(k)})^\top\,G(\phi_k^K-\phi_k^Q)\,\Rcal_n^{(k)}\,\bk^{(k)},
\]
with $\bq^{(k)},\bk^{(k)}\in\R^2$ the band-$k$ components. Within band $k$ every factor lies in the abelian group $SO(2)$, so the rotations compose by adding angles:
\[
(\Rcal_m^{(k)})^\top\,G(\phi_k^K-\phi_k^Q)\,\Rcal_n^{(k)}
= G(-m\theta_k)\,G(\phi_k^K-\phi_k^Q)\,G(n\theta_k)
= G((n-m)\theta_k)\,G(\phi_k^K-\phi_k^Q)
= \Rcal_{n-m}^{(k)}\,G(\phi_k^K-\phi_k^Q).
\]
Substituting back yields the displayed identity
\[
(R_3^Q\Rcal_m\bq)^\top(R_3^K\Rcal_n\bk)
= \sum_{k=1}^{K}\bq^{(k)\top}\,\Rcal_{n-m}^{(k)}\,G(\phi_k^K-\phi_k^Q)\,\bk^{(k)}.
\]
If $\phi_k^Q=\phi_k^K\pmod{2\pi}$ for all $k$, then $G(\phi_k^K-\phi_k^Q)=G(0)=I$ on every band and the sum collapses to $\sum_k\bq^{(k)\top}\Rcal_{n-m}^{(k)}\bk^{(k)}=\bq^\top\Rcal_{n-m}\bk$, the original RoPE relative-position score, which depends on $(m,n)$ only through $n-m$. If $\phi_k^Q\not\equiv\phi_k^K\pmod{2\pi}$ for some $k$, the band-$k$ term carries the extra factor $G(\phi_k^K-\phi_k^Q)$, a fixed nonidentity rotation. Being position-independent, it cannot be absorbed into the relative rotation $\Rcal_{n-m}^{(k)}=G((n-m)\theta_k)$ for all $(m,n)$ simultaneously (that would require $\phi_k^K-\phi_k^Q\equiv0\pmod{2\pi}$); it does not cancel against RoPE and perturbs the band-$k$ Q/K inner product on which attention relies. Hence a single shared online $R_3$ ($\phi_k^Q=\phi_k^K\pmod{2\pi}$ for all $k$) is required to preserve the attention geometry.
\end{proof}

\end{document}